\documentclass[10pt,journal,compsoc]{IEEEtran}
\usepackage{amsfonts}
\usepackage{bm}

\ifCLASSOPTIONcompsoc
  \usepackage[nocompress]{cite}
\else
  \usepackage{cite}
\fi
\ifCLASSINFOpdf
\else
\fi
\usepackage{amsmath, amssymb}
\usepackage{amsthm}
\usepackage{xcolor}
\usepackage{algorithm}
\usepackage{algorithmic}

\usepackage{titletoc}
\usepackage{tikz}
\usepackage[framemethod=tikz]{mdframed}
\usepackage{thm-restate}
\usepackage{booktabs}
\usepackage{hyperref}
\usepackage{subfigure} 
\usepackage{multirow}
\usepackage{ragged2e}
\usepackage{soul}
\usepackage{tcolorbox}
\usepackage[numbers]{natbib}
\usepackage{setspace} 
\usepackage{adjustbox}
\usepackage{enumitem}
\hypersetup{
	colorlinks=true,
	linkcolor=blue,
	urlcolor=orange,
	citecolor=orange,
}

\newtheorem{theorem}{Theorem}
\newtheorem{lemma}{Lemma}
\newtheorem{definition}{Definition}
\newtheorem{proposition}{Proposition}
\newtheorem{remark}{Remark}

\newtheorem{assumption}{Assumption}

\newcommand{\size}[1]{\left\lvert #1 \right\rvert}
\newcommand{\I}[1] {\mathbf{1} \left[ #1 \right]}
\newcommand{\pp}[1] {\mathbb{P} \left[ #1 \right]}
\newcommand{\E}[2] {\mathop{\mathbb{E}}\limits_{#1} \left[ #2 \right]}

\newcommand{\Ee}[2] {\mathop{\mathbb{E}}_{#1} [ #2 ]}

\definecolor{Top2}{RGB}{102, 171, 221}
\definecolor{Top1}{RGB}{245, 137, 112}
\newcommand{\Topone}[1]{{\color{Top1}\textbf{#1}}}
\newcommand{\Toptwo}[1]{{\color{Top2}\textbf{#1}}}

\usepackage{pifont}
\newcommand{\cmark}{\ding{51}}%
\newcommand{\xmark}{\ding{55}}%
\usepackage{bbding}
\newcommand{\bmark}{\ding{51} \kern-1.7ex\raisebox{0.6ex}{\rotatebox[origin=c]{125}{\textbf{--}}}}

\tcbuselibrary{skins}

\newenvironment{qbox-tight}
	{\begin{tcolorbox}[enhanced jigsaw, drop shadow=black!50!white,colback=white, width=\linewidth, center, left=2pt,right=2pt,top=1pt,bottom=1pt,halign=center]}
	{\end{tcolorbox}}

\usepackage{listings}
\usepackage{color}
\definecolor{dkgreen}{rgb}{0,0.6,0}
\definecolor{gray}{rgb}{0.5,0.5,0.5}
\definecolor{mauve}{rgb}{0.58,0,0.82}
\begin{document}
  %
\title{A Unified Perspective for Loss-Oriented Imbalanced Learning via Localization}

\author{Zitai~Wang,
        Qianqian~Xu*,~\IEEEmembership{Senior ~Member,~IEEE,}
        Zhiyong~Yang, 
        Zhikang~Xu\\
        Linchao~Zhang, 
        Xiaochun~Cao,~\IEEEmembership{Senior ~Member,~IEEE,}
        and~Qingming~Huang*,~\IEEEmembership{Fellow,~IEEE}
        \IEEEcompsocitemizethanks{
            \IEEEcompsocthanksitem Zitai Wang is with the State Key Laboratory of AI Safety, Institute of Computing Technology, Chinese Academy of Sciences, Beijing 100190, China (email: \texttt{wangzitai@ict.ac.cn}). 
            \IEEEcompsocthanksitem Qianqian Xu is with the State Key Laboratory of AI Safety, Institute of Computing Technology, Chinese Academy of Sciences, Beijing 100190, China, and also with Peng Cheng Laboratory, Shenzhen 518055, China (email: \texttt{xuqianqian@ict.ac.cn}). 
            \IEEEcompsocthanksitem Zhiyong Yang is with School of Computer Science and Technology, University of Chinese Academy of Sciences, Beijing 100049, China (email: \texttt{ yangzhiyong21@ucas.ac.cn}).
            \IEEEcompsocthanksitem Zhikang Xu is with Institute of Information Engineering, Chinese Academy of Sciences, Beijing 100093, China, and also with School of Cyber Security, University of Chinese Academy of Sciences, Beijing 100049, China (email: \texttt{xuzhikang@iie.ac.cn}). 
            \IEEEcompsocthanksitem Linchao Zhang is with Artificial Intelligence Institute of China Electronics Technology Group Corporation, Beijing 100041, China (email: \texttt{zhanglinchao@cetc.com.cn}).
            \IEEEcompsocthanksitem Xiaochun Cao is with School of Cyber Science and Technology, Shenzhen Campus of Sun Yat-sen University, Shenzhen 518107, China (email:\texttt{caoxiaochun@mail.sysu.edu.cn}).
            \IEEEcompsocthanksitem Qingming Huang is with the School of Computer Science and Technology, University of Chinese Academy of Sciences, Beijing 101408, China, also with the Key Laboratory of Big Data Mining and Knowledge Management (BDKM), University of Chinese Academy of Sciences, Beijing 101408, China, and also with the State Key Laboratory of AI Safety, Institute of Computing Technology, Chinese Academy of Sciences, Beijing 100190, China (email: \texttt{qmhuang@ucas.ac.cn}).
            \IEEEcompsocthanksitem *Corresponding authors.
        }
}

\markboth{Journal of \LaTeX\ Class Files,~Vol.~xx, No.~xx, January~2025}%
{Shell \MakeLowercase{\textit{et al.}}: Bare Demo of IEEEtran.cls for Computer Society Journals}
%

\IEEEtitleabstractindextext{%
    \begin{abstract}
        \justifying
        Due to the inherent imbalance in real-world datasets, na\"ive Empirical Risk Minimization (ERM) tends to bias the learning process towards the majority classes, hindering generalization to minority classes. To rebalance the learning process, one straightforward yet effective approach is to modify the loss function via class-dependent terms, such as re-weighting and logit-adjustment. However, existing analysis of these loss-oriented methods remains coarse-grained and fragmented, failing to explain some empirical results. After reviewing prior work, we find that the properties used through their analysis are typically global, \textit{i.e.}, defined over the whole dataset. Hence, these properties fail to effectively capture how class-dependent terms influence the learning process. To bridge this gap, we turn to explore the localized versions of such properties \textit{i.e.}, defined within each class. Specifically, we employ  localized calibration to provide consistency validation across a broader range of losses and localized Lipschitz continuity to provide a fine-grained generalization bound. In this way, we reach a unified perspective for improving and adjusting loss-oriented methods. Finally, a principled learning algorithm is developed based on these insights. Empirical results on both traditional ResNets and foundation models validate our theoretical analyses and demonstrate the effectiveness of the proposed method.   
    \end{abstract}
    
}
  \maketitle
  \IEEEdisplaynontitleabstractindextext
  \IEEEpeerreviewmaketitle

  \IEEEraisesectionheading{\section{Introduction}\label{sec:introduction}}
\IEEEPARstart{I}{n} recent years, machine learning methods have achieved remarkable success, driven by well-collected datasets with balanced class distributions \cite{krizhevsky2009learning,DBLP:journals/ijcv/RussakovskyDSKS15}. However, real-world datasets are often imbalanced: some classes contain many samples (\textit{i.e.}, majority classes), while others have very few (\textit{i.e.}, minority classes). This imbalance leads Empirical Risk Minimization (ERM) to focus primarily on majority classes, making it harder to generalize to minority classes. To tackle this challenge, imbalanced learning has gained significant attention in recent years \cite{DBLP:journals/pami/OksuzCKA21,DBLP:journals/pami/ZhangKHYF23,DBLP:journals/corr/abs-2408-00483}, where models are trained on imbalanced datasets but evaluated on balanced ones.

In this field, most methods fall into three categories: loss-oriented, data-oriented, and module-oriented ones (Please refer to Sec.\ref{sec:imbalanced_learning} for more details). Among these, loss-oriented ones can achieve superior performance with minimal modifications to the training pipeline.
To achieve this goal, two types of modifications are proposed: re-weighting \cite{DBLP:conf/icml/MorikBJ99,DBLP:conf/cvpr/CuiJLSB19,Park_2021_ICCV} and logit-adjustment \cite{DBLP:conf/nips/CaoWGAM19,DBLP:conf/iclr/MenonJRJVK21,DBLP:journals/corr/abs-2001-01385,DBLP:conf/cvpr/TanWLLOYY20,DBLP:conf/nips/KiniPOT21,DBLP:journals/pami/TanLLYYHO23,DBLP:conf/nips/WangX00CH23}, whose differences lie in the position of the modification. Abstractly, we have the following unified formulation:
\begin{equation}
    \label{eq:unified_loss}
    L_\text{modified}(f(\boldsymbol{x}), y) = \alpha_y \cdot L_\text{naive}(\tilde{f}(\boldsymbol{x}), y),
\end{equation}
where $\alpha_y$ denotes the re-weighting term of the class $y$, and $\tilde{f}(\boldsymbol{x}) := \beta_y f(\boldsymbol{x}) + \Delta_y$ is the adjusted logit; $\beta_y$ and $\Delta_y$ denote multiplicative and additive adjustment terms, respectively.

Given this formulation, one natural question arises: \textit{What kind of $\alpha_y, \beta_y, \Delta_y$ should we choose?} According to the classic paradigm of machine learning \cite{10.5555/2371238}, two properties are crucial. One is \textbf{Fisher consistency}, \textit{i.e.}, minimizing the modified loss should also optimize the performance defined on the balanced distribution. The other one is \textbf{sharp generalization bound}, \textit{i.e.}, the performance gap between the training set and the test set should be small. From the two perspectives, prior arts have justified their method with various techniques, as summarized in Tab.\ref{tab:comparison}. However, \textbf{existing theoretical insights are still fragmented and coarse-grained}. Specifically, \citet{DBLP:conf/nips/CaoWGAM19} and \citet{DBLP:conf/nips/RenYSMZYL20} construct a margin-based generalization bound to validate the necessity of the proposed Label Distribution Aware Margin (LDAM) loss ($\Delta_y \neq 0$). However, this theory cannot explain the significant improvement induced by the Deferred Re-Weighting (DRW) scheme ($\alpha_y \neq 1$). \citet{DBLP:conf/iclr/MenonJRJVK21} analyze the Fisher consistency of the proposed Logit-Adjustment (LA) loss ($\Delta_y \neq 0$), while the consistency of multiplicative adjustment $\beta_y$ is still unexplored. \citet{DBLP:conf/nips/KiniPOT21} provide a generalization analysis of the Vector-Scaling (VS) loss, but the results are limited in a linear model trained on linearly-separable two-class data. Recently, \citet{hasegawa2024multi} analyze logit-adjustment from the perspective of neural collapse. However, their insights only apply to multiplicative adjustment $\beta_y$. Besides, we find that some modifications are not necessarily compatible with each other, which is also out of the scope of existing theory. 

Given this, this paper aims to provide a unified and fine-grained analysis of both consistency and generalization. After revisiting prior arts, we find that the properties utilized through their proofs are generally global, \textit{i.e.}, defined over the whole dataset. However, model behaviors and loss properties are highly class-dependent due to the imbalance nature. To bridge this gap, we localize these properties within each class, such that the granularity of analysis can align with $\alpha_y, \beta_y, \Delta_y$. \textbf{For consistency}, in Sec.\ref{sec:consistency}, we find that the calibration assumption used in \cite{DBLP:conf/iclr/MenonJRJVK21} does not necessarily hold for minority classes, even when the model has been calibrated globally via the mixup technique \cite{DBLP:conf/iclr/ZhangCDL18}. Hence, we relax the assumption to its localized version and provide consistency validation for both additive and multiplicative adjustments.
\textbf{For generalization}, in Sec.\ref{sec:generalization}, we find that the only property of the loss function utilized in existing proofs is Lipschitz continuity \cite{10.5555/2371238,ledoux1991probability}, which is also global. 
Similarly, we derive its localized version, where the local Lipschitz constant of each class exactly corresponds to the class-dependent terms. Building on this, we establish a fine-grained generalization bound that provides theoretical insights into how to better utilize re-weighting and logit adjustment.

Finally, in Sec.\ref{sec:method}, based on the theoretical insights, we develop a principled learning algorithm for our proposed consistent  loss. In Sec.\ref{sec:experiments}, extensive experiments are conducted, \textbf{ranging from training ResNets from scratch under multiple protocols to fine-tuning pre-trained ViTs}. The empirical results not only validate our theoretical insights but also demonstrate the effectiveness of the proposed method.

Overall, the contributions of this paper are as follows:
\begin{itemize}
    \item \textbf{New technique.} We extend the classic model calibration and Lipschitz continuity to their localized versions, such that the granularity of analysis can align with the class-dependent terms.
    \item \textbf{Theoretical insights.} We present a fine-grained analysis of consistency and generalization for loss-oriented imbalanced learning. The theoretical results provide unified insights for re-weighting and logit-adjustment.
    \item \textbf{Principled learning algorithm.} A principled learning algorithm is proposed based on the theoretical insights. Extensive experiments on both ResNets and ViTs not only validate the theoretical results but also demonstrate the effectiveness of the proposed method.
\end{itemize}

\begin{table}[t]
    \caption{An overview of the existing analysis of loss-oriented imbalanced learning, where $\alpha_y, \beta_y, \Delta_y$ denote whether the analysis covers the corresponding terms; Cons. and Gene. denote Consistency and Generalization analysis, respectively; \bmark \ means the analysis is coarse-grained or limited in some restricted settings.}
    \label{tab:comparison}
    \centering
    \renewcommand{\arraystretch}{1.1}
    \begin{tabular}{l|ccc|cc}
        \toprule
        Reference & $\alpha_y$ & $\beta_y$ & $\Delta_y$ & Cons. & Gene. \\
        \midrule
        LDAM \cite{DBLP:conf/nips/CaoWGAM19,DBLP:conf/nips/RenYSMZYL20} & \xmark & \xmark & \cmark & \xmark & \bmark \\
        LA \cite{DBLP:conf/iclr/MenonJRJVK21} & \xmark & \xmark & \cmark & \cmark & \xmark \\
        CDT \cite{DBLP:journals/corr/abs-2001-01385,hasegawa2024multi} & \xmark & \cmark & \xmark & \xmark & \xmark \\
        VS \cite{DBLP:conf/nips/KiniPOT21} & \cmark & \cmark & \cmark & \xmark & \bmark \\
        \midrule
        Conf. ver. \cite{DBLP:conf/nips/WangX00CH23} & \cmark & \cmark & \cmark & \xmark & \cmark \\
        This paper & \cmark & \cmark & \cmark & \cmark & \cmark \\
        \bottomrule
    \end{tabular}
\end{table}

This work extends our NeruIPS 2023 Spotlight paper \cite{DBLP:conf/nips/WangX00CH23}, where we have provided a unified generalization analysis for loss-oriented methods. The novelty of this extended version lies in the following aspects:
\begin{itemize}
    \item In Sec.\ref{sec:consistency}, the concept of localization, introduced in our generalization analysis, is extended to consistency analysis (\textit{i.e.}, Prop.\ref{prop:local_nonlinear_loss}). To the best of our knowledge, this is \textbf{\color{blue}the first consistency validation} for multiplicative logit-adjustment.
    \item In Sec.\ref{sec:generalization}, we refine the generalization analysis to derive \textbf{\color{blue}a sharper term}, $\tilde{\beta}$ (Lem.\ref{lem:Lip_of_vs} and Prop.\ref{prop:vs_generalization}). The theoretical insights and corresponding empirical validation (Fig.\ref{fig:gamma}) have been updated accordingly.
    \item In Sec.\ref{sec:method}, we \textbf{\color{blue}adjust the learning algorithm} based on the new theoretical insights, replacing the logit-adjustment terms with the newly proposed ones (Eq.(\ref{eq:cvs})).
    \item In Sec.\ref{sec:experiments} and Appendix \ref{app:more_exp_results}, we provide additional empirical results, including training ResNets from scratch under \textbf{\color{blue}more protocols} (\textit{e.g.}, Tab.\ref{tab:caifar}, Tab.\ref{tab:caifar-protocol-c},  Fig.\ref{fig:sensitivity_new}, and Fig.\ref{fig:sensitivity_new_b}) and \textbf{\color{blue}fine-tuning pre-trained ViTs} (\textit{e.g.}, Tab.\ref{tab:caifar-protocol-c} and Tab.\ref{tab:ima_ina}).
\end{itemize}    

\section{Related Work}
\subsection{Imbalanced Learning}
\label{sec:imbalanced_learning}
Due to the prevalence of imbalance, there are a large number of studies focusing on imbalanced learning, \textit{a.k.a}, long-tailed learning \cite{DBLP:journals/pami/OksuzCKA21,DBLP:journals/pami/ZhangKHYF23,DBLP:journals/corr/abs-2408-00483}. In this field, the learning process is modified to alleviate the bias towards the majority classes. According to the position of modification, we divide existing methods into three categories: loss-oriented, data-oriented, and module-oriented ones.

Loss-oriented ones fall into the following categories: \textbf{Re-weighting} aims to achieve balanced learning by assigning larger weights to the losses of the minority classes \cite{DBLP:conf/icml/MorikBJ99,DBLP:conf/cvpr/CuiJLSB19,DBLP:conf/iccv/ParkLJ021}. For example, \citet{DBLP:conf/cvpr/CuiJLSB19} propose the Class-Balanced (CB) loss based on the effective number of each class. The Influence-Balanced (IB) loss re-weighs the samples according to their influence on the decision boundary  \cite{DBLP:conf/iccv/ParkLJ021}. Nevertheless, the re-weighting could potentially result in difficulties and instability during optimization \cite{DBLP:conf/cvpr/CuiJLSB19,DBLP:conf/nips/CaoWGAM19,DBLP:conf/cvpr/HuangLLT16}. To address this issue, Deferred Re-Weighting (DRW) performs re-weighting only during the terminal phase of training \cite{DBLP:conf/nips/CaoWGAM19}. 
\textbf{Logit-adjustment} utilizes class-dependent terms to adjust the logits \cite{DBLP:conf/nips/CaoWGAM19,DBLP:conf/iclr/MenonJRJVK21,DBLP:conf/nips/KiniPOT21,DBLP:journals/corr/abs-2001-01385,DBLP:conf/cvpr/TanWLLOYY20,DBLP:journals/pami/TanLLYYHO23}. For instance, the Label Distribution Aware Margin (LDAM) loss adjusts the logits to pursue an optimal shifted margin \cite{DBLP:conf/nips/CaoWGAM19}. The Logit-Adjustment (LA) loss \cite{DBLP:conf/iclr/MenonJRJVK21} and the Class-Dependent Temperatures (CDT) loss \cite{DBLP:journals/corr/abs-2001-01385} employ additive and multiplicative terms, respectively. Combining the two losses, \citet{DBLP:conf/nips/KiniPOT21} present a unified formulation named the Vector-Scaling (VS) loss. Besides, the LA loss has many variants, from the perspective of hard instances \cite{DBLP:conf/aaai/ZhaoC0HZ22}, feature space \cite{DBLP:conf/cvpr/LiCL22}, ensembling \cite{DBLP:conf/iccv/TaoSYCWYDZ23}, and shifted label distribution \cite{DBLP:conf/cvpr/AimarJFK23,DBLP:conf/nips/ZhuTSZ23,DBLP:journals/tcsv/ZhangGLC24,DBLP:conf/icml/0001XWLHBCH24}.

Data-oriented ones modify the training data to alleviate the imbalance issue. \textbf{Data re-sampling} aims to balance the class distribution by over-sampling the minority classes \cite{chawla2002smote} or under-sampling the majority classes \cite{drummond2003c4}. However, oversampling could induce overfitting, while undersampling could limit the learning process within a small sample space \cite{DBLP:journals/jbd/JohnsonK19}. Hence, \textbf{data augmentation} has risen as a promising alternative. For example, RandAugment \cite{DBLP:conf/nips/CubukZS020} is shown effective in imbalanced learning \cite{DBLP:conf/iccv/CuiZ00J21,DBLP:conf/iccv/TaoSYCWYDZ23,DBLP:journals/pami/DuWSH24}. Further, CUDA \cite{DBLP:conf/iclr/AhnKY23} is designed to find fine-grained class-wise augmentations. As another example, Mixup \cite{DBLP:conf/iclr/ZhangCDL18} and its variants \cite{DBLP:conf/eccv/ChouCPWJ20,DBLP:conf/nips/XuCY21} not only improve generalization but also help calibrate the model predictions. Besides, \textbf{data synthesis} uses complex techniques to generate samples, such as GANs \cite{DBLP:conf/cvpr/KhorramJSDL24} and pre-trained foundation models \cite{DBLP:conf/cvpr/ZhaoDL00024}.

Module-oriented ones are much more diverse. To name a few, \textbf{decoupling methods} adopts a two-stage learning paradigm, initially focusing on feature learning through a naïve learning process, followed by retraining the classifier with a balanced label distribution \cite{DBLP:conf/iclr/KangXRYGFK20,DBLP:conf/cvpr/ZhouCWC20,DBLP:conf/cvpr/LiCL22}. \textbf{SAM-based methods} utilize Sharpness-Aware Minimization (SAM) \cite{DBLP:conf/iclr/ForetKMN21} improve the generalization of the tail classes \cite{rangwani2022escaping,DBLP:conf/iccv/ZhouQXS23}. \textbf{contrastive-based methods} re-balance the supervised contrastive loss via a set of parametric class-wise learnable centers \cite{DBLP:conf/iccv/CuiZ00J21,DBLP:journals/pami/DuWSH24}. \textbf{Ensemble methods} assign each expert with different classes \cite{DBLP:conf/cvpr/ZhouCWC20,DBLP:conf/iccv/Cai0H21,DBLP:conf/iclr/WangLM0Y21} or distributions \cite{DBLP:conf/nips/ZhangHHF22,DBLP:conf/icml/0001XWLHBCH24} to acquire complementary knowledge.

Our analysis primarily focuses on loss-oriented methods. Techniques such as RandAugment and SAM are also utilized in some experiments. The empirical results show that when combined with these techniques, our proposed method achieves better performance.

\subsection{Model Calibration}
\label{sec:calibration}
In the early stage, neural networks are typically well-calibrated on binary classification tasks \cite{DBLP:conf/icml/Niculescu-MizilC05}, meaning that the predicted class probabilities of the model could accurately estimate the likelihood of correctness. However, \citet{DBLP:conf/icml/GuoPSW17} find that modern neural networks are poorly calibrated, perhaps due to their high capacity. According to the taxonomy in \cite{DBLP:journals/corr/abs-2308-01222}, calibration methods can be divided into three categories. The first one is post-hoc calibration, such as histogram binning \cite{DBLP:conf/icml/ZadroznyE01} and scaling-based ones \cite{DBLP:conf/icml/GuoPSW17}. The second one is regularization-based calibration, which aims to mitigate overfitting and thus improve calibration \cite{DBLP:conf/iclr/PereyraTCKH17,DBLP:conf/nips/ThulasidasanCBB19,DBLP:conf/cvpr/LiuAGD22}. For example, mixup \cite{DBLP:conf/iclr/ZhangCDL18} and label smoothing \cite{DBLP:conf/nips/MullerKH19} are shown to be effective in improving calibration \cite{DBLP:conf/nips/ThulasidasanCBB19,DBLP:conf/cvpr/LiuAGD22}. The third one alleviates miscalibration by injecting randomness \cite{DBLP:conf/icml/BlundellCKW15}.

Recently, calibration in imbalanced learning has attracted rising attention. Although the Focal loss \cite{DBLP:journals/pami/LinGGHD20} has shown effective in calibration \cite{DBLP:conf/emnlp/WangBSEPD22}, its empirical performance is inferior to loss-oriented methods. Hence, more attention has been paid to implicit regularization. For example, MiSLaS \cite{DBLP:conf/cvpr/ZhongC0J21} utilizes label-aware smoothing and shifted batch normalization. UniMix \cite{DBLP:conf/nips/XuCY21} replaces the beta distribution in mixup with an imbalance-aware one to construct a more class-balanced virtual dataset. \citet{DBLP:conf/cvpr/HongHCSKC21} find their proposed method, named LADE, is well-calibrated since the logits are regularized by the proposed loss. \citet{DBLP:conf/cvpr/AimarJFK23} show that the mixup technique can effectively improve calibration of the LA loss, with proper hyperparameters.

Although these methods have significantly improved the global calibration performance, they generally ignore the heterogeneity across different classes. Our analysis in Sec.\ref{subsec:calibration_issue} shows that the minority classes remain poorly calibrated, even when the model is well-calibrated globally with the mixup technique. This observation motivates us to extend existing calibration assumption to its localized version, which is crucial to the following consistency validation.

\begin{table}[t]
    \caption{Some important notations used in this paper. The subscript $y$ shows that localized properties are widely used in our analysis.}
    \label{tab:notation}
    \centering
    \renewcommand{\arraystretch}{1.05}
    \begin{tabular}{ll}
        \toprule
         & Description \\
        \midrule
        $C$ & the number of classes \\
        $\mathcal{X}, \mathcal{Y}$ & the input space and the label space \\
        $\mathcal{D}$ & the joint distribution defined on $\mathcal{Z} = \mathcal{X} \times \mathcal{Y}$ \\
        $\mathcal{D}_y$ & the class-conditional distribution $\mathbb{P}[\boldsymbol{x} \mid y]$ \\
        $\mathcal{D}_\text{bal}$ & the balanced distribution defined on $\mathcal{Z}$ \\
        $\mathcal{S}$ & the dataset drawn from $\mathcal{D}$, whose size is $N$ \\
        $\mathcal{S}_y$ & the sample set of class $y$, whose size is $N_y$ \\
        $\rho$ & the degree of imbalance, \textit{i.e.}, $N_1 / N_C$ \\
        $\pi_y$ & the class prior, \textit{i.e.}, $N_y / N$ \\
        $f$ & the score function mapping $\mathcal{X}$ to $\mathbb{R}^{C}$ \\
        $\mathcal{F}$ & the set of score functions \\
        $L$ & the loss function mapping $\mathcal{Y} \times \mathbb{R}^C$ to $\mathbb{R}_{+}$ \\
        $\mathcal{R}_\text{bal}(f)$ & the risk defined on $\mathcal{D}_\text{bal}$ \\
        $\mathcal{R}_\text{bal}^L(f)$ & the surrogate risk defined on $\mathcal{D}_\text{bal}$ \\
        $\hat{\mathcal{R}}_\text{bal}^L(f)$ & the empirical risk defined on $\mathcal{S}$ \\
        \midrule
        $\alpha_y$ & the re-weighting term of class $y$ \\
        $\beta_y$ & the multiplicative adjustment term of class $y$ \\
        $\Delta_y$ & the additive adjustment term of class $y$ \\
        $\nu, \gamma, \tau$ & the hyperparameters of $\alpha_y, \beta_y, \Delta_y$, resp. \\
        $\kappa_y^*, \kappa_y^+$ & the local calibration terms \\
        \midrule
        $\mathcal{G}$ & the function family $\{L \circ f,  f \in \mathcal{F}\}$ \\
        $\hat{\mathfrak{C}}_{\mathcal{S}_y}(\mathcal{G}) $ & the complexity of $\mathcal{G}$ defined on $\mathcal{S}_y$ \\
        $\mu_y$ & the local Lipschitz constant of class $y$ \\
        $B_y(f)$ & the minimal prediction of $f$ on class $y$ \\
        $\text{margin}_y^\downarrow$ & the minimal margin of class $y$ \\
        \bottomrule
    \end{tabular}
\end{table}

\section{Preliminary}
\label{sec:pre}
In this section, we present some preliminary knowledge and further outline our main idea. Some important notations are summarized in Tab.\ref{tab:notation}.

We assume that samples are drawn \textit{i.i.d.} from a product space $\mathcal{Z} = \mathcal{X} \times \mathcal{Y}$, where $\mathcal{X}$ denotes the input space and $\mathcal{Y} = \{1, \cdots, C\}$ is the label space. In imbalanced learning, the training set $\mathcal{S} = \{(\boldsymbol{x}^{(n)}, y^{(n)})\}_{n=1}^N$ is sampled from an imbalanced distribution $\mathcal{D}$ defined on $\mathcal{Z}$. Specifically, let $\mathcal{S}_y = \{\boldsymbol{x} \mid (\boldsymbol{x}, y) \in \mathcal{S}\}$ denote the set of samples from the class $y$ and $N_y := \size{\mathcal{S}_y}$ be the size of $\mathcal{S}_y$. Without loss of generality, we assume $N_1 \ge N_2 \ge \cdots \ge N_C$. Then, the degree of imbalance $\rho := N_1 / N_C$ is large in this setting.

Let $\mathcal{D}_\text{bal}$ be the balanced distribution defined on $\mathcal{Z}$. To be concrete, a class $y$ is first uniformly sampled from $\mathcal{Y}$, and then the input $\boldsymbol{x}$ is sampled from the class-conditional distribution $\mathcal{D}_y := \pp{\boldsymbol{x} \mid y}$. Then, our task is to learn a score function $f: \mathcal{X} \to \mathbb{R}^{C}$ to minimize the risk defined on $\mathcal{D}_\text{bal}$:
\begin{equation}
    \label{eq:bal_risk}
    \mathcal{R}_\text{bal} (f) := \frac{1}{C} \sum_{y=1}^C \mathcal{R}_y(f) = \frac{1}{C} \sum_{y=1}^{C} \E{\boldsymbol{x} \sim \mathcal{D}_y}{M(f(\boldsymbol{x}), y)},
\end{equation}
where $\mathcal{R}_y$ is the risk defined on $\mathcal{D}_y$, and $M: \mathbb{R}^C \times \mathcal{Y} \to \mathbb{R}_{+}$ is the measure that evaluates the model performance at $\boldsymbol{z} \in \mathcal{Z}$. For example, one common choice is the error rate:
\begin{equation}
    M(f(\boldsymbol{x}), y) = \I{y \notin \arg \max_{y' \in \mathcal{Y}} f(\boldsymbol{x})_{y'}},
\end{equation}
where $\I{\cdot}$ is the indicator function. Since $M$ is generally non-differentiable and thus hard to optimize, one has to select a differential surrogate loss $L: \mathbb{R}^C \times \mathcal{Y} \to \mathbb{R}_{+}$, inducing the following surrogate risk:
\begin{equation}
    \label{eq:bal_surrogate}
    \mathcal{R}_\text{bal}^L (f) := \frac{1}{C} \sum_{y=1}^C \mathcal{R}_y^L(f) = \frac{1}{C} \sum_{y=1}^{C} \E{\boldsymbol{x} \sim \mathcal{D}_y}{L(f(\boldsymbol{x}), y)}.
\end{equation}
In balanced learning, we can directly minimize the empirical surrogate risk $\hat{\mathcal{R}}_\text{bal}^L$ defined on the balanced training set, and the generalization guarantee is available by traditional concentration techniques \cite{10.5555/2371238}. However, in imbalance learning, the balanced training set is unavailable. Instead, we can only optimize the empirical surrogate risk defined on the imbalanced training set $\mathcal{S}$:
\begin{equation}
    \label{eq:empirical_surrogate}
    \hat{\mathcal{R}}^L(f) := \frac{1}{N} \sum_{(\boldsymbol{x}, y) \in \mathcal{S}} L(f(\boldsymbol{x}), y).
\end{equation}

Given Eq.(\ref{eq:empirical_surrogate}), two research questions naturally arise according to the classic paradigm of machine learning \cite{10.5555/2371238}: 
\begin{itemize}[leftmargin=*]
    \item \textbf{(RQ1)} What kind of surrogate loss $L$ should we select?
    \item \textbf{(RQ2)} Can minimize the empirical surrogate risk $\hat{\mathcal{R}}^L(f)$ guarantee a satisfying generalization performance, \textit{i.e.}, a small $\mathcal{R}_\text{bal}^L(f)$?  
\end{itemize}
For \textbf{(RQ1)}, a basic requirement is Fisher consistency \cite{10.5555/2371238,DBLP:conf/icml/MenonNAC13}, \textit{i.e.}, minimizing the surrogate risk $\mathcal{R}_\text{bal}^L (f)$ can also guarantee a small balanced risk $\mathcal{R}_\text{bal}(f)$. For \textbf{(RQ2)}, we should further investigate the upper bound of the generalization error $\mathcal{R}_\text{bal}^L(f) - \hat{\mathcal{R}}^L(f)$. Combining these, a small balanced risk $\mathcal{R}_\text{bal}(f)$ can be achieved by 
$$
    \hat{\mathcal{R}}^L(f) \xrightarrow{\text{  Gene.  }} \mathcal{R}_\text{bal}^L (f) \xrightarrow{\text{  Cons.  }} \mathcal{R}_\text{bal}(f).
$$ 
In view of this, we next provide a systematic analysis of consistency and generalization in Sec.\ref{sec:consistency} and Sec.\ref{sec:generalization}, respectively. 

To provide unified insights, we consider a family of loss functions named Vector-Scaling (VS) \cite{DBLP:conf/nips/KiniPOT21}:
\begin{equation}
    \label{eq:vs}
    L_\text{VS}(f(\boldsymbol{x}), y) = \alpha_y \cdot L_\text{CE}(\boldsymbol{\beta} \odot  f(\boldsymbol{x}) + \boldsymbol{\Delta}, y),
\end{equation}
where $L_\text{CE}(f(\boldsymbol{x}), y) = - \log \left( \texttt{softmax}(f(\boldsymbol{x})_y) \right)$ is the traditional CE loss \cite{10.5555/2371238}; $\texttt{softmax}$ denotes the softmax function; $\alpha_y \in \mathbb{R}$ is the re-weighting term; $\boldsymbol{\beta}, \boldsymbol{\Delta} \in \mathbb{R}^C$ are additive and multiplicative adjustment terms, respectively; $\odot$ denotes the element-wise product.
This loss family generalizes a broad range of existing loss functions. For example, when $\boldsymbol{\beta} = \boldsymbol{1}, \boldsymbol{\Delta} = \boldsymbol{0}$, $\alpha_y = \pi_y^{-1}$ and $\alpha_y = (1 - p) / (1 - p^{N_y}), p \in (0, 1)$ recover the classic balanced loss \cite{DBLP:conf/icml/MorikBJ99} and the Class-Balanced (CB) loss \cite{DBLP:conf/cvpr/CuiJLSB19}, respectively, where $\pi_y := N_y / N$ denotes the ratio of class $y$. Given $\alpha_y = 1, \boldsymbol{\beta} = \boldsymbol{1}$, $\Delta_y = \tau \log \pi_y, \tau > 0$ yields the Logit-Adjusted (LA) loss \cite{DBLP:conf/iclr/MenonJRJVK21}. When $\alpha_y = 1, \beta_y = (N_y / N_1)^{\gamma}, \boldsymbol{\Delta} = \boldsymbol{0}, \gamma > 0$, it beomces the CDT loss \cite{DBLP:journals/corr/abs-2001-01385}.

\begin{figure*}[t]
    \centering
    \includegraphics[width=0.9\textwidth]{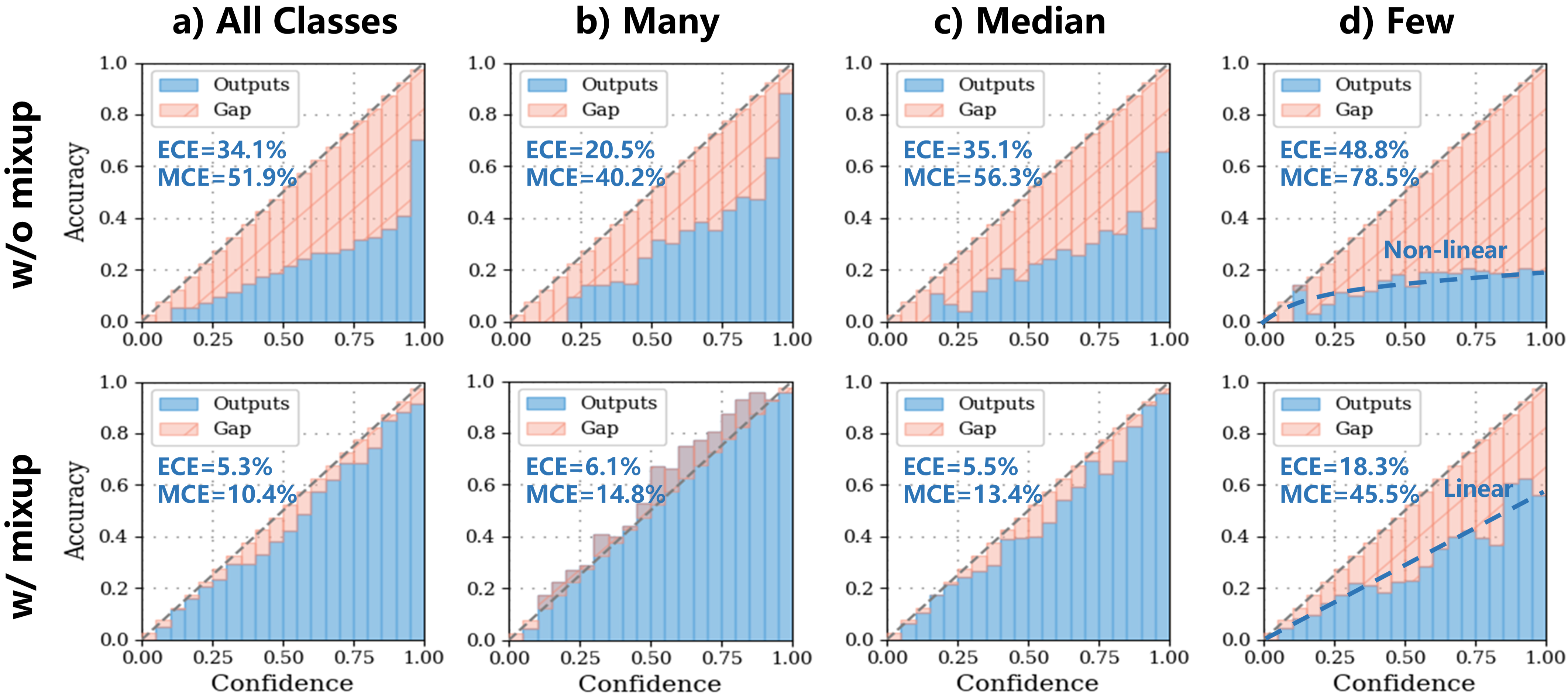}
    \caption{The calibration issue of minority classes: Although the mixup technique can significantly improve the overall/many/median calibration (\textbf{a, b, c}, respectively), the model still exhibits poor calibration in the few classes (\textbf{d}).}
    \label{fig:local_calibration}
\end{figure*}

\section{Consistency Analysis}
\label{sec:consistency}
Given Eq.(\ref{eq:vs}), how to choose proper $\alpha, \boldsymbol{\beta}, \boldsymbol{\Delta}$? To answer this question, in Sec.\ref{subsec:existing_analysis_consistency}, we review the existing consistency analysis, which are based on the global assumption. Then, fine-grained empirical results in Sec.\ref{subsec:calibration_issue} show that the minority classes are not well calibrated, even if the model has been globally calibrated. On top of this, we study the consistency property under localized calibration assumptions in Sec.\ref{subsec:consistency_under_calibration}.

\subsection{Consistency under Global Calibration}
\label{subsec:existing_analysis_consistency}
Formally, Fisher consistency requires that optimizing the modified loss $L$ can simultaneously induce a small balanced risk $\mathcal{R}_\text{bal}(f)$ \cite{10.5555/2371238,DBLP:conf/iclr/MenonJRJVK21}:
\begin{definition}[Fisher Consistency]
    The loss function $L$ is Fisher consistent if for any sequence $\{ f^{(t)} \}_{t=1}^{\infty}$, when $t \to \infty$,
    $$\begin{aligned}
        \mathcal{R}_\text{bal}^L (f^{(t)}) \to \min_f \mathcal{R}_\text{bal}^L(f) \implies \mathcal{R}_\text{bal} (f^{(t)}) \to \min_f \mathcal{R}_\text{bal}(f).
    \end{aligned}$$
\end{definition}

\noindent In imbalanced learning, \citet{DBLP:conf/iclr/MenonJRJVK21} point out that under the following global assumption, a subset of the VS loss family can achieve Fisher consistency:
\begin{assumption}[Global Calibration]
    \label{asm:global}
    The posterior probability is proportional to the model confidence for all classes, \textit{i.e.}, $\forall (\boldsymbol{x}, y) \in \mathcal{Z}, \mathbb{P}\left[ y \mid \boldsymbol{x} \right] \propto \exp(f(\boldsymbol{x})_y) \propto \texttt{softmax}(f(\boldsymbol{x})_y)$.
\end{assumption}
\begin{proposition}[\cite{DBLP:conf/iclr/MenonJRJVK21}]
    \label{prop:la_fisher}
    Under Asm.\ref{asm:global}, the VS loss is Fisher consistent for any constants $\boldsymbol{\delta} \in \mathbb{R}_+^C$, if 
    $\alpha_y = \delta_y / \pi_y, \boldsymbol{\beta} = \boldsymbol{1}, \Delta_y = \log \delta_y$.
\end{proposition}
    This proposition shows that the LA loss ($\delta_y = \pi_y$) and the classic balanced loss ($\delta_y = 1$) are Fisher consistent. However, Asm.\ref{asm:global} might be unattainable in practice since models trained on imbalanced data tend to be more miscalibrated \cite{DBLP:conf/cvpr/ZhongC0J21,DBLP:conf/nips/XuCY21,DBLP:conf/cvpr/AimarJFK23}. Although some techniques, such as mixup \cite{DBLP:conf/iclr/ZhangCDL18} can improve the overall calibration \cite{DBLP:conf/cvpr/AimarJFK23}, we find that \textbf{different class groups generally exhibit different calibration performance}, which we will elaborate on in the next subsection.

\subsection{Local Calibration Issue}
\label{subsec:calibration_issue}
One popular approach to measure model calibration is to plot reliability diagram \cite{DBLP:conf/icml/GuoPSW17,DBLP:conf/icml/Niculescu-MizilC05}. Specifically, the model predictions $f(\boldsymbol{x})$ are first divided into $M$ interval bins $\{\mathcal{B}_i\}_{i=1}^{M}$, where the $i$-th bin $\mathcal{B}_i$ contains the instances whose confidence ranges in $( \frac{i-1}{M}, \frac{i}{M} ]$. Then, the confidence and accuracy of $\mathcal{B}_i$ are defined as:
\begin{equation}
    \begin{aligned}
        \texttt{Acc}(\mathcal{B}_i) & = \frac{1}{|\mathcal{B}_i|} \sum_{\boldsymbol{x}^{(n)} \in \mathcal{B}_i} \texttt{Acc}(f(\boldsymbol{x}^{(n)})), \\
        \texttt{Conf}(\mathcal{B}_i) &= \frac{1}{|\mathcal{B}_i|} \sum_{\boldsymbol{x}^{(n)} \in \mathcal{B}_i} \texttt{Conf}(f(\boldsymbol{x}^{(n)})).
    \end{aligned}
\end{equation}
where the model confidence is defined by 
\begin{equation}
    \texttt{Conf}(f(\boldsymbol{x})) := \max_{y \in \mathcal{Y}} \texttt{softmax}(f(\boldsymbol{x})_y).
\end{equation}

\noindent If a model is perfectly calibrated, the diagram should be an identity function, \textit{i.e.}, $\forall i \in [1, M], \texttt{Acc}(\mathcal{B}_i) = \texttt{Conf}(\mathcal{B}_i)$. Motivated by this, Expected Calibration Error (ECE) and Maximum Calibration Error (MCE) \cite{DBLP:conf/aaai/NaeiniCH15,DBLP:conf/icml/GuoPSW17} are two widely-used numerical metrics, defined by the mean and maximal calibration error over all bins, respectively:
\begin{equation}
    \begin{aligned}
        &\texttt{ECE}(f) = \sum_{i=1}^M\frac{|\mathcal{B}_i|}{N}|\texttt{Acc}(\mathcal{B}_i) - \texttt{Conf}(\mathcal{B}_i)|, \\
        &\texttt{MCE}(f) = \max_{i\in\{1,\cdots,B\}}|\texttt{Acc}(\mathcal{B}_i) - \texttt{Conf}(\mathcal{B}_i)|.
    \end{aligned}
\end{equation}

In Fig.\ref{fig:local_calibration}, we report the calibration results of a ResNet-32 \cite{DBLP:conf/cvpr/HeZRS16} model trained on the CIFAR-100 LT dataset \cite{krizhevsky2009learning} with the VS loss ($\alpha_y = 1, \beta_y = (N_y / N_1)^{\gamma}, \Delta_y = \tau \log \pi_y$). The classes are divided into three categories according to the number of samples: many, median, and few, following the protocol in \cite{DBLP:conf/cvpr/0002MZWGY19,DBLP:conf/icml/Shi00SH024}. From the results, we have the following observations: 1) Due to the imbalance nature, the model without mixup exhibits poor calibration, especially in the minority classes, which is consistent with the findings in \cite{DBLP:conf/cvpr/ZhongC0J21,DBLP:conf/nips/XuCY21,DBLP:conf/cvpr/AimarJFK23}. 2) The mixup technique significantly improves the overall calibration, which is also consistent with the observation in \cite{DBLP:conf/cvpr/AimarJFK23}. 3) However, even with mixup, the model still exhibits poor calibration in the minority classes, which has not been reported. We should pay attention to this issue since it is closely related to the assumption of consistency.

\subsection{Consistency under Local Calibration}
\label{subsec:consistency_under_calibration}

Faced with this issue, one intuitive strategy is to calibrate the predictions with off-the-shelf techniques \cite{DBLP:conf/icml/GuoPSW17,DBLP:conf/aistats/BertaBJ24}. However, these strategies are rather unfavorable for minority classes. Specifically, calibration methods generally require holding out a portion of data, which can hinder the generalization of minority classes. 
Hence, we turn to investigate the consistency property under localized calibration assumptions (\textit{i.e.}, Asm.\ref{asm:local_linear} and Asm.\ref{asm:local_nonlinear}). Surprisingly, such analysis also provides new theoretical insights for the role of multiplicative terms.

The first assumption directly localizes Asm.\ref{asm:global}. That is, different classes have different slopes $\kappa_y^+$ in the reliability diagram, as shown in Fig.\ref{fig:local_calibration}.

\begin{assumption}[Localized Calibration I]
    \label{asm:local_linear}
    The posterior probability is \textbf{locally} proportional to the model confidence \textit{i.e.}, 
    $\forall (\boldsymbol{x}, y) \in \mathcal{Z}, \mathbb{P}\left[ y \mid \boldsymbol{x} \right] \propto \kappa_y^+ \texttt{softmax}(f(\boldsymbol{x})_y), \kappa_y^+ \in \mathbb{R}$.
\end{assumption}

Accordingly, the loss family defined in Prop.\ref{prop:la_fisher} should be further adjusted by $\kappa_y^+$, whose proof can be found in Appendix \ref{subsec:proof_local_linear_loss}.

\begin{restatable}{proposition}{locallinearloss}
    \label{prop:local_linear_loss}
    Under Asm.\ref{asm:local_linear}, the VS loss is Fisher consistent for any constants $\boldsymbol{\delta} \in \mathbb{R}_+^C$, if 
    $\alpha_y = \delta_y / \pi_y, \boldsymbol{\beta} = \boldsymbol{1}, \Delta_y = \log \delta_y / \kappa_y^+$.
\end{restatable}

\noindent When $\delta_y = \pi_y$, the induced loss function has an additional calibration-related term compared to the LA loss:  
\begin{equation}
    \label{eq:cla}
    L_\text{CLA}(f(\boldsymbol{x}), y) := L_\text{CE}(f(\boldsymbol{x})_y + \tau \log (\pi_y / \kappa_y^+)), 
\end{equation}
where we omit $y$ in the CE loss for conciseness.

In the second assumption, the posterior probability is proportional to the logit, rather than the model confidence. 
\begin{assumption}[Localized Calibration II]
    \label{asm:local_nonlinear}
    The posterior probability is \textbf{locally} proportional to the model prediction, \textit{i.e.}, 
    $\forall (\boldsymbol{x}, y) \in \mathcal{Z}, \mathbb{P}\left[ y \mid \boldsymbol{x} \right] \propto \kappa_y^* f(\boldsymbol{x})_y, \kappa_y^* \in \mathbb{R}$.
\end{assumption}
This assumption suggests that the model performance is proportional to the logarithm of model confidence, instead of being linear. This aligns with the observation in Fig.\ref{fig:local_calibration} and Fig.\ref{fig:local_calibration_exp}, where samples with high confidence tend to have comparable accuracy, especially for the minority classes. Similarly, we have another consistent loss family, whose proof can be found in Appendix \ref{subsec:proof_local_nonlinear_loss}.

\begin{restatable}{proposition}{localnonlinearloss}
    \label{prop:local_nonlinear_loss}
    Under Asm.\ref{asm:local_nonlinear}, the VS loss is Fisher consistent for any constants $\boldsymbol{\delta} \in \mathbb{R}_+^C$, if 
    $\alpha_y = \delta_y / \pi_y, \beta_y = \delta_y / \kappa_y^*, \boldsymbol{\Delta} = \boldsymbol{0}$.
\end{restatable}

\noindent Interestingly, although Asm.\ref{asm:local_nonlinear} slightly differs from Asm.\ref{asm:local_linear}, Prop.\ref{prop:local_nonlinear_loss} provides a consistency insight for the multiplicative term, instead of the additional ones. To the best of our knowledge, \textbf{this is the first time that the multiplicative term is justified by consistency analysis}, since existing works generally focus on its effect on feature space \cite{DBLP:journals/corr/abs-2001-01385,hasegawa2024multi}. When $\delta_y = \pi_y$, this proposition induces a Multiplicative Logit-Adjusted loss with consistency guarantee:
\begin{equation}
    \label{eq:mla}
    L_\text{MLA}(f(\boldsymbol{x}), y) = L_\text{CE}(f(\boldsymbol{x})_y \cdot \pi_y^{\gamma} / \kappa_y^*),
\end{equation}
where $\gamma > 0$ is the temperature hyperparameter, similar to that in the CDT loss \cite{DBLP:journals/corr/abs-2001-01385}.

Combining Prop.\ref{prop:local_linear_loss} and Prop.\ref{prop:local_nonlinear_loss}, we can see that the consistency of the VS loss family is guaranteed under different calibration assumptions. To facilitate the following discussion, we define the subset with Consistency Guarantee as:
\begin{equation}
    \label{eq:cvs}
    L_\text{CVS}(f(\boldsymbol{x}), y) = \alpha_y \cdot L_\text{CE}(\boldsymbol{\beta}^* \odot  f(\boldsymbol{x}) + \boldsymbol{\Delta}^+, y),
\end{equation} 
where $\beta_y^* := \pi_y^{\gamma} / \kappa_y^*$ and $\Delta_y^+ := \tau \log (\pi_y / \kappa_y^+)$ follow the definition in Eq.(\ref{eq:cla}) and Eq.(\ref{eq:mla}), respectively.

\section{Generalization Analysis}
\label{sec:generalization}
So far, consistency analysis has suggested proper class-dependent terms $\alpha_y, \boldsymbol{\beta}, \boldsymbol{\Delta}$. However, we cannot directly minimize the CVS loss since 1) their consistency holds under different assumptions, and 2) their generalization property is still a mystery. Fortunately, a systematic analysis of \textbf{(RQ2)} can provide some clues. Next, in Sec.\ref{subsec:existing_analysis_generalization}, we first show that existing generalization analysis is coarse-grained, failing to provide inspiring insights. Then, Sec.\ref{subsec:data_dependent_analysis} provides a fine-grained generalization analysis based on the localization technique. Finally, we provide a generalization guarantee for the VS loss in Sec.\ref{subsec:application}.

\subsection{Coarse-Grained Generalization Analysis}
\label{subsec:existing_analysis_generalization}
For generalization analysis, \citet{DBLP:conf/nips/CaoWGAM19} and \citet{DBLP:conf/nips/RenYSMZYL20} aggregate the class-wise generalization bound directly with a union bound over class-wise results \cite{10.5555/2371238}:
\begin{proposition}[Union Bound for Imbalanced Learning \cite{DBLP:conf/nips/CaoWGAM19}]
    \label{prop:ldam_bound}
    Given the function set $\mathcal{F}$ and a loss function $L: \mathbb{R}^C \times \mathcal{Y} \to [0, M]$, let $\mathcal{G} := \{L \circ f: f \in \mathcal{F}\}$. Then, for any $\delta \in (0, 1)$, with probability at least $1 - \delta$ over the training set $\mathcal{S}$, the following generalization bound holds for all $g \in \mathcal{G}$:
    \begin{equation}
        \label{eq:ldam_bound}
        \begin{aligned}
            & \mathcal{R}_\text{bal}^L (f) = \frac{1}{C} \sum_{y=1}^C \mathcal{R}_y^L(f) \\
            & \precsim \frac{1}{C} \sum_{y=1}^C \left( \hat{\mathcal{R}}_y^L(f) + \hat{\mathfrak{C}}_{\mathcal{S}_y}(\mathcal{G}) + 3 M \sqrt{\frac{ \log 2 C / \delta}{2 N_y}} \right),
        \end{aligned}
    \end{equation}
    where $\hat{\mathcal{R}}_y^L(f)$ is the empirical risk defined on $\mathcal{S}_y$; 
    $$
        \hat{\mathfrak{C}}_{\mathcal{S}}(\mathcal{G}) := \Ee{\boldsymbol{\xi}}{\sup_{g \in \mathcal{G}} \frac{1}{N}\sum_{n=1}^{N} \xi^{(n)} g(\boldsymbol{z}^{(n)}) }
    $$
    denotes the empirical complexity of the function set $\mathcal{G}$, and $\boldsymbol{\xi} := ( \xi^{(1)}, \xi^{(2)}, \cdots, \xi^{(N)} )$ are sampled from independent distributions such as the uniform distribution with $\{1, -1\}$; $\precsim$ denotes the asymptotic notation that omits undominated terms, that is,
    $$
        f(t) \precsim g(t) \Longleftrightarrow  \exists \text{ a constant } C > 0, f(t) \le C \cdot g(t).
    $$
\end{proposition}

To further bound the complexity term $\hat{\mathfrak{C}}_{\mathcal{S}_y}(\mathcal{G})$, \citet{DBLP:conf/nips/CaoWGAM19} apply the traditional contraction lemma \cite{ledoux1991probability} with the assumption that $L(f, y)$ is Lipschitz continuous \cite{ledoux1991probability}:
\begin{definition}[Lipschitz Continuity]
    \label{def:lipschitz}
    Let $\Vert \cdot \Vert$ denote the 2-norm. Then, we say the loss function $L(f, y)$ is Lipschitz continuous with constant $\mu$ if for any $f, f' \in \mathcal{F}$, $\boldsymbol{x} \in \mathcal{S}$,
    \begin{equation}
        \label{eq:lipschitz}
        | L(f, y) - L(f', y) | \le \mu \cdot \Vert f(\boldsymbol{x}) - f'(\boldsymbol{x}) \Vert.
    \end{equation}
\end{definition}
\begin{lemma}[Contraction Lemma]
    \label{lem:contract_lemma}
    Assume that the loss function $L(f, \boldsymbol{x})$ is Lipschitz continuous with a constant $\mu$. Then, the following inequality holds:
    \begin{equation}
        \hat{\mathfrak{C}}_{\mathcal{S}}(\mathcal{G}) \le \mu \cdot \hat{\mathfrak{C}}_{\mathcal{S}}(\mathcal{F}).
    \end{equation}
\end{lemma}

Finally, the upper bound of $\hat{\mathfrak{C}}_{\mathcal{S}_y}(\mathcal{F})$ is available by the standard margin-based generalization bound \cite{DBLP:conf/nips/KakadeST08}. However, this union bound has the following limitations: 
    \begin{itemize}[leftmargin=*]
    \item \textbf{(L1)} The empirical risk $\frac{1}{C} \sum_{y=1}^C \hat{\mathcal{R}}_y^L(f)$ is not a good approximation of the training risk $\hat{\mathcal{R}}^L(f)$, \textit{i.e.}, Eq.(\ref{eq:ldam_bound}), due to the absence of the class distribution.
    \item \textbf{(L2)} Theoretically, this generalization bound is coarse-grained and not sharp enough. Specifically, the differences among loss functions lie in the choice of class-dependent terms $\alpha_y, \boldsymbol{\beta}, \boldsymbol{\Delta}$. However, the Lipschitz continuity, which is the only property of $L$ used in the proof, is global in nature and thus obscures these differences.
    \item \textbf{(L3)} Empirically, although the induced LDAM loss outperforms the CE loss,  the improvement is not so significant. Fortunately, when combining the Deferred Re-Weighting (DRW) technique \cite{DBLP:conf/nips/CaoWGAM19}, where $\alpha_y = (1 - p) / (1 - p^{N_y}), p \in (0, 1)$ \cite{DBLP:conf/cvpr/CuiJLSB19} during the terminal phase of training, the improvement becomes impressive. However, Eq.(\ref{eq:ldam_bound}) fails to explain this phenomenon.
\end{itemize}

Besides, \citet{DBLP:conf/nips/KiniPOT21} provide a generalization analysis for the VS loss. However, the results, which only hold for linear models with linearly separable data, can only explain the roles of $\boldsymbol{\beta}$. In a nutshell, existing generalization analysis is still coarse-grained and fragmented, failing to provide systematic insights.

\subsection{Data-Dependent Generalization Analysis}
\label{subsec:data_dependent_analysis}
    Different from Eq.(\ref{eq:ldam_bound}), we hope to build a direct bound between $\mathcal{R}_\text{bal}^L (f)$ and $\widehat{\mathcal{R}}^L (f)$, as a solution to the limitation \textbf{(L1)}. To this end, our analysis is based on the following lemma, whose proof can be found in Appendix \ref{app:basic_lem}:
\begin{restatable}{lemma}{basiclem}
    \label{lem:basic_lem}
    Given the function set $\mathcal{F}$ and a loss function $L: \mathbb{R}^C \times \mathcal{Y} \to [0, M]$, then for any $\delta \in (0, 1)$, with probability at least $1 - \delta$ over the training set $\mathcal{S}$, the following generalization bound holds for all $g \in \mathcal{G}$:
    \begin{equation}
        \label{eq:basic_lem}
        \mathcal{R}_\text{bal}^L (f) \precsim  \Phi(L, \delta) + \frac{1}{C \pi_C} \cdot \hat{\mathfrak{C}}_{\mathcal{S}}(\mathcal{G}),
    \end{equation}
    where $\Phi(L, \delta) := \frac{1}{C \pi_C} [ \widehat{\mathcal{R}}^L(f) + 3 M \sqrt{\frac{\log 2 / \delta}{2 N}} ]$ contains the empirical risk on $\mathcal{S}$ and the $\delta$ term. 
\end{restatable}
\begin{remark}
    Recall that $\pi_C := N_C / N, N_1 \ge N_2 \ge \cdots \ge N_C$. Hence, this lemma reveals how the model performance depends on the degree of imbalance in the data. When the dataset becomes extremely imbalanced, \textit{i.e.}, $\pi_C \to 0$, the generalization bound becomes rather loose.
\end{remark}

    As shown in Sec.\ref{subsec:existing_analysis_generalization}, the fine-grained analysis is unavailable due to the global nature of the classic Lipschitz continuous property. In view of this, we extend this traditional definition with the localization technique, as a solution to the limitation \textbf{(L2)}.
\begin{definition}[Local Lipschitz Continuity]
    A loss function $L(f, y)$ is \textbf{locally} Lipschitz continuous with constants $\{\mu_y\}_{y = 1}^C$ if for any $f, f' \in \mathcal{F}, y \in \mathcal{Y}$, $\boldsymbol{x} \in \mathcal{S}_y$, 
    \begin{equation}
        | L(f, y) - L(f', y) | \le \mu_y \cdot \Vert f(\boldsymbol{x}) - f'(\boldsymbol{x}) \Vert.
    \end{equation}
\end{definition}
\noindent Then, the following data-dependent contraction inequality helps us obtain a sharper bound, whose proof is given in Appendix \ref{app:data_dependent_contraction}.

\begin{assumption}
    \label{asm:complexity}
    We assume that $\hat{\mathfrak{C}}_\mathcal{S}(\mathcal{F}) \sim \mathcal{O}( 1 / \sqrt{N} )$. Note that this result holds for kernel-based models with traditional techniques \cite{10.5555/2371238} and neural networks with the latest techniques \cite{DBLP:conf/iclr/LongS20}. And the prior arts also adopt this assumption \cite{DBLP:conf/nips/CaoWGAM19,DBLP:conf/nips/WangX00CH23}.
\end{assumption}

\begin{restatable}[Data-Dependent Contraction]{lemma}{datadeplem}
    \label{lem:data_dependent_contraction}
    Assume that the loss function $L(f, \boldsymbol{x})$ is locally Lipschitz continuous with constants $\{\mu_{y}\}_{y = 1}^C$. Then, the following inequality holds under Asm.\ref{asm:complexity}:
    \begin{equation}
        \hat{\mathfrak{C}}_{\mathcal{S}}(\mathcal{G}) \precsim \hat{\mathfrak{C}}_{\mathcal{S}}(\mathcal{F}) \sum_{y=1}^{C} \mu_{y} \sqrt{\pi_y},
    \end{equation}
\end{restatable}

Combining Lem.\ref{lem:basic_lem} and Lem.\ref{lem:data_dependent_contraction}, we have the following theorem. 
\begin{theorem}[Data-Dependent Bound for Imbalanced Learning]
    \label{thm:main}
    Given the function set $\mathcal{F}$ and a bounded loss function $L$, for any $\delta \in (0, 1)$, with probability at least $1 - \delta$ over the training set $\mathcal{S}$, the following generalization bound holds for all $f \in \mathcal{F}$:
    \begin{equation}
        \label{eq:main}
        \mathcal{R}_\text{bal}^L (f) \precsim  \Phi(L, \delta) + \frac{\hat{\mathfrak{C}}_{\mathcal{S}}(\mathcal{F})}{C \pi_C} \sum_{y=1}^{C} \mu_{y} \sqrt{\pi_y}.
    \end{equation}
\end{theorem}

    According to Def.\ref{def:lipschitz}, we have $\mu = \max_y \mu_y$. Hence, Thm.\ref{thm:main} is strictly sharper than the traditional generalization bounds that are built upon the global Lipschitz continuity. Besides, when $\mu_y$ is decreasing \textit{w.r.t.} $\pi_y$, this bound becomes sharper since the majority classes are assigned with smaller weigthts. This can explain the design of existing losses, which we will elaborate in the next subsection.

\begin{figure*}[t]
    \centering
    \subfigure[CIFAR-100 LT ($\rho = 100$)]{
        \includegraphics[width=0.28\linewidth]{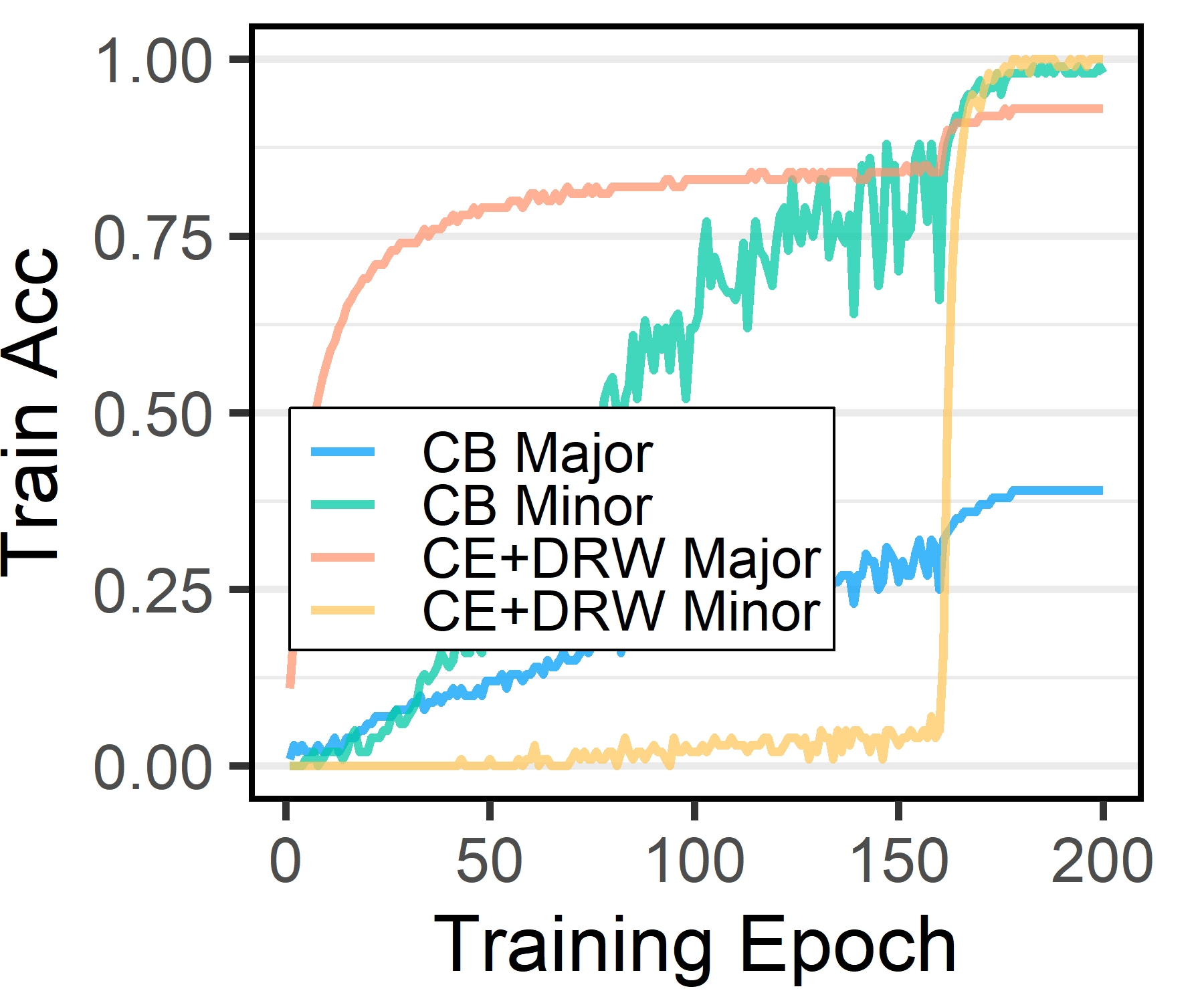}
    }
    \subfigure[CIFAR-100 LT ($\rho = 100$)]{
    \includegraphics[width=0.28\linewidth]{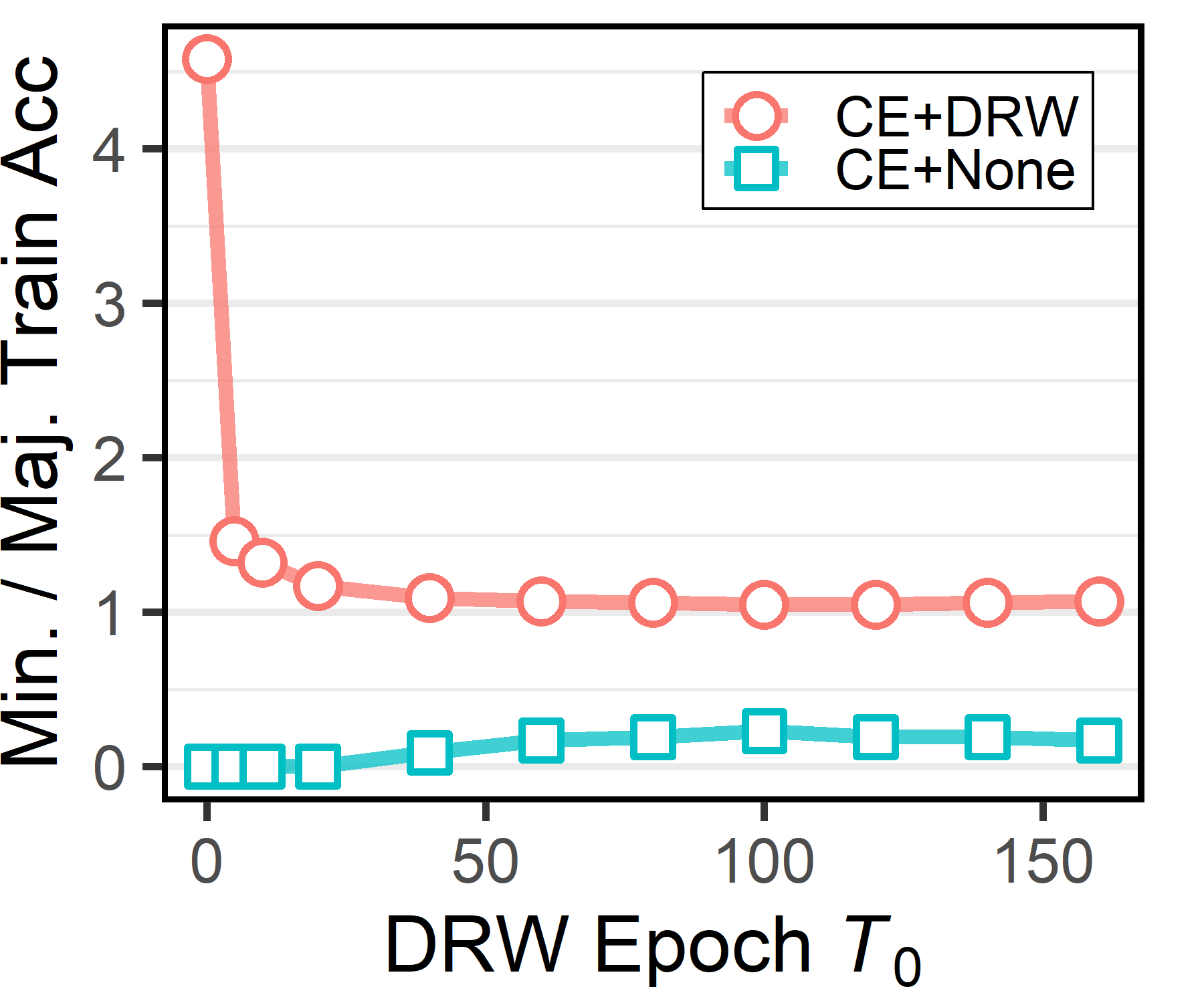}
    }
    \subfigure[CIFAR-100 LT ($\rho = 100$)]{
      \includegraphics[width=0.28\linewidth]{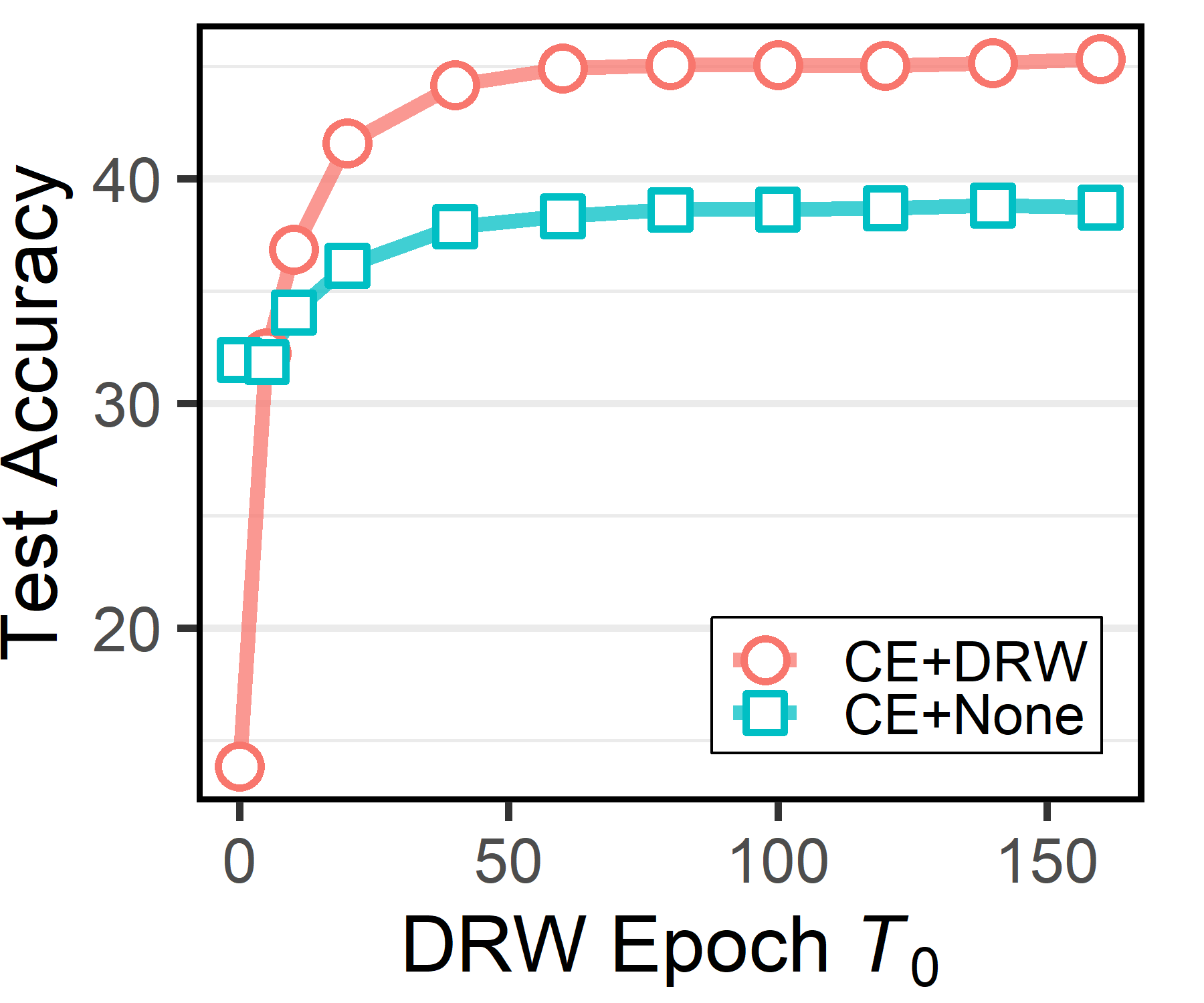}
    }
    \caption{(a) Training accuracy of CE+DRW ($T_0 = 160$) and the CB loss \textit{w.r.t.} training epoch. (b) $\widehat{Acc}_\text{min} / \widehat{Acc}_\text{maj}$ \textit{w.r.t.} the DRW epoch $T_0$, where $\widehat{Acc}_\text{min}$ and $\widehat{Acc}_\text{maj}$ denote the training accuracy of the best model on the minority/majority classes, respectively. (c) The test accuracy of the best model \textit{w.r.t.} the DRW epoch $T_0$. We can find that the DRW scheme balances the training accuracy between the majority classes and the minority classes and thus improves the model performance on the test set, which is consistent with the theoretical insight \textbf{(In2)}.}
    \label{fig:drw_cifar100}
\end{figure*}

\subsection{Application to the VS Loss}
\label{subsec:application}
Next, we apply Thm.\ref{thm:main} to the VS loss to reveal the role of both re-weighting and logit-adjustment. To this end, it is necessary to analyze the local Lipschitz property of the VS loss, whose proof is presented in Appendix \ref{app:Lip_of_vs}.
\begin{restatable}{lemma}{lipofmargin}
    \label{lem:Lip_of_vs}
    Assume that the score function is bounded. Then, the VS loss is locally Lipschitz continuous with constants 
    \begin{equation}
        \label{eq:Lip_of_vs}
        \mu_{y} = \alpha_{y} \tilde{\beta} \left[ 1 - \text{softmax}\left( \beta_{y} B_y(f) + \Delta_{y} \right) \right],
    \end{equation}
    where $\tilde{\beta} := \sqrt{\sum_{y'=1}^{C} \beta_{y'}^2}$, and
    $B_y(f)$ denotes the minimal prediction on the ground-truth class $y$, i.e., $B_y(f) := \min_{\mathbf{x} \in S_y} f(\boldsymbol{x})_y$.
\end{restatable}

\begin{remark}
    $B_y(f)$ is closely related to the minimal margin defined by $\text{margin}_y^\downarrow := \min_{\mathbf{x} \in S_y} (f(\boldsymbol{x})_y - \max_{j \neq y} f(\boldsymbol{x})_j)$. It is not difficult to check that under mild assumptions,
    $$
        0 \le B_y(f) - \text{margin}_y^\downarrow \le \max_{\mathbf{x} \in S_y, j \neq y} f(\boldsymbol{x})_j.
    $$ 
    Hence, as we improve the model performance on class $y$, the RHS of the above inequality, i.e., the gap between $B_y(f)$ and $\text{margin}_y^\downarrow$ will decrease. In other words, both the minimal margin and $B_y(f)$ will increase.
\end{remark}

\begin{remark}
    The local Lipschitz constant defined in eq.(\ref{eq:Lip_of_vs}) is tighter and more concise than our conference version \cite{DBLP:conf/nips/WangX00CH23}, where $\tilde{\beta}$ is replaced by $\sqrt{\beta_y^2 + \left( \sum_{y' \neq y} \beta_{y'} \right)^2}$. \textbf{According to this modification, we have updated the theoretical insight (In4) in the following discussion}, where the negative impact of the multiplicative adjustment term is further highlighted.
\end{remark}

Then, combining Thm.\ref{thm:main} and Lem.\ref{lem:Lip_of_vs}, we have the following proposition, which reveals how loss-oriented methods improve generalization performance by exploiting the data priors.
\begin{proposition}[Data-Dependent Bound for the VS Loss]
    \label{prop:vs_generalization}
    Given the function set $\mathcal{F}$ and the VS loss  $L_\text{VS}$, for any $\delta \in (0, 1)$, with probability at least $1 - \delta$ over the training set $\mathcal{S}$, the following generalization bound holds for all $f \in \mathcal{F}$:
    \begin{equation}
        \label{eq:bound_margin}
        \begin{aligned}
            & \mathcal{R}_\text{bal}^L (f) \precsim \Phi(L_\text{VS}, \delta) + \\
            & \frac{\tilde{\beta} \cdot \hat{\mathfrak{C}}_{\mathcal{S}}(\mathcal{F})}{C \pi_C} \sum_{y=1}^{C} \alpha_{y} \sqrt{\pi_y} \left[ 1 - \text{softmax}\left( \beta_{y} B_y(f) + \Delta_{y} \right) \right].
        \end{aligned}
    \end{equation}
\end{proposition}

From Eq.(\ref{eq:bound_margin}), we have the following insights. Notably, \textbf{(In1)-(In3)} focus on the design of existing methods, while \textbf{(In4)} foretells the negative impact of the multiplicative adjustment term. The empirical validation of these insights can be found in Sec.\ref{subsec:theory_validation}.

\textbf{(In1) Why re-weighting and logit-adjustment are necessary?} Due to $\sqrt{\pi_y}$ and $B_y(f)$, the generalization bound is also imbalanced among classes. Both re-weighting and logit-adjustment can induce a sharper generalization bound by assigning lower weights to the classes with larger $\sqrt{\pi_y}$ and $B_y(f)$.

    \textbf{(In2) Why the deferred scheme is necessary?} As pointed out in \cite{DBLP:conf/cvpr/CuiJLSB19,DBLP:conf/cvpr/HuangLLT16}, weighting up the minority classes will cause difficulties and instability in optimization, especially when the distribution is extremely imbalanced. To fix this issue, \citet{DBLP:conf/nips/CaoWGAM19} develop a deferred scheme, where $\alpha_y = 1$ and $(1 - p) / (1 - p^{N_y}), p \in (0, 1)$ during the initial and terminal phase of training, respectively. Why does this scheme work? Prop.\ref{prop:vs_generalization} can provide some explanations, as a solution to the limitation \textbf{(L3)}.
Specifically, although a weighted loss can boost the optimization of the minority classes, it is harmful to the further improvement of the majority classes, as shown in Fig.\ref{fig:drw_cifar100}. Hence, the majority/minority classes will have relatively small/large $B_y(f)$ respectively, and the generalization bound becomes even looser. By contrast, in the DRW scheme, we have $\alpha_y = 1$ during the initial phase of training. Such a warm-up phase will encourage the model to focus on the majority classes and induce a small $B_y(f)$ for both majority and minority classes after weighing up the minority classes. On top of this, the generalization bound can become sharper, which explains the effectiveness of the deferred scheme.

\textbf{(In3) How does our result explain the design of losses?} On one hand, according to Eq.(\ref{eq:bound_margin}), $\alpha_y$ should decrease as $\pi_y$ increases, which is consistent with the balanced loss with $\alpha_y = \pi_y^{-1}$ \cite{DBLP:conf/icml/MorikBJ99} and $\alpha_y = (1 - p) / (1 - p^{N_y}), p \in (0, 1)$ \cite{DBLP:conf/cvpr/CuiJLSB19}. On the other hand, due to the negative sign, both $\beta_y$ and $\Delta_y$ should increase as $\pi_y$ increases. This insight is consistent with the LDAM loss ($\Delta_y \propto - N_y^{- 1 / 4}$) \cite{DBLP:conf/nips/CaoWGAM19}, the LA loss ($\Delta_y = \tau \log \pi_y$) \cite{DBLP:conf/iclr/MenonJRJVK21}, the ALA loss ($\Delta_y = - 1 / (\log (N_y / N_1 + 1))$) \cite{DBLP:conf/aaai/ZhaoC0HZ22}, the CDT loss ($\beta_y = (N_y / N_1)^{\gamma}$) \cite{DBLP:journals/corr/abs-2001-01385}, and our proposed CVS loss ($\beta_y = \pi_y^{\gamma} / \kappa_y^*, \Delta_y = \tau \log (\pi_y / \kappa_y^+)$). 

\textbf{(In4) Is it completely safe to use these terms?} \textbf{(a)} Unfortunately, for multiplicative logit-adjustment, the answer is negative. To be specific, $\tilde{\beta}$ achieves its minimum value when and only when $\beta_y =1$ for all class $y$, where we assume $\sum_{y=1}^{C} \beta_y = C$ without loss of generality. In other words, although imbalanced $\beta_y$ could help balance the generalization bound among classes, $\tilde{\beta}$ can lead to a looser generalization bound, especially when the generalization bound has been relatively balanced among classes. Combining the insight \textbf{(In2)}, we conjecture that the multiplicative logit-adjustment could be incompatible with the DRW scheme. \textbf{(b)} Fortunately, for re-weighting and additive logit-adjustment, the answer is positive since both terms can generally induce a sharper generalization bound.

\begin{algorithm}[t]
    \setstretch{1.3}
    \caption{The Learning Algorithm for the CVS Loss}
    \label{alg:algorithm}
    \begin{algorithmic}[1]
        \REQUIRE The training set $\mathcal{S} = \{ (\boldsymbol{x}^{(n)}, y^{(n)}) \}_{n = 1}^N$ 
        \REQUIRE A model $f$ parameterized by $\Theta$
        \STATE Initialize the model parameters $\Theta$ randomly
        \FOR{$t = 1, 2, \cdots, T$}
            \FOR{$\mathcal{B}$ in $\text{SampleMiniBatch}(\mathcal{S}, m)$}
                \IF {$t < T_0$} 
                    \STATE Set $\alpha_y = 1, \beta_y = \pi_y^{\gamma} / \hat{\kappa}_y^*, \boldsymbol{\Delta} = \boldsymbol{0}$ \hfill $\triangleright$ MLA
                \ELSE
                    \STATE Set $\alpha_y = 1, \boldsymbol{\beta} = \boldsymbol{1}, \Delta_y = \tau \log (\pi_y / \hat{\kappa}_y^+)$ \hfill $\triangleright$ CLA 
                    \STATE Optional: set $\alpha_y \propto \pi_y^{- \nu}$ \hfill $\triangleright$ ADRW 
                \ENDIF
                \STATE $L(f, \mathcal{B}) \leftarrow \frac{1}{m} \sum_{(\boldsymbol{x}, y) \in \mathcal{B}} L_\text{CVS}(f(\boldsymbol{x}), y)$
                \STATE $\Theta \leftarrow \Theta - \eta \nabla_{\Theta} L(f, \mathcal{B})$
            \ENDFOR
            \STATE Optional: esitimate $\hat{\kappa}_y^*$ or $\hat{\kappa}_y^+$
            \STATE Optional: anneal the learning rate $\eta$
            \ENDFOR
    \end{algorithmic} 
\end{algorithm}

\section{Methdology}
\label{sec:method}
Hitherto, we have provided a systematic theoretical analysis from the perspective of consistency and generalization. Inspired by these theoretical results, we next present a principled learning algorithm for the CVS loss. 
\begin{itemize}
    \item According to \textbf{(In1)-(In3)}, it is crucial to comprehensively utilize re-weighting and logit-adjustment, as they all contribute to improving the generalization bound. When using re-weighting, the DRW scheme is necessary to guarantee a small $B_y(f)$ for both majority and minority classes.
    \item According to the propositions in Sec.\ref{subsec:consistency_under_calibration}, the additive term $\boldsymbol{\Delta}$ and the multiplicative term $\boldsymbol{\beta}$ are consistent under different calibration assumptions. Meanwhile, \textbf{(In4)} indicates that the multiplicative term $\boldsymbol{\beta}$ could be incompatible with the other terms. 
    Combining these insights and the DRW scheme, the MLA loss should be used during the early phase of training. Later, we switch to the CLA loss during the terminal phase of training, optionally equipped with the re-weighting term. Notably, this strategy can be viewed as a special but enhanced version of the TLA scheme proposed in our conference version \cite{DBLP:conf/nips/WangX00CH23}, where $\beta_y = 1$ during the terminal phase of training.
    \item According to Prop.\ref{prop:local_linear_loss} and Prop.\ref{prop:vs_generalization}, we set $\alpha_y \propto \pi_y^{- \nu}, \nu > 0$ to align with $\sqrt{\pi_y}$. This re-weighting term, named Aligned DRW (ADRW), has been proposed in our conference version \cite{DBLP:conf/nips/WangX00CH23}. Notably, its formulation also aligns with the consistency analysis in Sec.\ref{subsec:consistency_under_calibration}. 
    \item To calculate the CVS loss, it is necessary to estimate the calibration terms $\hat{\kappa}_y^*$ and $\hat{\kappa}_y^+$. For efficiency, we perform linear regression on the training set only at the end of each epoch. Recall that $\texttt{Acc}(\mathcal{B}_i)$ and $\texttt{Conf}(\mathcal{B}_i)$ denote the accuracy and the confidence of the $i$-th bin, respectively. Then, according to Asm.\ref{asm:local_linear}, the regression is performed between $\texttt{Acc}(\mathcal{B}_y)$ and $\texttt{Conf}(\mathcal{B}_y)$ to evaluate $\hat{\kappa}_y^+$. For $\hat{\kappa}_y^*$, according to Asm.\ref{asm:local_nonlinear}, we regress $\texttt{Acc}(\mathcal{B}_y)$ and $\texttt{Logit}(\mathcal{B}_y)$, where $\texttt{Logit}(\mathcal{B}_y)$ is the logits averaged within the $i$-th bin.
\end{itemize}

\begin{remark}
    \citet{DBLP:conf/nips/KiniPOT21} point out that the multiplicative term $\boldsymbol{\beta}$ adjusts the optimal solution by moving the decision boundary towards the majority classes, which is consistent with the other work \cite{DBLP:conf/nips/CaoWGAM19,DBLP:journals/corr/abs-2001-01385,hasegawa2024multi}. However, their analysis states that the re-weighting term $\alpha_y$ and the additive term $\boldsymbol{\Delta}$ become ineffective during the terminal phase of training. According to this result, it seems that $\alpha_y$ and $\boldsymbol{\Delta}$ should not be used during the terminal phase of training, which contradicts our method. However, their analysis is limited within linear models trained on linearly separable data for binary classification. By contrast, our analysis can be applied to multiclass classification and non-linear models such as modern neural networks. In Sec.\ref{sec:experiments}, empirical results will also validate the effectiveness of $\alpha_y$ and $\boldsymbol{\Delta}$ during the terminal phase of training.
\end{remark}

Overall, we summarize the learning algorithm in Alg.\ref{alg:algorithm}. In line 5, we use the MLA loss during the early phase of training. Since the calibration terms could be unstable during the early phase of training, we estimate the calibration term $\kappa_y^*$ after several warm-up epochs. In lines 7-9, we switch to the CLA loss in the terminal phase of training, where the re-weighting term $\alpha_y$ is optional. 

\begin{remark}
    Compared with our conference version \cite{DBLP:conf/nips/WangX00CH23}, the CVS loss mainly has two differences. On one hand, the multiplicative logit-adjustment term $(N_y / N_1)^{\gamma}$ is replaced by its consistent counterpart $\pi_y^{\gamma} / \kappa_y^*$, \textit{i.e.}, line 5, according to the theoretical analysis in Sec.\ref{subsec:consistency_under_calibration}. On the other hand, based on the local calibration assumptions, we introduce $\kappa_y^*, \kappa_y^+$ to the additive/multiplicative logit-adjustment terms, respectively, \textit{i.e.}, lines 5, 7, and 13. 
\end{remark}

\begin{figure*}[t]
    \centering
    \subfigure[PEFT ResNet-50 on iNaturalist.]{
    \includegraphics[width=0.31\textwidth]{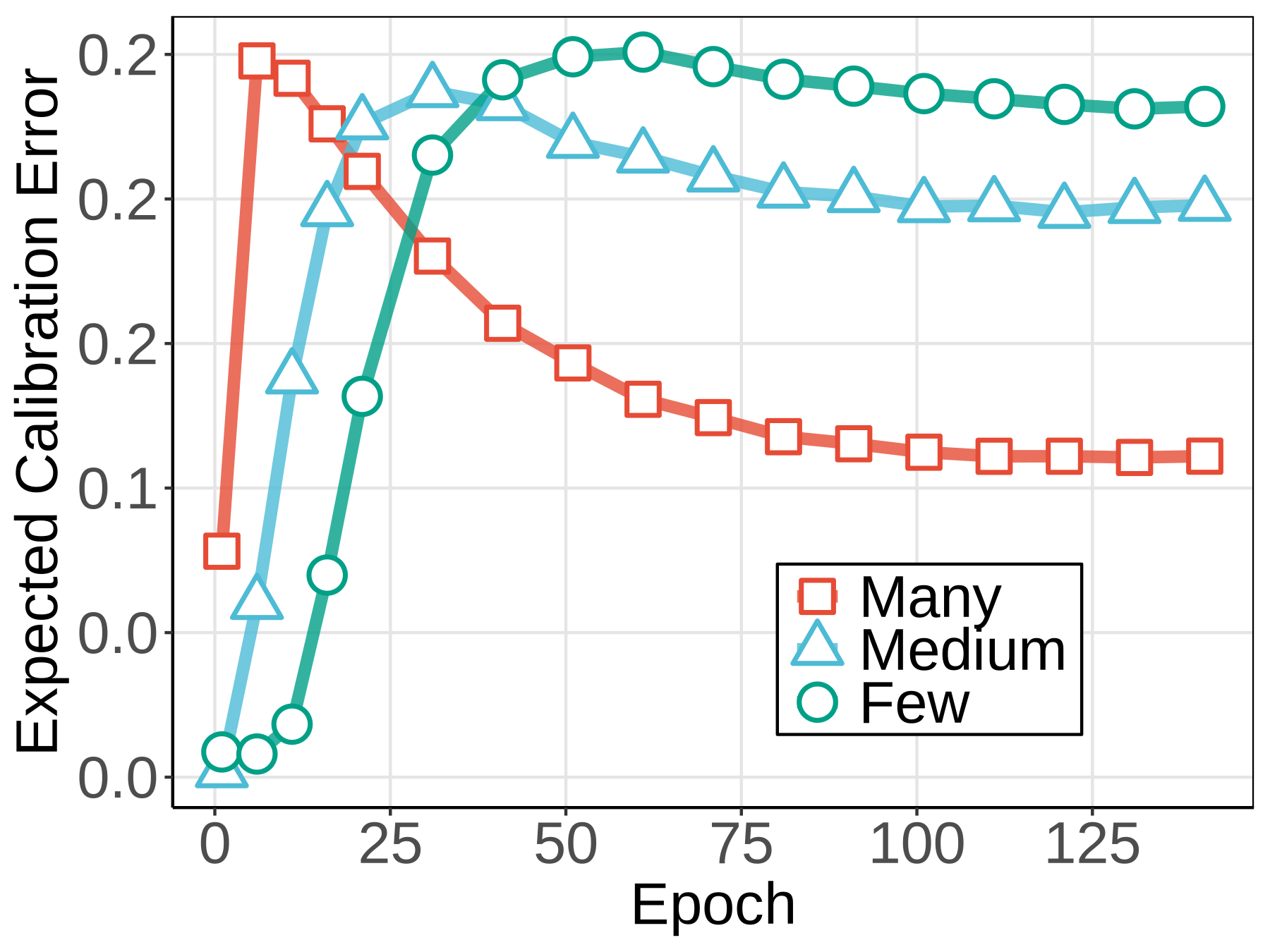} 
    \label{fig:r2q1_resnet50_peft_inat}
    }
    \subfigure[FFT ViT-B/16 on iNaturalist.]{
    \includegraphics[width=0.31\textwidth]{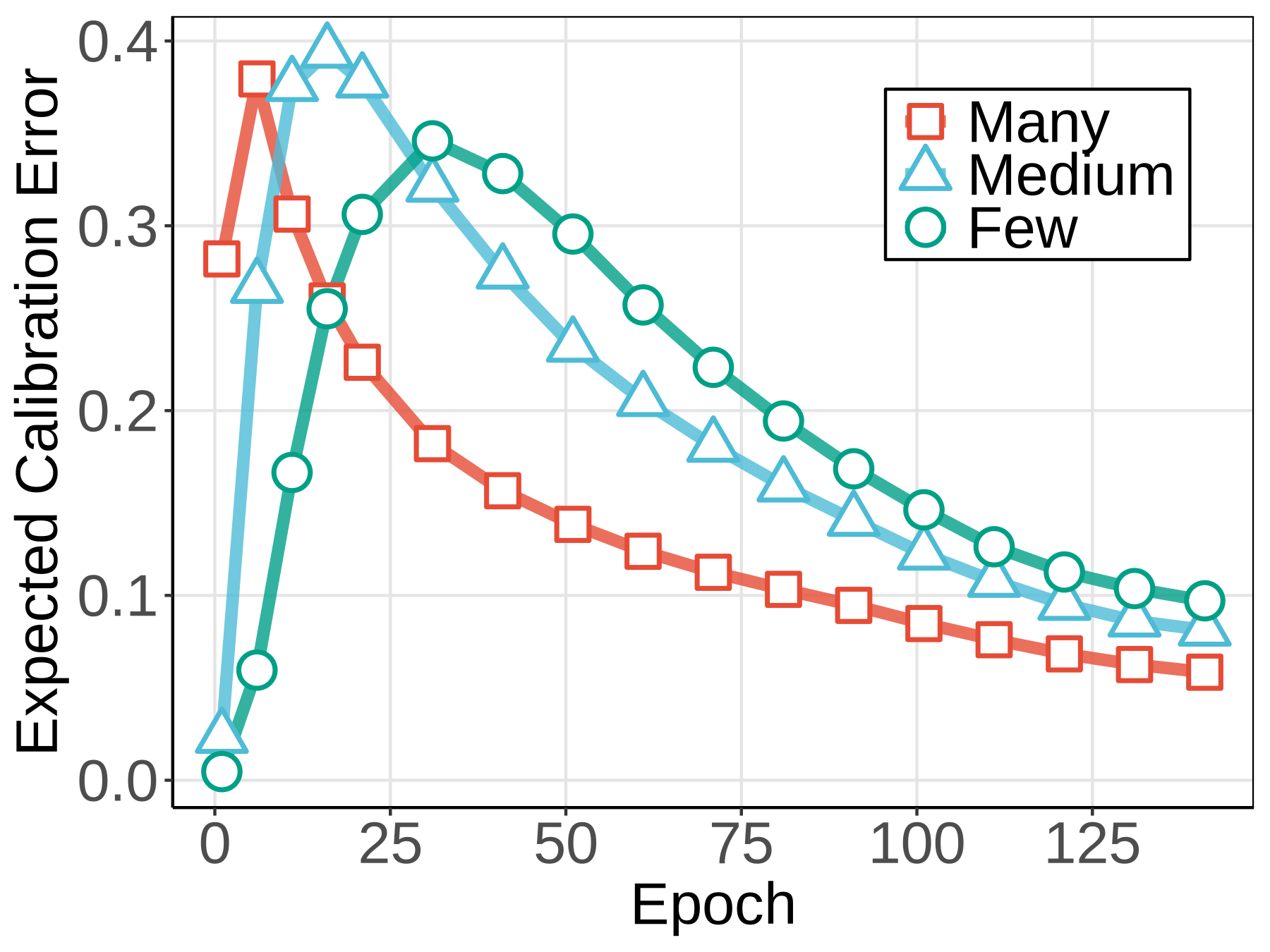} 
    \label{fig:r2q1_vit_fft_inat}
    }
    \subfigure[PEFT ViT-B/16 on iNaturalist.]{
    \includegraphics[width=0.31\textwidth]{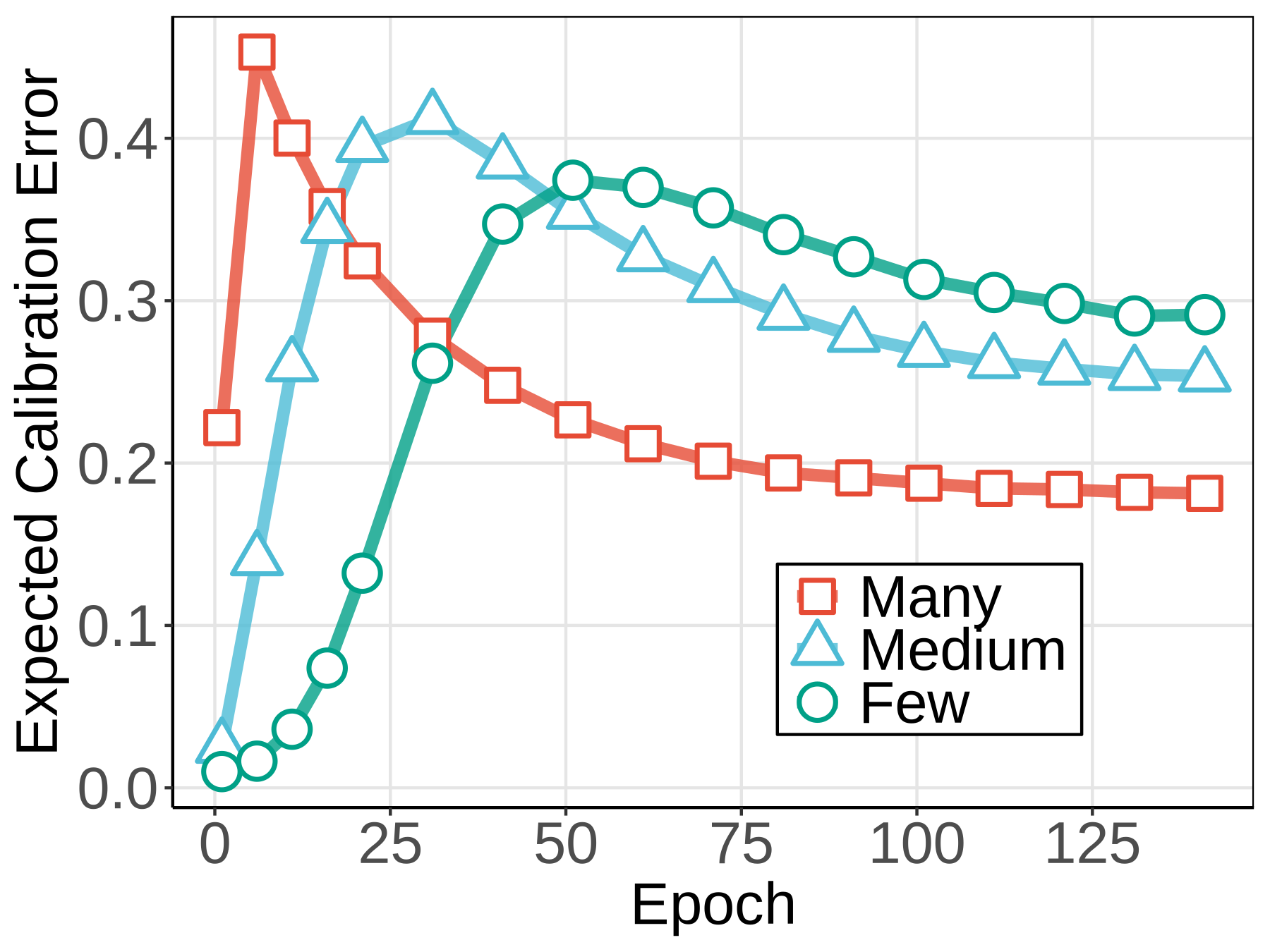} 
    \label{fig:r2q1_vit_peft_inat}
    }
    \caption{The empirical validation of Assumptions 2 and 3 under different settings. The x-axis denotes the training epochs, and the y-axis denotes the Expected Calibration Error (ECE).}
    \label{fig:r2q1_all}
\end{figure*}

At the fisrt galance, the search space of the CVS loss is $\mathcal{O}(N_\gamma \cdot N_\tau \cdot N_\nu)$ and induce additional tuning burden, where $N_\gamma, N_\tau, N_\nu$ are the number of candidate values for the $\nu, \gamma, \tau$, respectively. However, benefiting from the theoretical validation presented in Sec.\ref{subsec:theory_validation}, the tuning burden is actually much smaller than it seems.   

On one hand, according to \textbf{(In4-a)}, the multiplicative logit-adjustment could be incompatible with the DRW scheme. Hence, a small $\gamma$ is enough for most cases, according to our empirical results in Fig.\ref{fig:sensitivity_new} and Fig.\ref{fig:sensitivity_new_b}.
On the other hand, we can adopt a two-stage strategy to reduce the tuning burden. Specifically, in the first stage, we only tune $\gamma$. In the second stage, we can set $\gamma$ as the optimal value found in the first stage and tune $\nu, \tau$ with only the last few epochs.   
In this way, the tuning burden is reduced to $\mathcal{O}(n_\gamma + N_\nu N_\tau / r_\text{defer})$, where $n_\gamma$ denotes the cardinality of a small candidate set for $\gamma$, and $r_\text{defer} := T - T_0 / T$ denotes the ratio of the last few epochs to the total training epochs.

Notably, such a tuning burden is comparable to those of prior arts. For example, the LDAM+DRW loss requires tuning the hyperparameters of margin and re-weighting. If we also adopt the same two-stage strategy, its tuning burden becomes $\mathcal{O}(N_\text{margin} + N_\text{re-weight} / r_\text{defer})$, which is similar to ours. The VS loss requires tuning the hyperparameters of additive logit-adjustment and multiplicative logit-adjustment simultaneously. Hence, its tuning burden is $\mathcal{O}(N_\gamma \cdot N_\tau)$, which could be more burdensome than our method. In a nutshell, our tuning burden can be significantly reduced according to the theoretical analysis and the two-stage strategy, making it acceptable in practice. 

\section{Experiments}
\label{sec:experiments}
In this section, we conduct extensive experiments to validate the effectiveness of the CVS loss on benchmark datasets. The implementation ranges from training ResNets from scratch to fine-tuning pre-trained ViTs. All these results speak to our theoretical insights and the effectiveness of the CVS loss.

\subsection{Experiment Protocols}

\textbf{Datasets.} We conduct experiments on four popular benchmark datasets for imbalanced learning:
\begin{itemize}
    \item \textbf{CIFAR-10 LT and CIFAR-100 LT}: The original version of CIFAR-10\footnotemark[1] and CIFAR-100\footnotemark[1] \cite{krizhevsky2009learning} consists of 50,000 training images and 10,000 validation images, uniformly sampled from 10 and 100 classes, respectively. Following the protocol in \cite{DBLP:journals/nn/BudaMM18,DBLP:conf/cvpr/CuiJLSB19,DBLP:conf/nips/CaoWGAM19}, we consider two types of imbalance: long-tailed imbalance (\textbf{LT}), where the number of training samples for each class decreases exponentially, and step imbalance (\textbf{STEP}), which reduces the sample size of half of the classes to a fixed ratio. For both imbalance types, we report the balanced accuracy averaged over 5 random seeds with the imbalance ratio $\rho \in \{10, 100\}$. Besides, for CIFAR-100 LT ($\rho = 100$), the classes are divided into three subcategories: Many ($>$ 100 images), Medium (20$\sim$100 images), and Few ($<$ 20 images), following the protocol in \cite{DBLP:conf/cvpr/0002MZWGY19,DBLP:conf/icml/Shi00SH024}.
    \item \textbf{ImageNet-LT and iNaturalist}: We use the long-tailed version of the ImageNet dataset\footnotemark[2] \cite{DBLP:journals/ijcv/RussakovskyDSKS15} proposed by \cite{DBLP:conf/cvpr/0002MZWGY19}, which contains 115.8K images from 1K classes with $N_1 = 1280$ and $N_C = 5$. iNaturalist\footnotemark[3] \cite{DBLP:conf/cvpr/HornASCSSAPB18} is a real-world long-tailed dataset with 437.5K images from 8,142 classes. Following the protocol in \cite{DBLP:conf/cvpr/ZhongC0J21}, the classes are split into three subsets under long-tailed imbalance: Many, Medium, and Few. ImageNet-LT has 390/445/165 many/medium/few classes, respectively, while iNaturalist has 842/3701/3599, respectively. We report the balanced accuracy on all the classes and each subset.
\end{itemize}
\footnotetext[1]{\url{https://www.cs.toronto.edu/~kriz/cifar.html}. Licensed MIT.}
\footnotetext[2]{\url{https://image-net.org/index.php}. Licensed MIT.}
\footnotetext[3]{\url{https://github.com/visipedia/inat_comp}. Licensed MIT.}

    \textbf{Competitors.} On the CIFAR datasets, we select several popular loss-oriented methods as the baselines, including the CE loss (CE) \cite{10.5555/2371238}, the LDAM loss (LDAM) \cite{DBLP:conf/nips/CaoWGAM19}, LDAM+DRW \cite{DBLP:conf/nips/CaoWGAM19}, and the VS loss (VS) \cite{DBLP:conf/nips/KiniPOT21} that generalizes the LA loss \cite{DBLP:conf/iclr/MenonJRJVK21} and the CDT loss \cite{DBLP:journals/corr/abs-2001-01385}. We tune all the hyperparameters according to the suggestions in the original papers. Besides, we also compare our conference version \cite{DBLP:conf/nips/WangX00CH23}, denoted by CE+ADRW and VS+TLA+ADRW. On ImageNet-LT and iNaturalist, we select superior methods as competitors, especially the loss-oriented ones, which are listed in Tab.\ref{tab:ima_ina}. 

\textbf{Training protocols.} When \textit{training ResNets}:  For the CIFAR datasets, we train the models under three different protocols. \textbf{(a)} The standard protocol \cite{DBLP:conf/nips/CaoWGAM19} trains a ResNet-32 model \cite{DBLP:conf/cvpr/HeZRS16} for 200 epochs by SGD with a momentum of 0.9, a weight decay of 2e-4, and a bath size of 128 \cite{DBLP:conf/icml/SutskeverMDH13}. A multistep learning rate schedule is used with an initial learning rate of 0.1, divided by 10 at the 160th and 180th epoch by default. \textbf{(b)} The second protocol is the same as the first one but with a tuned weight decay. As shown in \cite{DBLP:conf/cvpr/AlshammariWRK22}, tuning weight decay can encourage balanced weigthts across classes and thus drastically improve the model performance. \textbf{(c)} The last protocol aims to pursue a superior performance. Inspired by \cite{DBLP:conf/iccv/CuiZ00J21}, the learning rate decays by a cosine schedule \cite{DBLP:conf/iclr/LoshchilovH17} from 0.1 to 1e-4 in 400 epochs, and RandAugment \cite{DBLP:conf/nips/CubukZS020} also benefits the training process. Besides, we incorporate the Sharpness-Aware Minimization (SAM) technique \cite{DBLP:conf/iclr/ForetKMN21} to facilitate the optimization of the minority classes, allowing them to escape saddle points and converge to flat minima \cite{rangwani2022escaping}. 
For ImageNet-LT and iNaturalist, we follow the implementation in \cite{DBLP:conf/cvpr/ZhongC0J21,rangwani2022escaping}, where the ResNet-50 model \cite{DBLP:conf/cvpr/HeZRS16} is trained for 90 epochs by SGD with a momentum of 0.9 and a batch size of 256. The learning rate is decided by a cosine schedule, with an initial value of 0.2 and 0.1, respectively. The SAM technique is also used, with a hyperparameter as suggested in \cite{rangwani2022escaping}.
Note that these techniques do not introduce any additional knowledge or model parameters. 
When \textit{fine-tuning ViTs}: Following the protocols in \cite{DBLP:conf/icml/Shi00SH024}, we fine-tune the image encoder of CLIP \cite{DBLP:conf/icml/RadfordKHRGASAM21} with a ViT-B/16 backbone \cite{DBLP:conf/iclr/DosovitskiyB0WZ21}. We use the same hyperparameters as those in \cite{DBLP:conf/icml/Shi00SH024}, except that the learning rate and training epochs are set to $0.001$ and $20$, respectively. 
Besides, the test-time ensembling (TTE) technique \cite{DBLP:conf/icml/Shi00SH024} is also used to further improve the model performance on the ImageNet-LT and iNaturalist datasets. 

\begin{figure*}[t]
    \centering
    \subfigure[CIFAR-10 LT]{
      \includegraphics[width=0.23\linewidth]{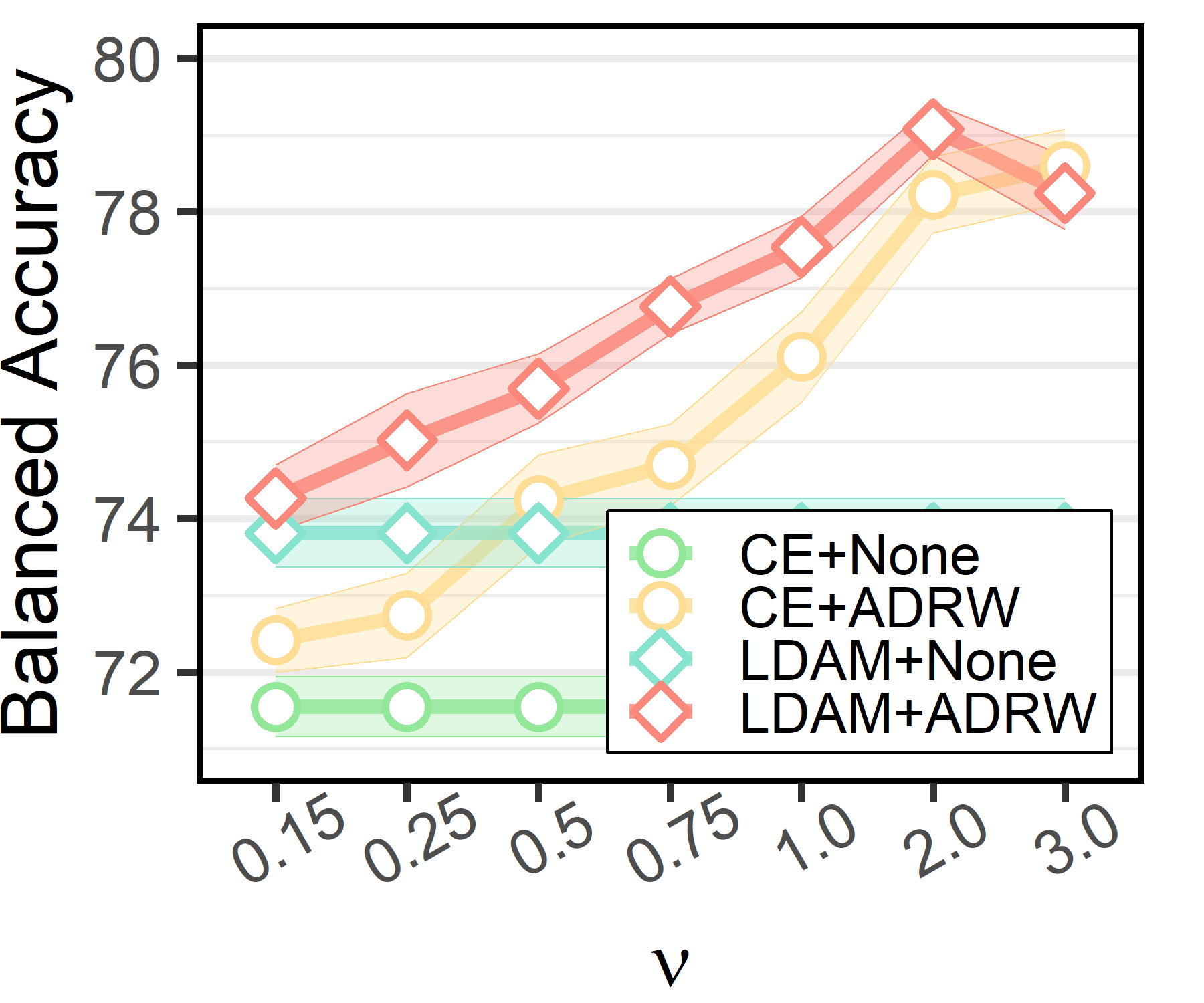}
     }
    \subfigure[CIFAR-10 Step]{
      \includegraphics[width=0.23\linewidth]{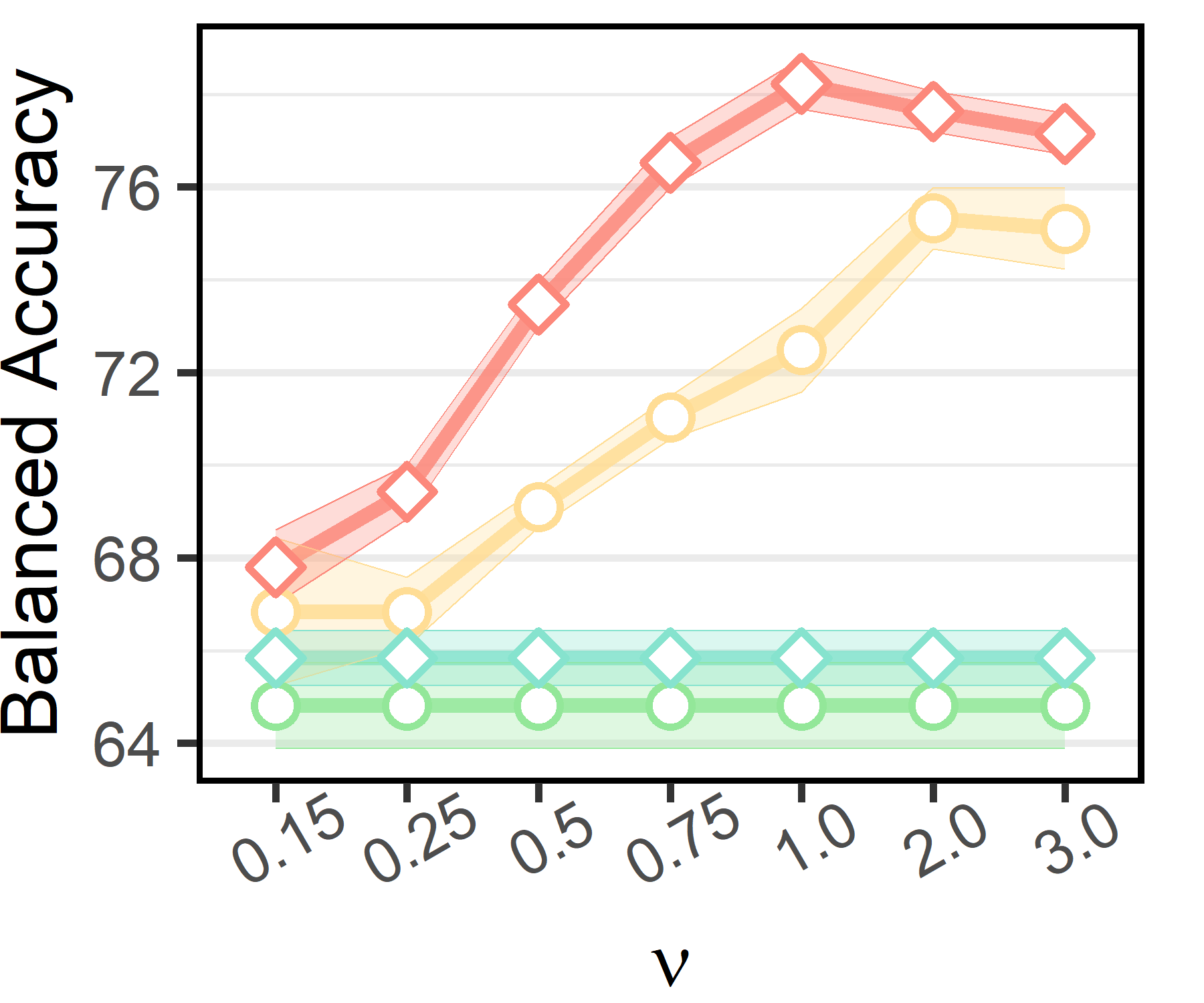}
     }
    \subfigure[CIFAR-100 LT]{
      \includegraphics[width=0.23\linewidth]{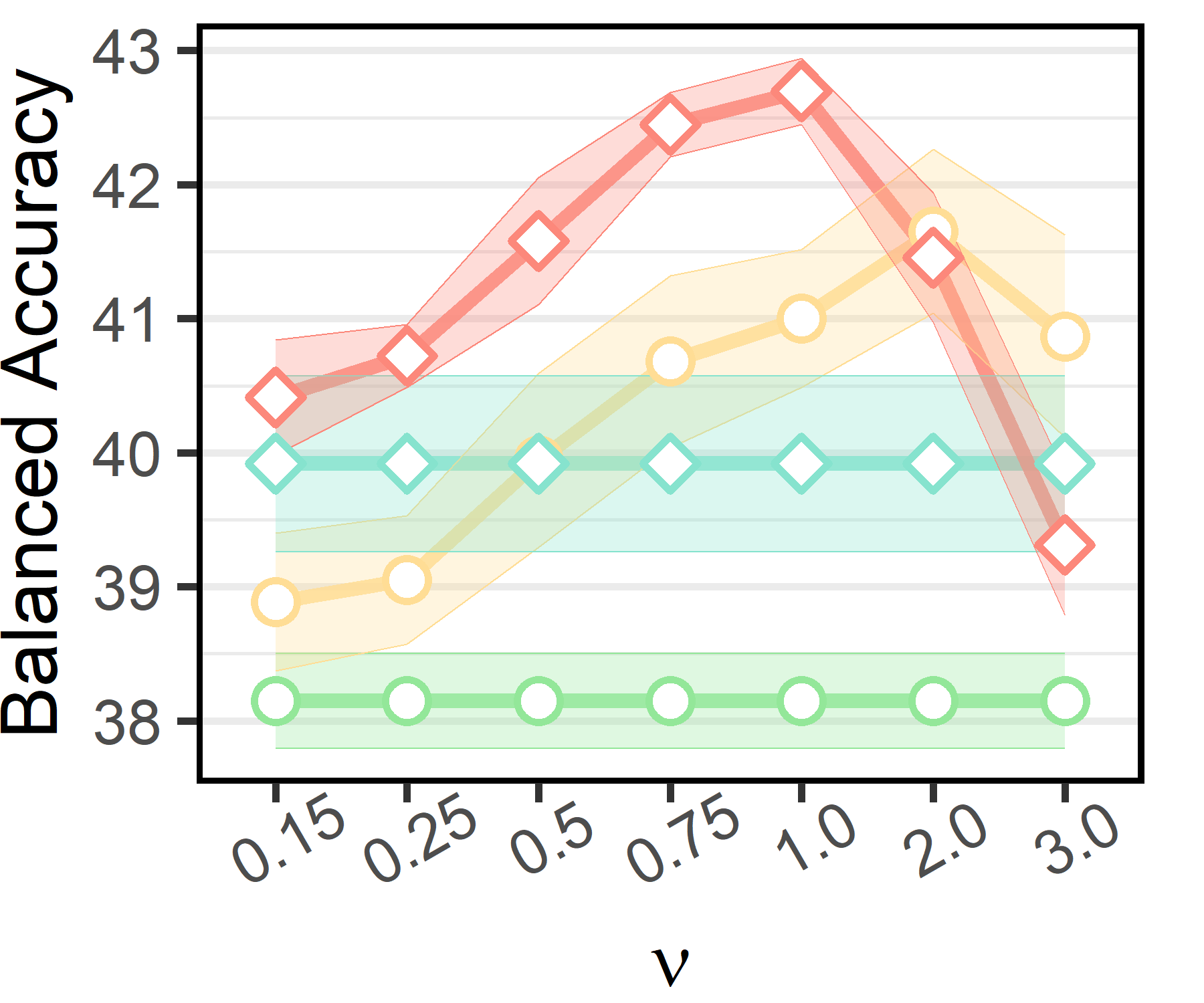}
     }
    \subfigure[CIFAR-100 Step]{
    \includegraphics[width=0.23\linewidth]{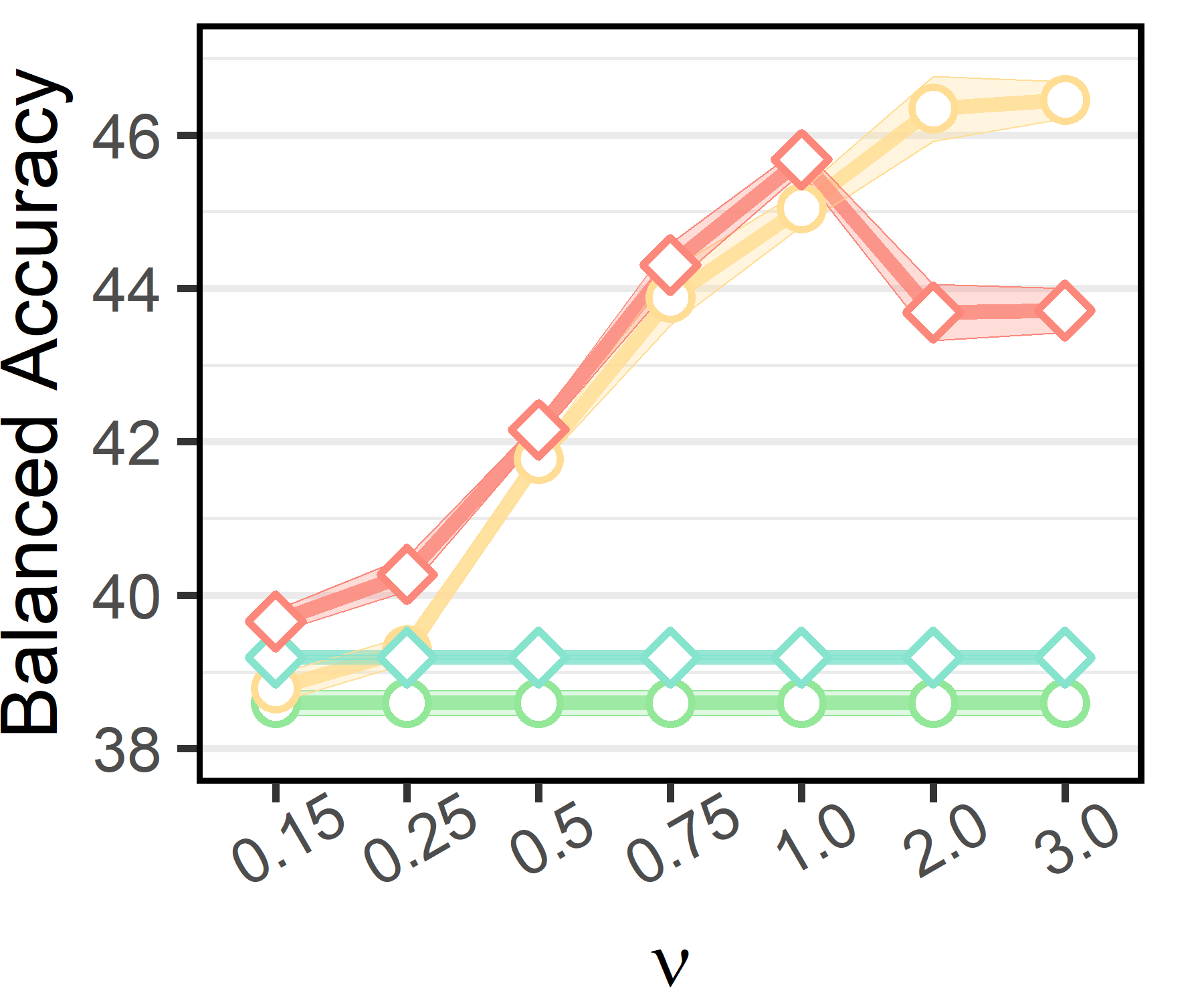}
    }
    \caption{The balanced accuracy of the CE loss and the LDAM loss \textit{w.r.t.} $\alpha_y \propto \pi_y^{-\nu}$ on the CIFAR datasets, where the imbalance ratio $\rho = 100$. Both re-weighting and logit-adjustment boost the model performance, which is consistent with the theoretical insight \textbf{(In1)} and \textbf{(In4-b)}.}
    \label{fig:alpha}
\end{figure*}

\begin{figure*}[!t]
    \centering
    \subfigure[CIFAR-10 LT]{
      \includegraphics[width=0.45\linewidth]{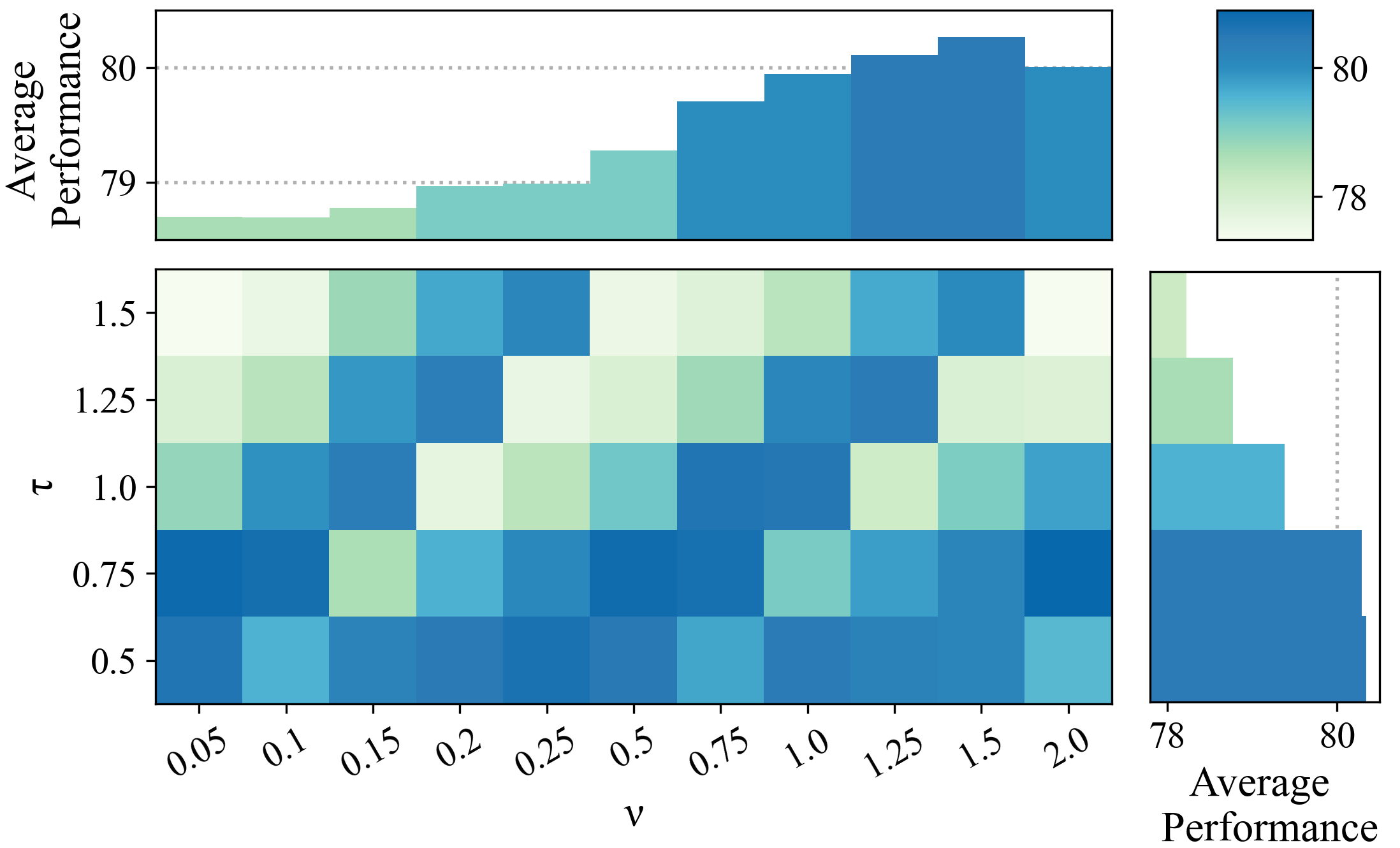}
     }
    \subfigure[CIFAR-10 Step]{
      \includegraphics[width=0.45\linewidth]{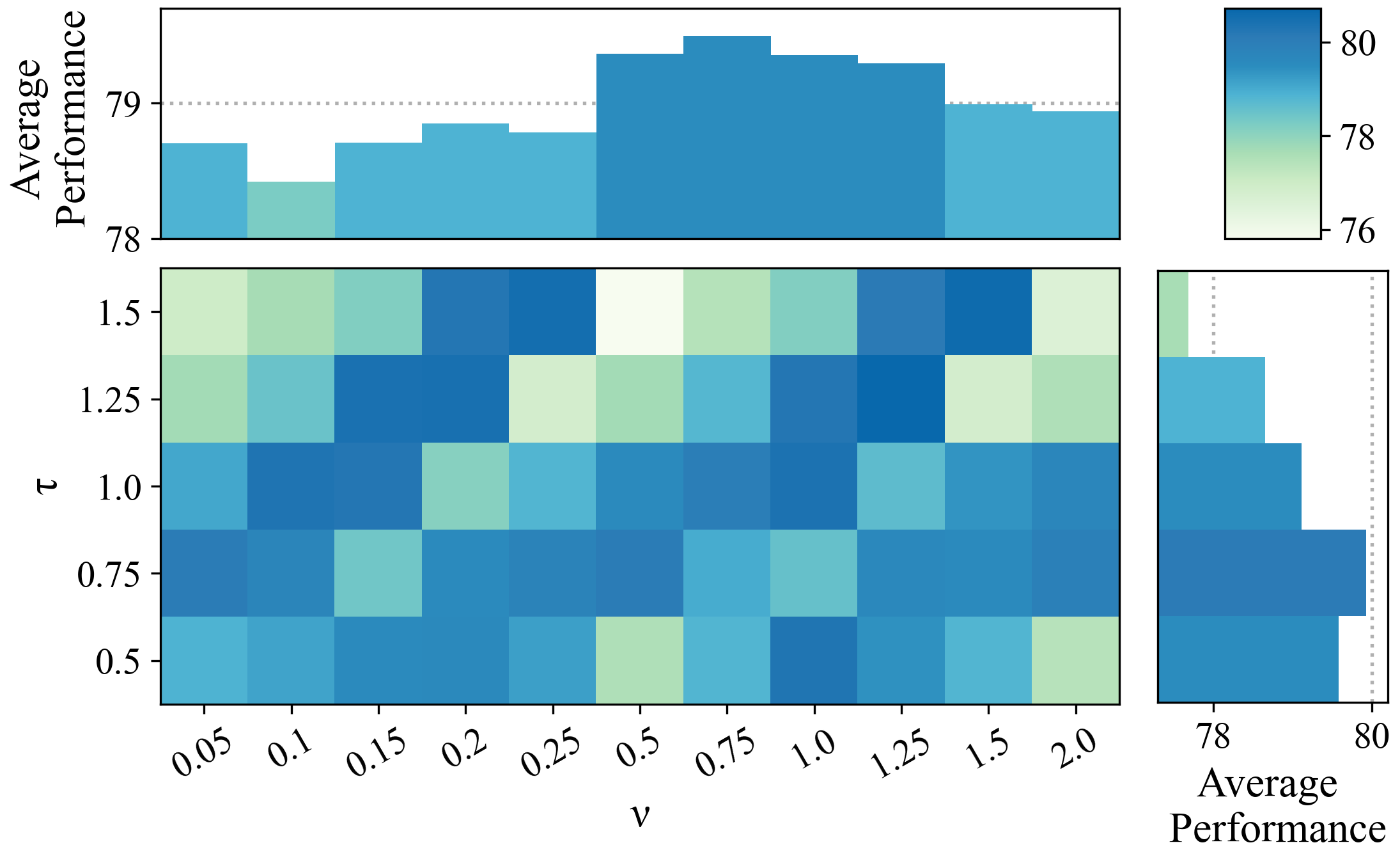}
     }
    \caption{Sensitivity analysis of VS+ADRW \textit{w.r.t.} $\alpha_y \propto \pi_y^{-\nu}$ and $\Delta_y = \tau \log \pi_y$ on the CIFAR-10 dataset, where the imbalance ratio $\rho = 100$. Both re-weighting and logit-adjustment boost the model performance, which is consistent with the theoretical insights \textbf{(In1)} and \textbf{(In4-b)}.}
    \label{fig:sensitivity_cifar10}
\end{figure*}

\begin{figure*}[t]
    \centering
    \subfigure[CIFAR-10 LT]{
      \includegraphics[width=0.23\linewidth]{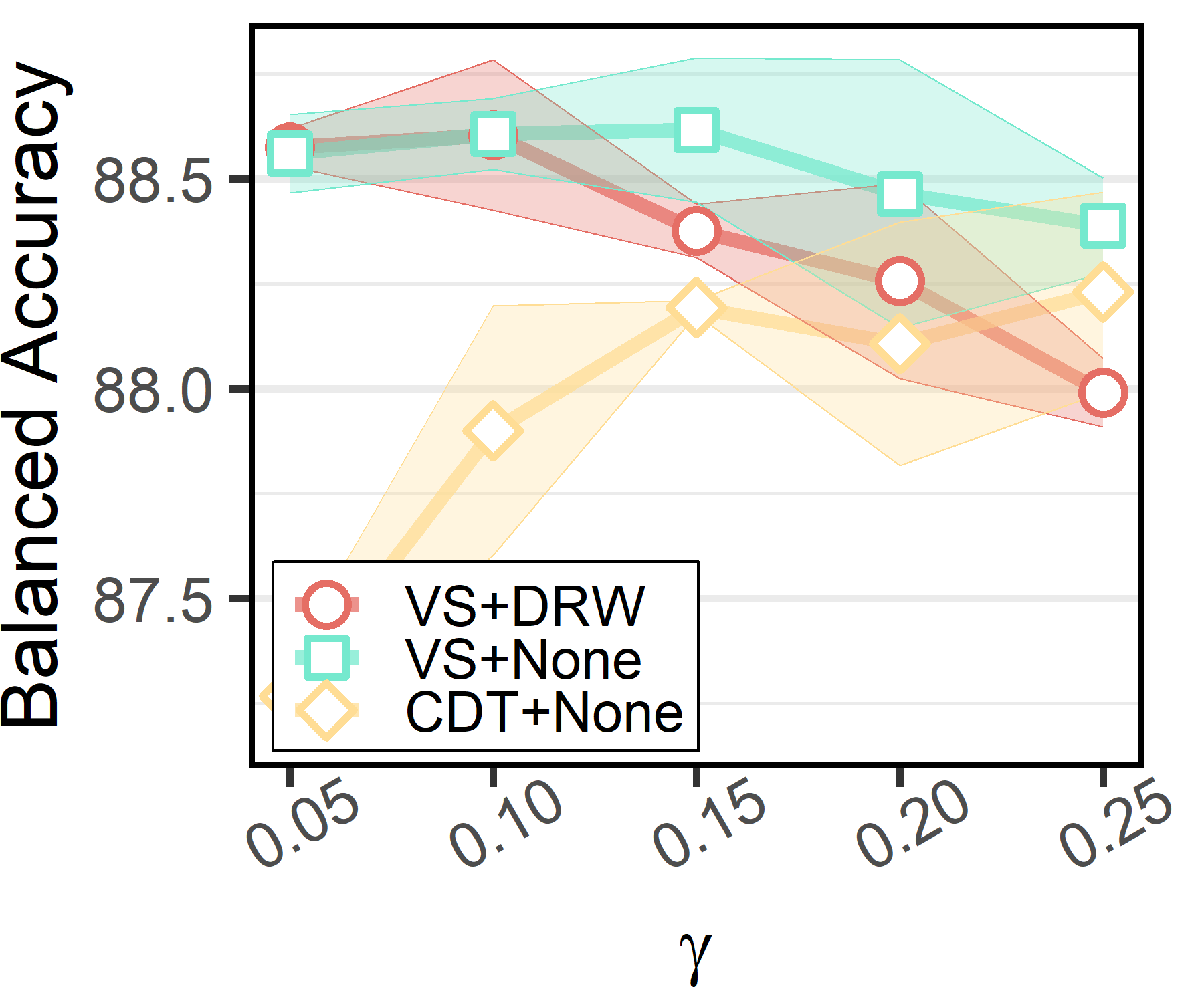}
     }
    \subfigure[CIFAR-10 Step]{
      \includegraphics[width=0.23\linewidth]{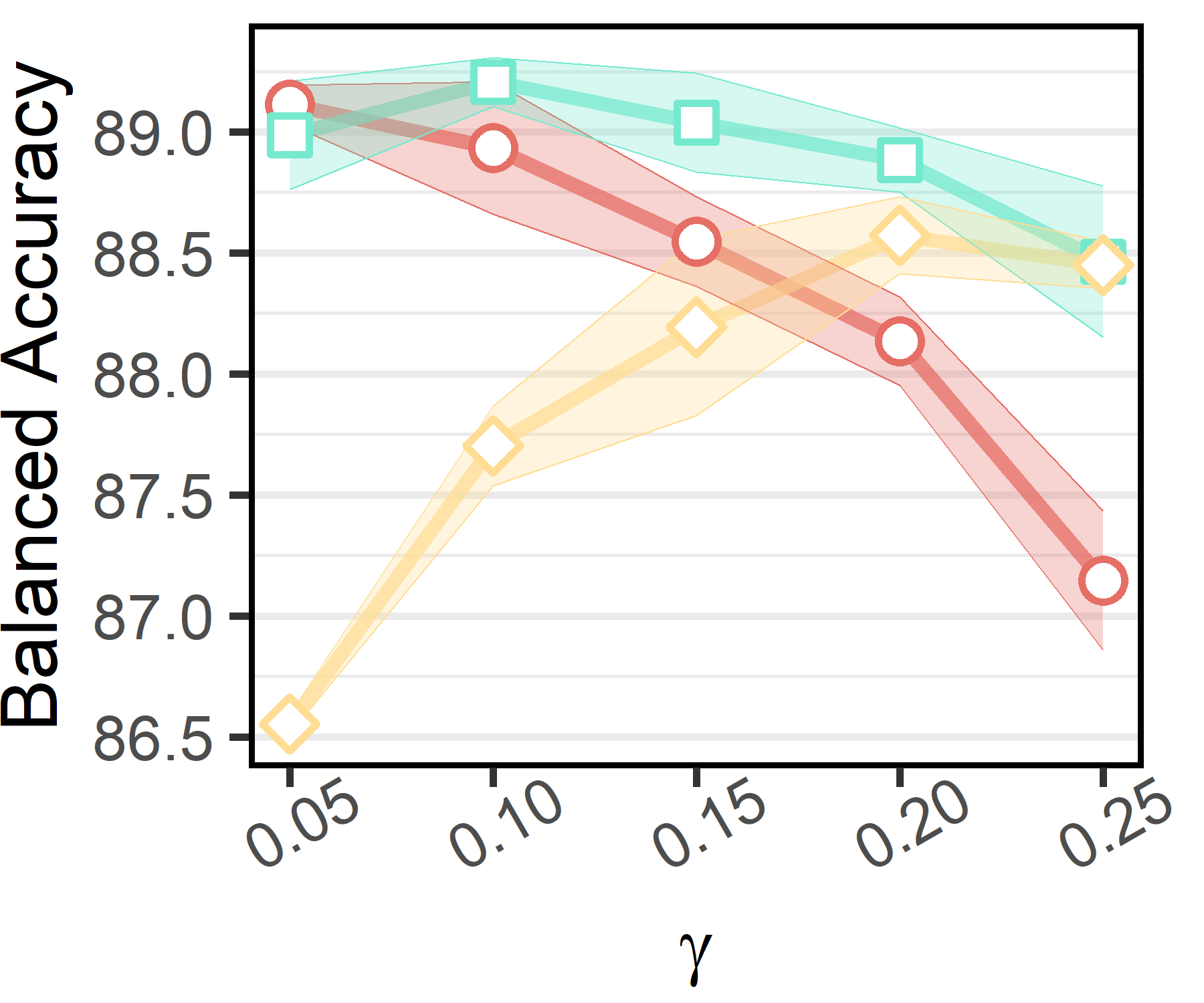}
     }
    \subfigure[CIFAR-100 LT]{
      \includegraphics[width=0.23\linewidth]{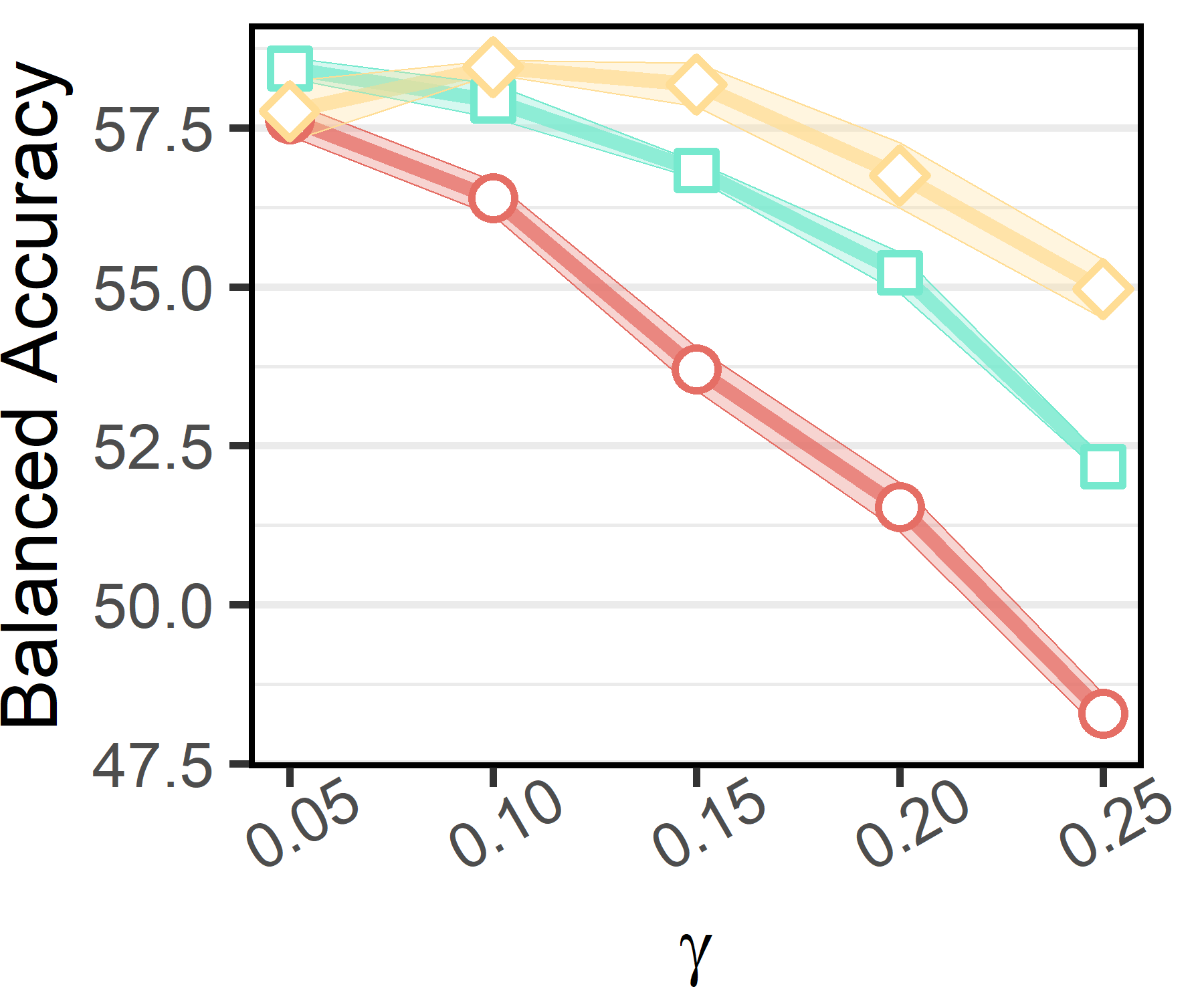}
     }
    \subfigure[CIFAR-100 Step]{
    \includegraphics[width=0.23\linewidth]{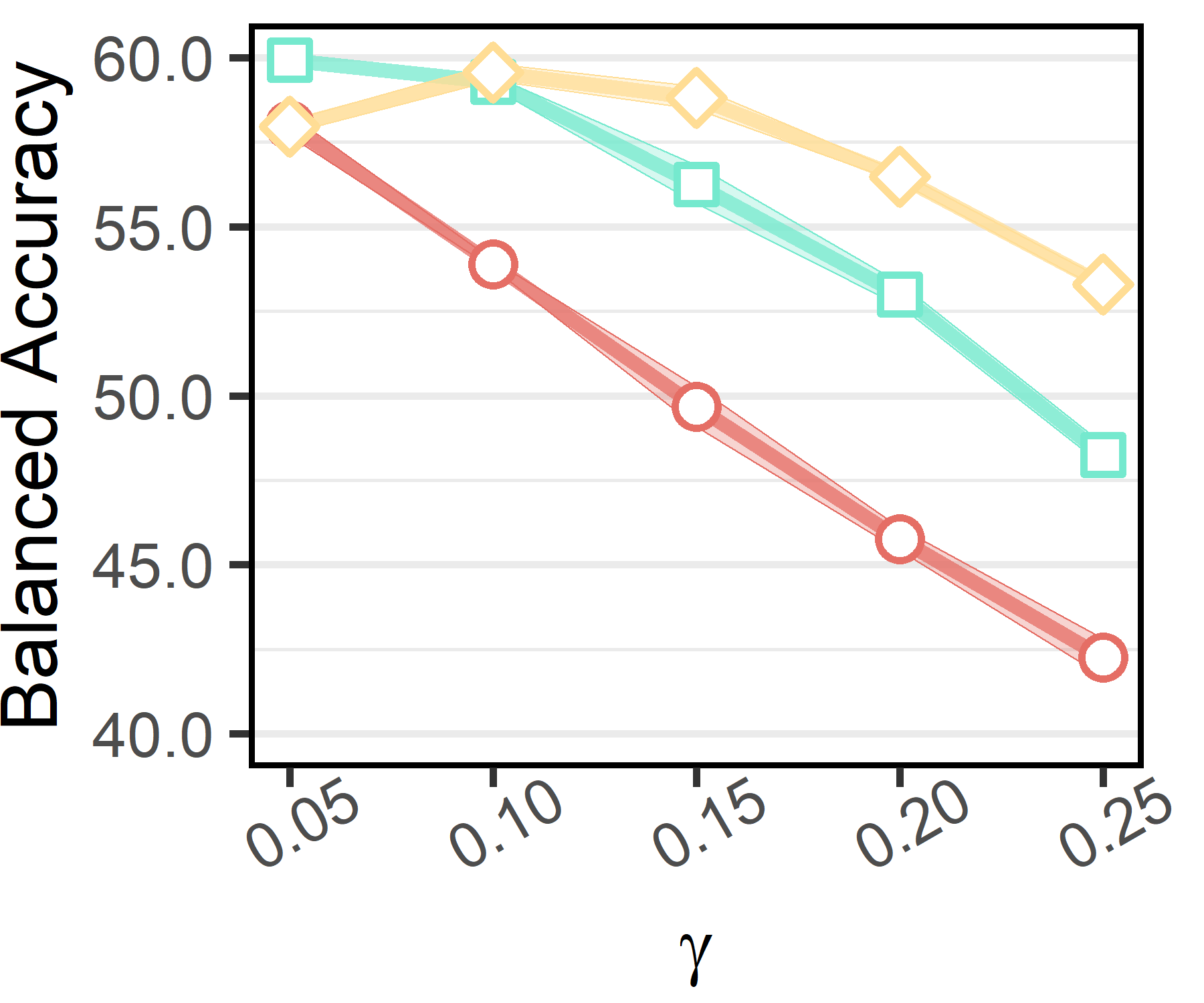}
    }
    \caption{The balanced accuracy of the VS loss \textit{w.r.t.} $\beta_y = (N_y / N_1)^\gamma$ on the CIFAR datasets, where the imbalance ratio $\rho = 10$. We can find that VS+DRW performs inferior to VS+None, especially when $\gamma$ is large, which is consistent with the theoretical insight \textbf{(In4-a).}}
    \label{fig:gamma}
\end{figure*}

\textbf{Infrastructure.} For the CIFAR datasets, we perform the experiments on an Ubuntu server equipped with Nvidia(R) RTX 4090 GPUs, whereas the experiments on ImageNet-LT and iNaturalist are conducted on NVIDIA(R) A100 GPUs. We implement the codes via \texttt{python} (v-3.8.10), and the main third-party packages include \texttt{pytorch} (v-2.0.0) \cite{DBLP:conf/nips/PaszkeGMLBCKLGA19}, \texttt{numpy} (v-1.24.2) \cite{DBLP:journals/nature/HarrisMWGVCWTBS20}, \texttt{scikit-learn} (v-1.2.2) \cite{DBLP:journals/jmlr/PedregosaVGMTGBPWDVPCBPD11} and \texttt{torchvision} (v-0.15.1) \cite{DBLP:conf/mm/MarcelR10}.

    \textbf{Parameter search}. In our conference version \cite{DBLP:conf/nips/WangX00CH23}, the parameters of VS+TLA+ADRW are searched according to the prior arts \cite{DBLP:conf/iclr/MenonJRJVK21,DBLP:conf/nips/KiniPOT21}. For sensitivity analysis, the search space of the CVS loss is listed as follows: $\nu$ is searched in $\{0.05, 0.1, 0.15, 0.2, 0.25\}$; $\tau$ is searched in $\{0.65, 0.75, 0.8, 0.85, 0.9, 0.95, 1.0\}$; $\gamma$ is searched in $\{0.005, 0.01, 0.02, 0.03, 0.04\}$. If  allowed, the weight decay is additionally searched in $[2e-4, 1e-3]$. 

\begin{table*}[!t]
    \centering
    \caption{The balanced accuracy averaged over 5 random seeds on the CIFAR datasets. The best and the runner-up methods under each protocol are marked with \Topone{red} and \Toptwo{blue}, respectively. The best baseline model is marked with \underline{underline}.}
        \renewcommand{\arraystretch}{1.3}
        \begin{tabular}{l|cccc|cccc}
            \toprule
            Dataset & \multicolumn{4}{c|}{CIFAR-10} & \multicolumn{4}{c}{CIFAR-100} \\
            \midrule
            Imbalance Type & \multicolumn{2}{c}{LT} & \multicolumn{2}{c|}{Step} & \multicolumn{2}{c}{LT} & \multicolumn{2}{c}{Step} \\
            \midrule
            Imbalance Ratio & 100   & 10    & 100   & 10    & 100   & 10    & 100   & 10 \\
            \midrule
            \multicolumn{9}{c}{Protocol (a): the standard one \cite{DBLP:conf/nips/CaoWGAM19} } \\
            \midrule
            CE    & 71.5$_{\pm 0.4}$ & 87.0$_{\pm 0.2}$ & 64.8$_{\pm 0.9}$ & 85.1$_{\pm 0.3}$ & 38.3$_{\pm 0.4}$ & 56.7$_{\pm 0.4}$ & 38.6$_{\pm 0.2}$ & 54.4$_{\pm 0.3}$ \\
            CE+DRW & 75.8$_{\pm 0.3}$ & 87.9$_{\pm 0.3}$ & 72.2$_{\pm 0.8}$ & 88.0$_{\pm 0.3}$ & 40.8$_{\pm 0.6}$ & 58.1$_{\pm 0.3}$ & 45.4$_{\pm 0.4}$ & 59.1$_{\pm 0.3}$ \\
            LDAM  & 73.8$_{\pm 0.4}$ & 86.4$_{\pm 0.4}$ & 65.8$_{\pm 0.6}$ & 85.0$_{\pm 0.3}$ & 39.9$_{\pm 0.7}$ & 55.7$_{\pm 0.5}$ & 39.2$_{\pm 0.0}$ & 50.5$_{\pm 0.2}$ \\
            LDAM+DRW & 77.7$_{\pm 0.4}$ & 87.5$_{\pm 0.2}$ & 77.8$_{\pm 0.5}$ & 87.8$_{\pm 0.3}$ & \underline{42.7$_{\pm 0.5}$} & 57.5$_{\pm 0.3}$ & 45.3$_{\pm 0.6}$ & 56.9$_{\pm 0.2}$ \\
            VS    & 78.8$_{\pm 0.2}$ & \underline{88.7$_{\pm 0.1}$} & 76.1$_{\pm 0.7}$ & \underline{88.3$_{\pm 0.1}$} & 41.8$_{\pm 0.7}$ & \underline{58.4$_{\pm 0.2}$} & \underline{46.2$_{\pm 0.3}$} & \underline{59.9$_{\pm 0.2}$} \\
            VS+DRW & \underline{80.1$_{\pm 0.1}$} & 88.6$_{\pm 0.1}$ & \underline{78.2$_{\pm 0.2}$} & 88.1$_{\pm 0.1}$ & 41.3$_{\pm 0.4}$ & 57.6$_{\pm 0.3}$ & 44.0$_{\pm 0.3}$ & 58.0$_{\pm 0.3}$ \\
            \midrule
            CE+ADRW & 78.6$_{\pm 0.5}$ & 88.2$_{\pm 0.3}$ & 75.5$_{\pm 0.6}$ & 88.5$_{\pm 0.2}$ & 41.8$_{\pm 0.6}$ & 58.3$_{\pm 0.4}$ & 46.5$_{\pm 0.3}$ & 59.2$_{\pm 0.3}$ \\
            VS+TLA+ADRW & 81.1$_{\pm 0.2}$ & \Toptwo{89.0$_{\pm 0.2}$} & 80.9$_{\pm 0.2}$ & 89.3$_{\pm 0.1}$ & 43.4$_{\pm 0.6}$ & 59.2$_{\pm 0.2}$ & 47.8$_{\pm 0.1}$ & 60.5$_{\pm 0.3}$ \\
            \midrule
            \textbf{CVS} & \Toptwo{81.3$_{\pm 0.2}$} & 88.9$_{\pm 0.1}$ & \Toptwo{81.5$_{\pm 0.2}$} & \Toptwo{89.6$_{\pm 0.1}$} & \Toptwo{43.5$_{\pm 0.2}$} & \Toptwo{59.4$_{\pm 0.2}$} & \Toptwo{48.2$_{\pm 0.1}$} & \Toptwo{61.0$_{\pm 0.2}$} \\  
            {\textbf{CVS+ADRW}} & \Topone{81.5$_{\pm 0.2}$} & \Topone{89.0$_{\pm 0.1}$} & \Topone{81.6$_{\pm 0.2}$} & \Topone{89.6$_{\pm 0.1}$} & \Topone{44.4$_{\pm 0.1}$} & \Topone{59.9$_{\pm 0.1}$} & \Topone{48.6$_{\pm 0.2}$} & \Topone{61.2$_{\pm 0.2}$} \\
            \midrule
            \multicolumn{9}{c}{Protocol (b): standard + tuned wd} \\
            \midrule
            CE & 73.4$_{\pm 0.3}$ & 87.4$_{\pm 0.2}$ & 66.2$_{\pm 0.4}$ & 85.6$_{\pm 0.1}$ & 40.7$_{\pm 0.3}$ & 58.8$_{\pm 0.2}$ & 38.9$_{\pm 0.1}$ & 54.6$_{\pm 0.1}$ \\
            LDAM+DRW & 79.8$_{\pm 0.4}$ & 88.0$_{\pm 0.2}$ & 78.7$_{\pm 0.1}$ & 87.7$_{\pm 0.2}$ & 45.4$_{\pm 0.1}$ & 58.3$_{\pm 0.3}$ & 46.4$_{\pm 0.3}$ & 57.1$_{\pm 0.2}$ \\
            VS & \underline{81.6$_{\pm 0.1}$} & \underline{89.1$_{\pm 0.1}$} & \underline{81.3$_{\pm 0.2}$} & \underline{89.3$_{\pm 0.1}$} & \underline{46.3$_{\pm 0.3}$} & \underline{61.1$_{\pm 0.3}$} & \underline{46.7$_{\pm 0.5}$} & \underline{61.0$_{\pm 0.3}$}  \\
            \midrule
            CE+ADRW & 80.5$_{\pm 0.1}$ & 88.7$_{\pm 0.1}$ & 78.7$_{\pm 0.2}$ & 88.9$_{\pm 0.2}$ & 46.2$_{\pm 0.4}$ & 61.1$_{\pm 0.1}$ & 46.9$_{\pm 0.1}$ & 60.1$_{\pm 0.2}$ \\
            VS+TLA+ADRW & 81.6$_{\pm 0.1}$ & 89.2$_{\pm 0.1}$ & 81.3$_{\pm 0.2}$ & 89.6$_{\pm 0.1}$ & 46.3$_{\pm 0.5}$ & 61.3$_{\pm 0.3}$ & 47.8$_{\pm 0.3}$ & 61.4$_{\pm 0.4}$ \\
            \midrule
            \textbf{CVS} & \Toptwo{82.2$_{\pm 0.1}$} & \Toptwo{89.2$_{\pm 0.1}$} & \Toptwo{81.7$_{\pm 0.2}$} & \Toptwo{89.7$_{\pm 0.0}$} & \Toptwo{48.2$_{\pm 0.1}$} & \Toptwo{61.6$_{\pm 0.1}$} & \Toptwo{48.7$_{\pm 0.2}$} & \Toptwo{61.8$_{\pm 0.1}$} \\
            \textbf{CVS+ADRW} & \Topone{82.4$_{\pm 0.2}$} & \Topone{89.3$_{\pm 0.1}$} & \Topone{81.8$_{\pm 0.4}$} & \Topone{89.8$_{\pm 0.1}$} & \Topone{48.3$_{\pm 0.1}$} & \Topone{61.8$_{\pm 0.1}$} & \Topone{48.8$_{\pm 0.1}$} & \Topone{61.8$_{\pm 0.1}$} \\
            \bottomrule
        \end{tabular}%
    \label{tab:caifar}%
\end{table*}%

\subsection{Theory Validation}
\label{subsec:theory_validation}

\subsubsection{Validation of Calibration Analysis}
To validate the local calibration assumptions proposed in Sec.\ref{subsec:calibration_issue}, we extend our observations to more diverse settings, where the LA loss is used as the objective.
\begin{itemize}
    \item Training ResNet-50 \cite{DBLP:conf/cvpr/HeZRS16} via parameter-efficient fine-tuning (PEFT) on the iNaturalist dataset, where the tunable parameters include the bias terms and a scaling-and-shifting (SSF) \cite{DBLP:conf/nips/LianZFW22} module after the backbone, as suggested in \cite{DBLP:conf/icml/Shi00SH024};
    \item FFT ViT-B/16 \cite{DBLP:conf/iclr/DosovitskiyB0WZ21} on the iNaturalist dataset; 
    \item PEFT ViT-B/16 on the iNaturalist dataset, following the practice in \cite{DBLP:conf/icml/Shi00SH024}.
\end{itemize}
Besides, we observe the whole training process, no longer limited to the last epoch. For comprehensive analysis, the total training epochs are set to 150. From the results in Fig.\ref{fig:r2q1_all}, we have the following observations:
\begin{itemize}
    \item In all the settings, the ECE values of different class groups increase rapidly during the early phase of training and gradully decrease during the later phase. This phenomenon indicates that the model first memorizes the training data and then gradually learns to calibrate itself.
    \item The calibration dynamics are different among different class groups. On one hand, the ECE values of many-shot classes increase more rapidly during the early phase, making them more miscalibrated than the other classes. On the other hand, although all the classes are calibrated to some extent during the later phase, the medium-shot and few-shot classes are significantly less calibrated than the many-shot classes. To summaryize, the local calibration phenomenon exists during the whole training process.
    \item ResNet-50 is more calibrated than the two ViT-B/16 models during the early training phase. However, FFT ViT-B/16 achieves the best calibration performance during the later training phase. This observation indicates that the overall calibration performance may vary across different architectures and training protocols.
\end{itemize}

\subsubsection{Validation of Generalization Analysis}
In this part, we validate the generalization analysis in Sec.\ref{subsec:application} on the CIFAR datasets under protocol (a). Notably, we mainly observe the baselines in this part and leave the sensitivity analysis of the newly proposed methods in Sec.\ref{subsec:sensitivity_analysis}, which can also validate the analysis.

\textbf{Validation of (In1) and (In4-b).} We report the model performance of the baselines \textit{w.r.t.} the hyperparameters in Fig.\ref{fig:alpha} and Fig.\ref{fig:sensitivity_cifar10}. From these results, we can find that (1) Both {\color[RGB]{230,180, 80} CE+ADRW} and {\color[RGB]{94, 227, 206} LDAM} perform better than {\color[RGB]{127, 231, 153} CE}. In other words, either re-weighting or logit-adjustment can boost the model performance. (2) {\color[RGB]{252, 136, 123} LDAM+ADRW} outperforms {\color[RGB]{230,180, 80} CE+ADRW} and {\color[RGB]{94, 227, 206} LDAM}. For VS+ADRW, increasing the hyperparameters $\nu$ and $\tau$ appropriately can also bring performance gains. All these results validate the compatibility between re-weighting and the additive logit-adjustment.

\textbf{Validation of (In2).} We present a series of results in Fig.\ref{fig:drw_cifar100}, where the training accuracy of different classes represents the corresponding $B_f(y)$. To be specific, Fig.\ref{fig:drw_cifar100}(a) demonstrates the trend of training accuracy on the majority classes and the minority classes. For the CB loss, the learning process only focuses on the minority classes and hinders the performance improvement on the majority classes ({\color[RGB]{64, 182, 251} CB Major} \textit{v.s.} {\color[RGB]{64, 215, 189} CB Minor}). Hence, an extremely imbalanced $B_y(f)$ induces a poor generalization performance. By contrast, the DRW scheme first focuses on the majority classes ({\color[RGB]{252, 136, 123} CE+DRW Major} \textit{v.s.} {\color[RGB]{230,180, 80} CE+DRW Minor} with the training epoch $t \le T_0 = 160$) and then pays more attention to the minority classes during the terminal phase of training ({\color[RGB]{252, 136, 123} CE+DRW Major} \textit{v.s.} {\color[RGB]{230,180, 80} CE+DRW Minor} with the training epoch $t > T_0 = 160$). Benefiting from the DRW scheme, both the majority classes and the minority classes are well-trained and thus have a balanced term $B_y(f)$ (Fig.\ref{fig:drw_cifar100}(b)), leading to a corresponding improvement on the test accuracy (Fig.\ref{fig:drw_cifar100}(c)). Even if we remove the re-weighting term ({\color[RGB]{250, 152, 146}CE+None}), the imbalance degree of $B_y(f)$ is still consistent with the test performance. Note that we also decrease the learning rate of CE+None at the corresponding epoch $T_0$, making the line in Figure 3(b) and 3(c) not constant.

\textbf{Validation of (In4-a).} In Fig.\ref{fig:gamma}, we present the sensitivity analysis of multiplicative logit-adjustment $\gamma$, under different losses. From the results, we can find 1) when the multiplicative adjustment is used alone (\textit{i.e.}, {\color[RGB]{230,180, 80} CDT+None}), a larger $\gamma$ can induce better performance, which is consistent with the prior arts \cite{DBLP:journals/corr/abs-2001-01385}. However, {\color[RGB]{230,180, 80} CDT+None} performs inferior to {\color[RGB]{64, 207, 211} VS} and {\color[RGB]{252, 136, 123} VS+DRW}, which is not surprising since they use fewer adjustment terms. 2) When the multiplicative adjustment is combined with the other terms, a larger $\gamma$ does not necessarily lead to better performance. Notably, the performance is monotonically decreasing \textit{w.r.t.} $\gamma$ on the CIFAR-100 dataset; {\color[RGB]{252, 136, 123} VS+DRW} performs inferior to {\color[RGB]{64, 207, 211} VS}, especially when $\gamma$ is large. These results validate the theoretical insight (\textbf{In4-a}), \textit{i.e.}, the multiplicative logit-adjustment is incompatible with the re-weighting term.

\begin{table}[t]
    \centering
    \caption{The balanced accuracy averaged on the CIFAR-100 LT datasets. The best and the runner-up methods are marked with \Topone{red} and \Toptwo{blue}, respectively. The best baseline model is marked with \underline{underline}.}
    \renewcommand{\arraystretch}{1.25}
    \begin{tabular}{lccccc}
        \toprule
        \multirow{2}{*}{Imbalance Ratio} & \multicolumn{4}{c}{100} & \multicolumn{1}{c}{10} \\
        \cmidrule(lr){2-6}    & Many  & Med. & Few & All  & All  \\
        \midrule
            \multicolumn{6}{c}{Training ResNets from scratch, under Protocol (c)} \\
        \midrule
        CE    & \textbf{76.7} & 47.0 & 11.1 & 46.7 & 64.6 \\
        LDAM+DRW & 68.4 & 55.9 & 31.0 & 52.9 & 65.0 \\
        VS    & 68.4 & \textbf{56.6} & 34.5 & \underline{54.1} & \underline{66.8} \\
        \midrule
        CE+ADRW & 68.3 & 56.0 & 31.9 & 53.1 & 66.7  \\
        VS+TLA+ADRW & 66.7 & 56.3 & 37.0 & 53.2 & 67.0 \\
        \midrule
        \textbf{CVS} & 69.2 & 56.1 & 35.4 & \Toptwo{54.5} & \Toptwo{67.3} \\
        \textbf{CVS+ADRW} & 66.9 & 56.5 & \textbf{37.8} & \Topone{54.6} & \Topone{67.5} \\
        \midrule
            \multicolumn{6}{c}{FFT pre-trained ViTs} \\
        \midrule
        CE & 90.4 & 74.1 & 45.5 & 71.2 & 83.8 \\
        LDAM+DRW & 87.4 & 75.6 & 48.9 & 71.8 & 84.2 \\
        VS & 88.4 & 79.3 & 67.9 & \underline{79.1} & \underline{85.1} \\
        \midrule
        CE+ADRW & 88.0 & 76.7 & 52.6 & 73.4 & 84.3 \\
        VS+TLA+ADRW & 87.2 & 78.7 & 70.5 & 79.2 & 85.3 \\
        \midrule
        \textbf{CVS} & 87.3 & 80.8 & 68.4 & \Toptwo{79.4} & \Toptwo{85.5} \\
        \textbf{CVS+ADRW} & 87.1 & 79.8 & 72.0 & \Topone{80.0} & \Topone{85.7} \\
        \midrule
            \multicolumn{6}{c}{PEFT pre-trained ViTs via LIFT \cite{DBLP:conf/icml/Shi00SH024}} \\
        \midrule
        CE & 87.9 & 70.7 & 36.9 & 66.6 & 80.5 \\
        LDAM+DRW & 84.8 & 78.9 & 70.9 & \underline{78.6} & \underline{83.0} \\
        VS & 82.1 & 80.4 & 69.4 & 77.7 & 82.7 \\
        \midrule
        CE+ADRW & 82.0 & 79.5 & 76.5 & 79.5 & 82.9 \\
        VS+TLA+ADRW & 81.9 & 79.8 & 77.2 & 79.8 & 83.1 \\
        \midrule
        \textbf{CVS} & 82.1 & 81.5 & 75.4 & \Toptwo{79.9} & \Toptwo{83.1} \\
        \textbf{CVS+ADRW} & 79.7 & 80.5 & 80.7 & \Topone{80.3} & \Topone{83.5} \\
        \bottomrule
    \end{tabular}%
    \label{tab:caifar-protocol-c}%
\end{table}

\begin{table*}[t]
    \centering
    \caption{The balanced accuracy on ImageNet-LT and iNaturalist. The results of the competitors are referenced from the corresponding paper or the survey \cite{DBLP:journals/corr/abs-2408-00483}. The best and the runner-up methods are marked with \textbf{bold} and \underline{underline}, respectively.}
    \renewcommand{\arraystretch}{1.15}
    \begin{tabular}{l|c|cccc|cccc}
      \toprule
      \multicolumn{1}{c}{\multirow{2}[2]{*}{Method}} & \multicolumn{1}{c}{\multirow{2}[2]{*}{Loss-ori.}} & \multicolumn{4}{c}{ImageNet-LT} & \multicolumn{4}{c}{iNaturalist} \\
  \cmidrule{3-10}    \multicolumn{1}{c}{} & \multicolumn{1}{c}{} & Many  &  Med. & Few   & \multicolumn{1}{c}{All } & Many  &  Med. & Few   & All  \\
      \midrule
        \multicolumn{9}{c}{Training ResNets} & \\
      \midrule
      OLTR \cite{DBLP:conf/cvpr/0002MZWGY19}  & \xmark & 43.2  & 35.1  & 18.5  & 35.6  & 59.0  & 64.1  & 64.9  & 63.9  \\
      LFMR \cite{DBLP:conf/eccv/XiangDH20}  & \xmark & 47.1  & 35.0  & 17.5  & 37.2  & -     & -     & -     & - \\
      BBN \cite{DBLP:conf/cvpr/ZhouCWC20}  & \xmark & -     & -     & -     & -     & 49.4  & 70.8  & 65.3  & 66.3  \\
      cRT \cite{DBLP:conf/iclr/KangXRYGFK20}  & \xmark & 61.8  & 46.2  & 27.3  & 49.6  & 69.0  & 66.0  & 63.2  & 65.2  \\
      $\tau$-norm \cite{DBLP:conf/iclr/KangXRYGFK20} & \xmark & 59.1  & 46.9  & 30.7  & 49.4  & 65.6  & 65.3  & 65.5  & 65.6  \\
      De-confound \cite{DBLP:conf/nips/TangHZ20} & \xmark & 62.7  & 48.8  & 31.6  & 51.8  & -     & -     & -     & - \\
      LADE \cite{DBLP:conf/cvpr/HongHCSKC21} & \xmark & 62.3 & 49.3 & 31.2 & 51.9 & - & - & - & 70.0  \\
      DRO-LT \cite{DBLP:conf/iccv/SamuelC21} & \xmark & \underline{64.0}  & 49.8  & 33.1  & 53.5  & -     & -     & -     & 69.7  \\
      DiVE \cite{DBLP:conf/iccv/HeWW21}  & \xmark & \textbf{64.1}  & 50.4  & 31.5  & 53.1  & 70.6  & 70.0  & 67.6  & 69.1  \\
      DisAlign \cite{DBLP:conf/cvpr/ZhangLY0S21} & \xmark & 61.3  & 52.2  & 31.4  & 52.9  & 69.0  & \underline{71.1}  & 70.2  & 70.6  \\
      WB \cite{DBLP:conf/cvpr/AlshammariWRK22}   & \xmark & 62.5  & 50.4  & \textbf{41.5}  & 53.9  & 71.2  & 70.4  & 69.7  & 70.2  \\
      SBCL \cite{DBLP:conf/iccv/HouZWZ23} & \xmark & 63.8 & 51.3 & 31.2 & 53.4 & -     & -     & -     & 70.8 \\
      \midrule
      CE \cite{rangwani2022escaping} & \cmark & 62.5 & 36.6 & 12.5 & 42.7 & \textbf{72.8} & 62.7 & 54.8 & 60.3 \\
      Focal \cite{DBLP:conf/cvpr/CuiJLSB19} & \cmark & 36.4  & 29.9  & 16.0  & 30.5  & -     & -     & -     & 61.1  \\
      CE+CB \cite{DBLP:conf/cvpr/CuiJLSB19} & \cmark & 39.6  & 32.7  & 16.8  & 33.2  & 53.4  & 54.8  & 53.2  & 54.0  \\
      IB \cite{Park_2021_ICCV} & \cmark & -     & -     & -     & -     & -  & - & - & 65.4  \\
      LDAM+DRW \cite{DBLP:conf/nips/CaoWGAM19} & \cmark & 61.8  & 47.2  & 31.4  & 50.7  & - & -  & - & 68.0  \\
      ALA \cite{DBLP:conf/aaai/ZhaoC0HZ22} & \cmark & 62.4  & 49.1  & 35.7  & 52.4  & \underline{71.3} & 70.8  & 70.4 & 70.7  \\
      DLLA \cite{DBLP:journals/tcsv/ZhangGLC24} & \cmark & -  & -  & -  & 53.7  & - & -  & - & \underline{71.2} \\
      \midrule
      Ours-Origin & \cmark & 62.9  & \underline{52.6}  & 37.1  & \underline{54.1}  & 64.7  & 70.7  & \underline{72.1}  & 70.7  \\
      Ours-CVS & \cmark & 63.0 & \textbf{53.2} & \underline{38.9} & \textbf{54.7} & 65.5 & \textbf{71.4} & \textbf{72.9} & \textbf{71.5}  \\
    \midrule
        \multicolumn{9}{c}{Fine-tuning pre-trained model} & \\
    \midrule
        Decoder \cite{DBLP:journals/ijcv/WangYWHCYXXZ24} & \xmark & -     & -     & -     & 73.2  & -     & -     & -     & 59.2 \\
        LPT \cite{DBLP:conf/iclr/DongZYZ23} & \xmark & -     & -     & -     & -     & -     & -     & 79.3  & 76.1 \\
        LIFT \cite{DBLP:conf/icml/Shi00SH024} & \xmark & 80.2 & 76.1 & 71.5 & 77.0 & 72.4 & 79.0 & 81.1 & 79.1   \\
        LIFT+TTE \cite{DBLP:conf/icml/Shi00SH024} & \xmark & \textbf{81.3} & \underline{77.4} & \underline{73.4} & \underline{78.3} & \underline{74.0} & \underline{80.3} & \underline{82.2} & \underline{80.4}   \\
    \midrule
        Ours-CVS & \cmark & 80.1 & 76.9 & 72.4 & 77.4  & 72.9 & 79.5 & 81.5 & 79.7   \\
        Ours-CVS+TTE & \cmark & \underline{81.2} & \textbf{78.1} & \textbf{74.2} & \textbf{78.7}  & \textbf{74.5} & \textbf{80.5} & \textbf{82.7} & \textbf{80.9}   \\
    \bottomrule
    \end{tabular}%
    \label{tab:ima_ina}%
\end{table*}

\subsection{Performance Comparison}
\label{subsec:performance_comparison}
Tab.\ref{tab:caifar} and Tab.\ref{tab:caifar-protocol-c} present the empirical results on the CIFAR datasets. When training ResNets, we have the following observations: 
\begin{itemize}
    \item Compared with CE, re-weighting and logit-adjustment both help improve the model performance. Notably, more terms tend to lead to better performance (\textit{e.g.}, VS+TLA+ADRW $\succ$ VS $\succ$ CE+ADRW $\succ$ CE; CVS+ ADRW $\succ$ CVS $\succ$ CE). These results validate the insights \textbf{(In1)} and \textbf{(In4-b)}.
    \item When $\rho = 10$ or on the CIFAR-100 dataset, VS+DRW performs inferior to VS, which is consistent with the results in Fig.\ref{fig:gamma}. Fortunately, if we do not use re-weighting and the multiplicative adjustment simultaneously, the performance can be improved (\textit{e.g.}, VS+TLA+DRW, CVS+ADRW $\succ$ VS). These results again validate our theoretical insight \textbf{(In4-a)}.
    \item CVS and CVS+ADRW achieve significant improvements over the best method in our conference version \textit{i.e.}, VS+TLA+ADRW. For example, on CIFAR-100 LT ($\rho = 100$), the results are \underline{44.4 $v.s.$ 43.4}, \underline{48.3 $v.s.$ 46.3}, \underline{54.6 $v.s.$ 53.2} under protocol (a), (b), and (c), respectively. Since the two versions mainly differ in the design of the multiplicative adjustment, this result speaks to the necessity of our calibration analysis in Sec.\ref{subsec:calibration_issue}.
    \item CVS+ADRW only slightly outperforms CVS, especially under protocol (b) and (c). From Tab.\ref{tab:caifar-protocol-c}, we can find that although ADRW can further improve the performance on the minority classes (35.4 $\to$ 37.8), the performance on the majority classes degenerates (69.2 $\to$ 66.9). Hence, the overall performance improvement only achieves a marginal gain. 
    \item The models under protocol (c) significantly outperform the ones under protocol (a) and (b), which highlights the effectiveness of more advanced training techniques, including the cosine learning rate schedule, RandAugment, and SAM.
\end{itemize}

\noindent When fine-tuning pre-trained ViTs, our method achieves a significant improvement over the baselines, particularly on more imbalanced datasets (e.g., for parameter-efficient fine-tuning, 80.3 $v.s.$ 79.8 on CIFAR-100 LT with $\rho = 100$). This performance gain mainly comes from the improvements on minority classes (80.7 $v.s.$ 77.2). 

Tab.\ref{tab:ima_ina} presents the results on the ImageNet-LT and iNaturalist datasets, where Ours-Origin and Ours-CVS denote VS+TLA+ADRW and CVS+ADRW, respectively. From these results, we have the following observations:
    \begin{itemize}
        \item Our method outperforms the competitors, including the recent loss-oriented ones, which again confirms the effectiveness of the proposed learning algorithm. For example, when training ResNets, Ours-CVS outperforms the best competitor by 1.0\% and 0.3\% on ImageNet-LT and iNaturalist, respectively. When fine-tuning pre-trained models, the performance gains on the iNaturalist dataset are 0.6\% and 0.4\% for Ours-CVS and Ours-CVS+TTE, respectively.
        \item Ours-CVS performs superior to Ours-Origin with a margin of 0.6\% and 0.8\%, which again validates the necessity of the calibration analysis in Sec.\ref{subsec:calibration_issue} and the effectiveness of the MLA loss.
    \end{itemize}

\begin{figure*}[t]
    \centering
    \subfigure[Re-weighting]{
        \includegraphics[width=0.31\linewidth]{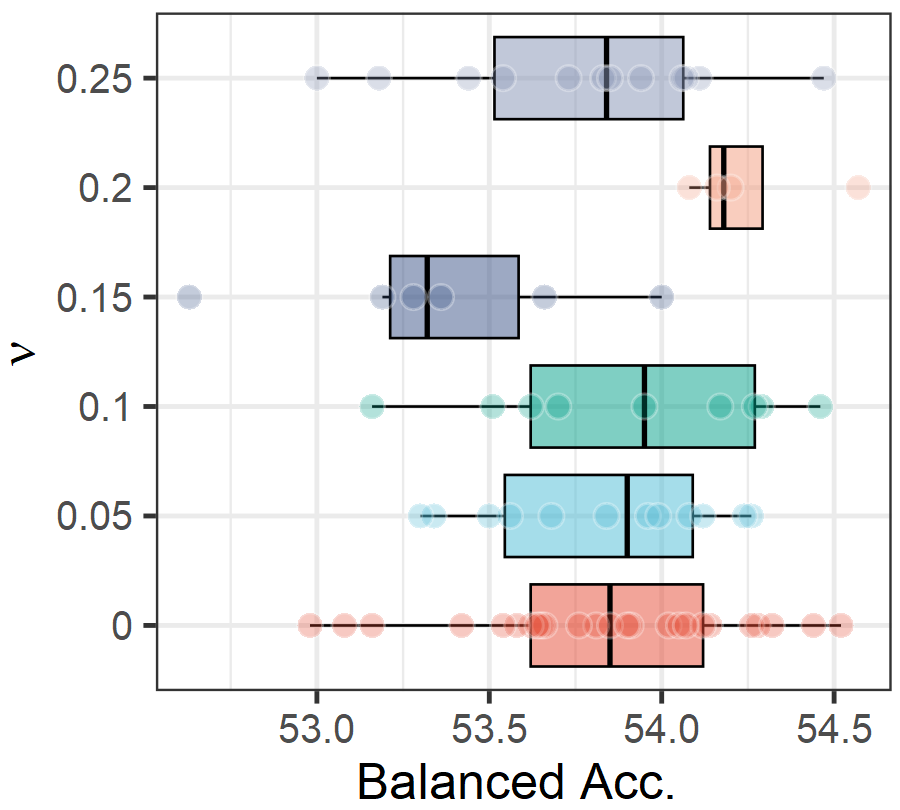}
       }
    \subfigure[Multiplicative adjustment]{
      \includegraphics[width=0.31\linewidth]{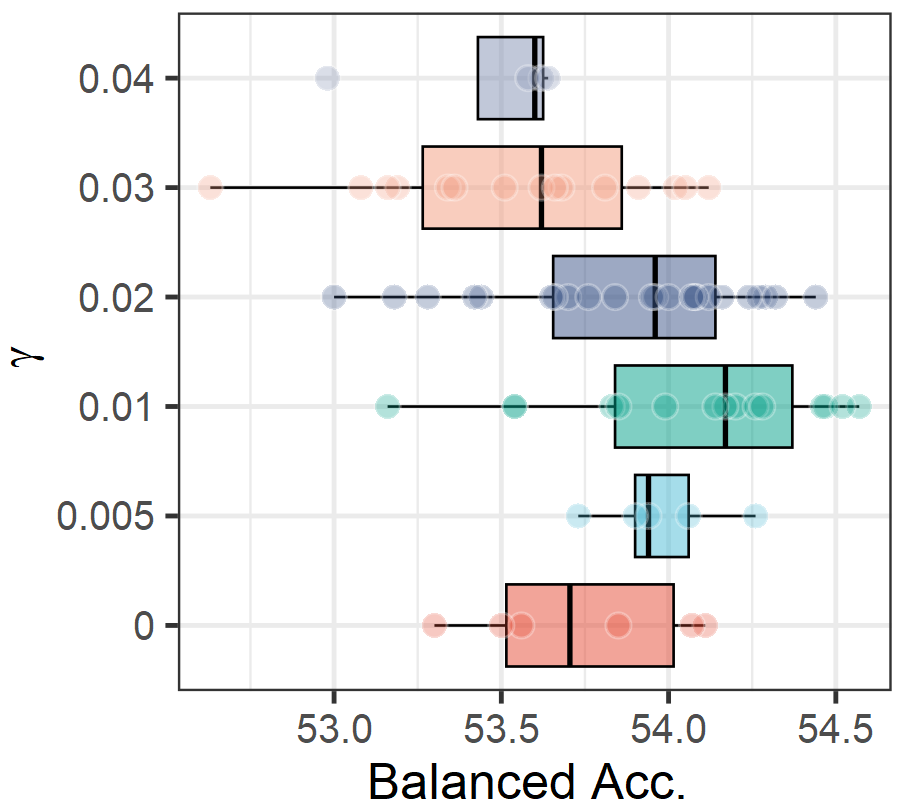}
     }
    \subfigure[Additive adjustment]{
      \includegraphics[width=0.31\linewidth]{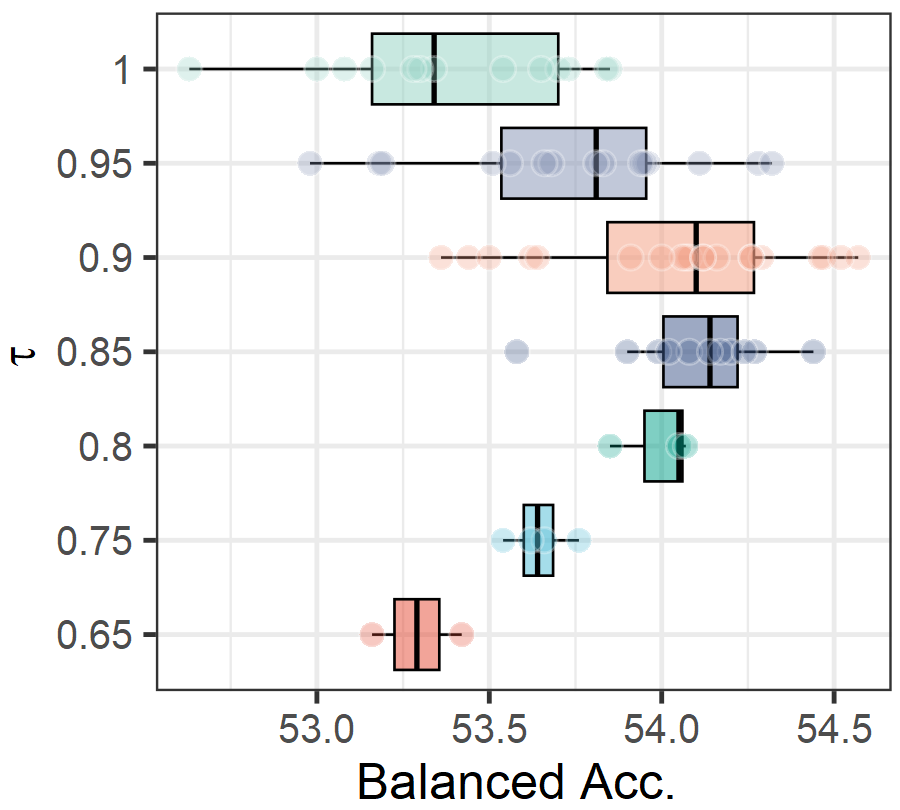}
     }
    \caption{Sensitivity analysis of the proposed method, conducted on CIFAR-100 LT ($\rho = 100$) under the protocol (c).}
    \label{fig:sensitivity_new}
\end{figure*}

\subsection{Sensitivity Analysis}
\label{subsec:sensitivity_analysis}
Next, we provide the sensitivity analysis of the proposed method. All the experiments are conducted on CIFAR-100 LT ($\rho = 100$) under the protocol (c). From the results in Fig.\ref{fig:sensitivity_new}, we have the following observations:
\begin{itemize}
    \item For $\nu$, the best overall performance is achieved when $\nu = 0.2$. However, this performance gain is marginal compared with $\nu = 0$ (54.5 $\to$ 54.6). From the results in Tab.\ref{tab:caifar-protocol-c}, we know that the performance improvement on minority classes (35.4 $\to$ 37.8) is offset by the performance degradation on majority classes (69.2 $\to$ 66.9). This marginal gain is consistent with our generalization analysis since CLA and MLA are enough to re-balance the generalization bound. Besides, the ADRW term could be more effective when the number of minority classes increases. 
    \item For $\gamma$, the best performance is achieved when $\gamma = 0.01$. When we further increase $\gamma$, the performance degrades, even worse than the case when $\gamma = 0$. On one hand, this result shows that benefiting from our learning algorithm, it is also beneficial to combine the multiplicative adjustment with the other terms, which overcomes the incompatibility issue observed in Fig.\ref{fig:gamma}. On the other hand, it is still dangerous to increase $\gamma$ too much, which again validates our generalization analysis.
    \item For $\tau$, the best performance is achieved when $\tau = 0.9$, which is consistent with the prior arts \cite{DBLP:conf/nips/KiniPOT21}. When we further increase $\tau$, the performance degrades gradually, which can also be explained by our generalization analysis.
\end{itemize}

\section{Conclusion and Future Work}
\label{sec:conclusion}
This paper provides a comprehensive analysis of consistency and generalization for loss-oriented imbalanced learning. We find that the properties used in prior arts are generally global, which should be blamed for their coarse-grained analysis. In view of this, we convert these properties to their localized counterparts to capture the model behavior and loss properties that differ among classes. On one hand, the localized calibration reveals under which assumptions the loss function can be Fisher consistent, and the multiplicative adjustment is justified by consistency analysis for the first time. On the other hand, the localized Lipschitz continuity provides a fine-grained generalization analysis for a wide range of loss functions. Based on these insights, we propose a principled learning algorithm for our CVS loss. Empirical results performed on ResNets/ViTs not only validate the theoretical insights but also demonstrate the effectiveness of the proposed method.

In the future, the proposed learning algorithm can integrate with some orthogonal techniques, such as multi-expert \cite{DBLP:conf/nips/ZhangHHF22,DBLP:conf/icml/0001XWLHBCH24} and hard-sample exploration \cite{DBLP:conf/aaai/ZhaoC0HZ22}. Besides, considering the prevalence of imbalance, it is also interesting to explore the potential of the proposed method in other tasks, such as segmentation and retrieval.

\ifCLASSOPTIONcompsoc
\section*{Acknowledgments}
\else
\section*{Acknowledgment}
\fi

This work was supported in part by National Natural Science Foundation of China: 62525212, 62236008, 62025604, 62441232, U21B2038, U23B2051 and 62441619, in part by Youth Innovation Promotion Association CAS, in part by the Strategic Priority Research Program of the Chinese Academy of Sciences under Grant No. XDB0680201, in part by the China National Postdoctoral Program for Innovative Talents under Grant BX20240384, in part by Beijing Natural Science Foundation under Grant No. L252144, in part by General Program of the Chinese Postdoctoral Science Foundation under Grant No. 2025M771558, and in part by the Fundamental Research Funds for the Central Universities under Grant No. E4EQ1101.

\ifCLASSOPTIONcaptionsoff
\newpage
\fi

  \renewcommand{\bibfont}{\normalsize}
  \bibliographystyle{IEEEtranN}
  \bibliography{IEEEabrv,citations}


\begin{IEEEbiography}
    [{\includegraphics[width=1in,height=1.25in,clip,keepaspectratio]{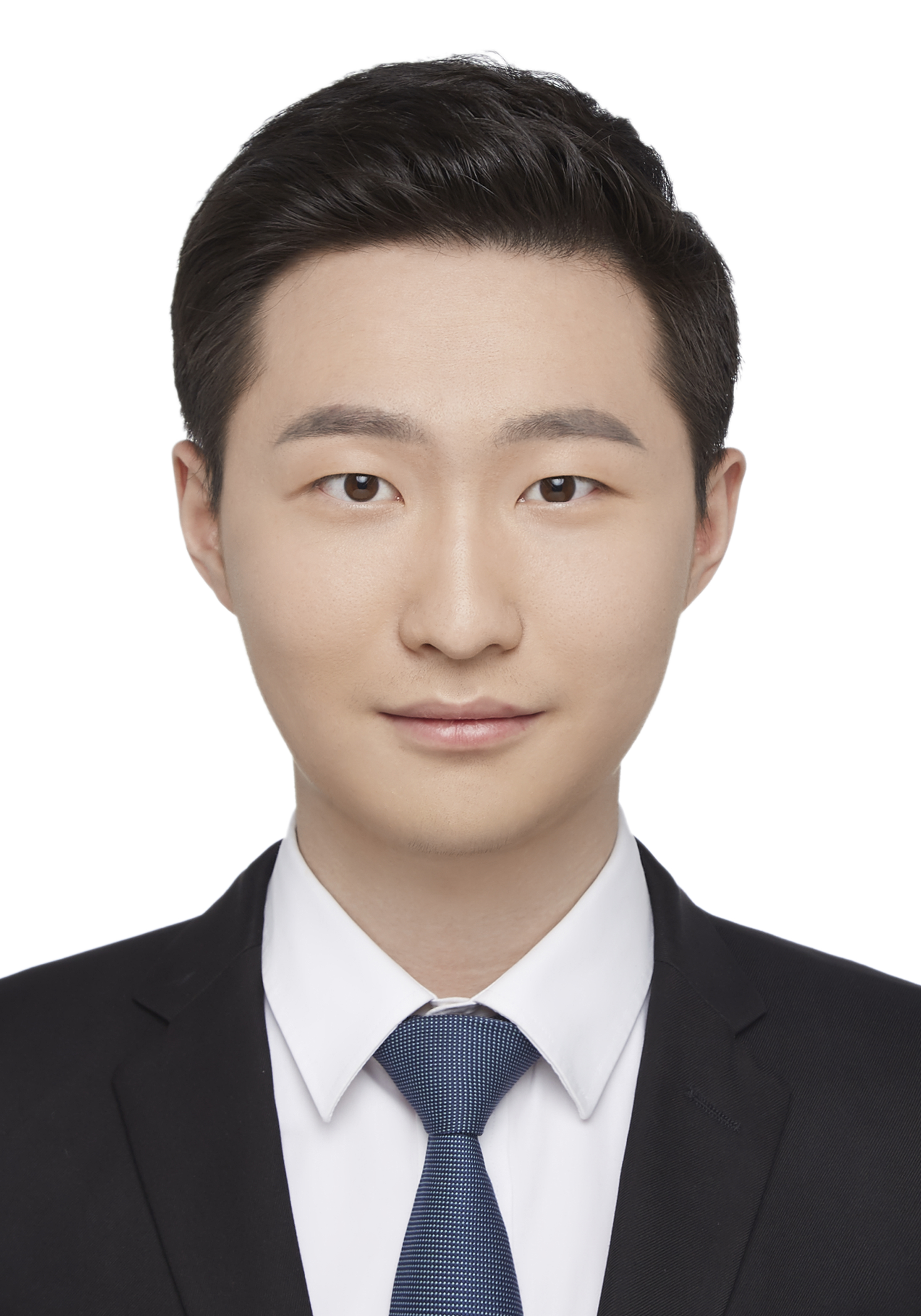}}]{Zitai Wang} received the B.E. degree in computer science and technology from Beijing Jiaotong University in 2019 and the Ph.D. degree in University of Chinese Academy of Sciences in 2024. He is currently a post-doctoral research fellow with the Institute of Computing Technology, Chinese Academy of Sciences. His research interests include machine learning and data mining. He has authored or coauthored 10+ academic papers in top-tier international conferences and journals including T-PAMI, IJCV, ICML, NeurIPS, AAAI, and ACM Multimedia. He served as a reviewer for several top-tier journals and conferences such as T-PAMI, T-CSVT, NeurIPS, ICLR, CVPR, AISTATS, and AAAI.
\end{IEEEbiography}

\begin{IEEEbiography}
	[{\includegraphics[width=1in,height=1.25in,clip,keepaspectratio]{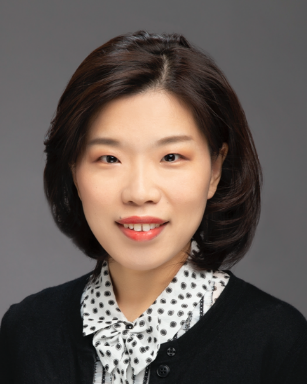}}]{Qianqian Xu} received the B.S. degree in computer science from China University of Mining and Technology in 2007 and the Ph.D. degree in computer science from University of Chinese Academy of Sciences in 2013. She is currently a Professor with the Institute of Computing Technology, Chinese Academy of Sciences, Beijing, China. Her research interests include statistical machine learning, with applications in multimedia and computer vision. She has authored or coauthored 100+ academic papers in prestigious international journals and conferences (including T-PAMI, IJCV, T-IP, NeurIPS, ICML, CVPR, AAAI, etc). Moreover, she serves as an associate editor of IEEE Transactions on Circuits and Systems for Video Technology, IEEE Transactions on Multimedia, and ACM Transactions on Multimedia Computing, Communications, and Applications.
\end{IEEEbiography}

\begin{IEEEbiography}
	[{\includegraphics[width=1in,height=1.25in,clip,keepaspectratio]{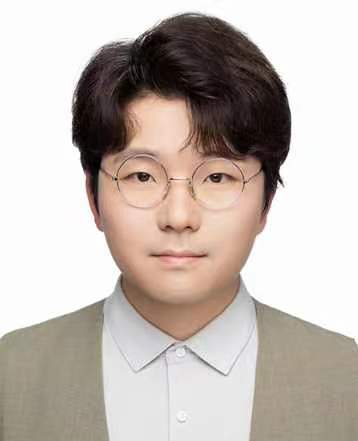}}]{Zhiyong Yang} received his M.S. degree in computer science and technology from the University of Science and Technology Beijing (USTB) in 2017, and Ph.D. degree from the University of Chinese Academy of Sciences (UCAS) in 2021. He is currently an Associate Professor at the University of Chinese Academy of Sciences. His research interests include trustworthy machine learning, long-tail learning, and optimization frameworks for complex metrics. He is one of the key developers of the XCurve learning framework (https://xcurveopt.github.io/), designed to address decision biases between model trainers and users. His work has been recognized with various awards, including Top 100 Baidu AI Chinese Rising Stars Around the World, Top-20 Nomination for the Baidu Fellowship, Asian Trustworthy Machine Learning (ATML) Fellowship, and the China Computer Federation (CCF) Doctoral Dissertation Award. He has authored or co-authored over 60 papers in top-tier international conferences and journals, including more than 30 papers in T-PAMI, ICML, and NeurIPS. He has also served as an Area Chair (AC) for NeurIPS 2024/ICLR 2025/ICML 2025, a Senior Program Committee (SPC) member for IJCAI 2021, and as a reviewer for several prestigious journals and conferences, such as T-PAMI, IJCV, TMLR, ICML, NeurIPS, and ICLR.
\end{IEEEbiography}

\begin{IEEEbiography}
	[{\includegraphics[width=1in,height=1.25in,clip,keepaspectratio]{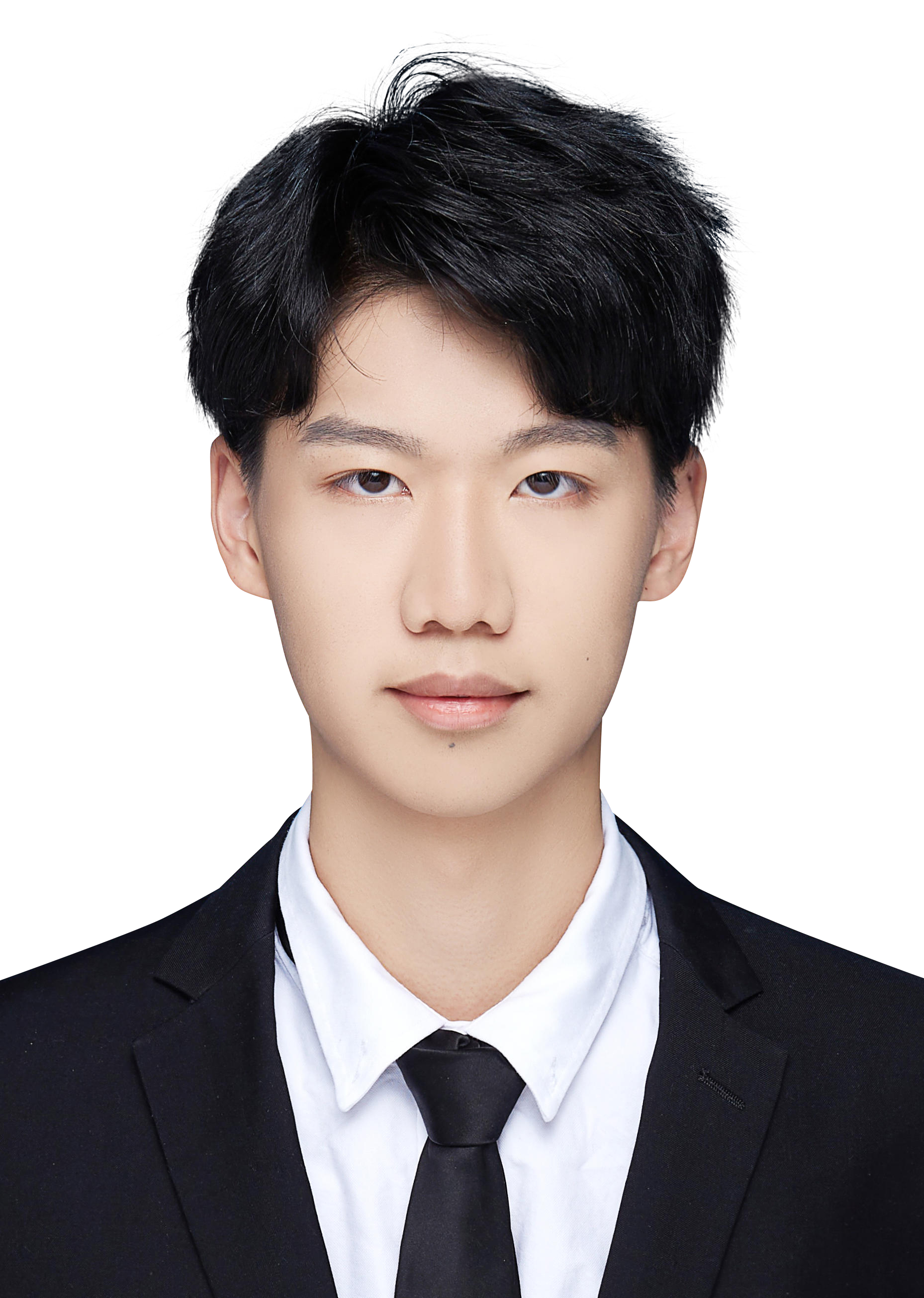}}]{Zhikang Xu} received the B.E. degree in Data Science and Big Data Technology from Tongji University in 2024. He is currently pursuing his M.S. degree from University of Chinese Academy of Sciences. His research interests include machine learning and computer vision.
\end{IEEEbiography}

\begin{IEEEbiography}
	[{\includegraphics[width=1in,height=1.25in,clip,keepaspectratio]{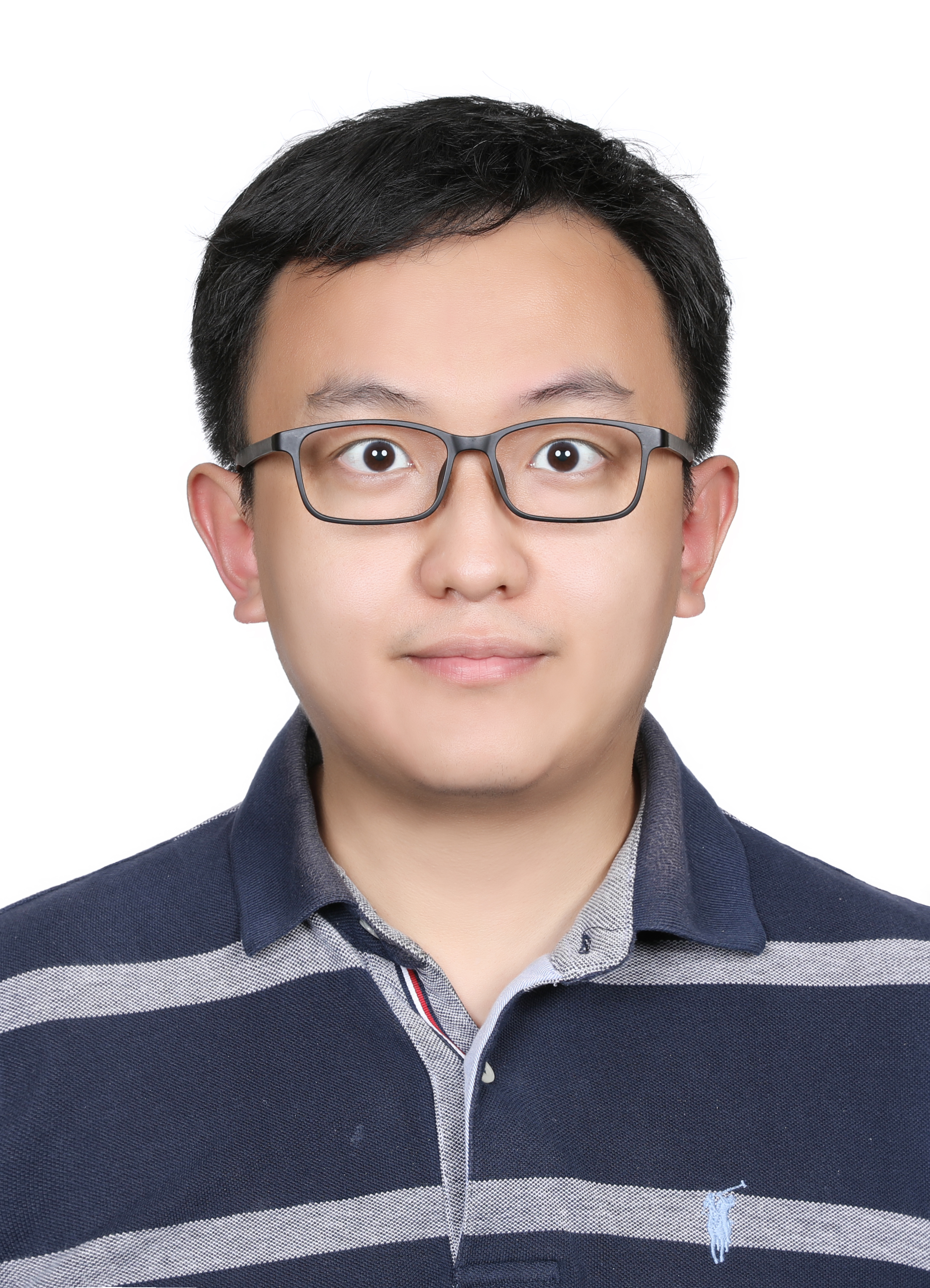}}]{Linchao Zhang} received the B.S. degree in computer science and technology from Harbin Institute of Technology, China in 2008, and the Ph.D. degree in information and system engineering from Politecnico di Torino, Italy in 2014. He is currently a senior engineer with Artificial Intelligence Institute of China Electronics Technology Group Corporation, Beijing, China. His research interests include trusted artificial intelligence and computer security.
\end{IEEEbiography}

\begin{IEEEbiography}
	[{\includegraphics[width=1in,height=1.25in,clip,keepaspectratio]{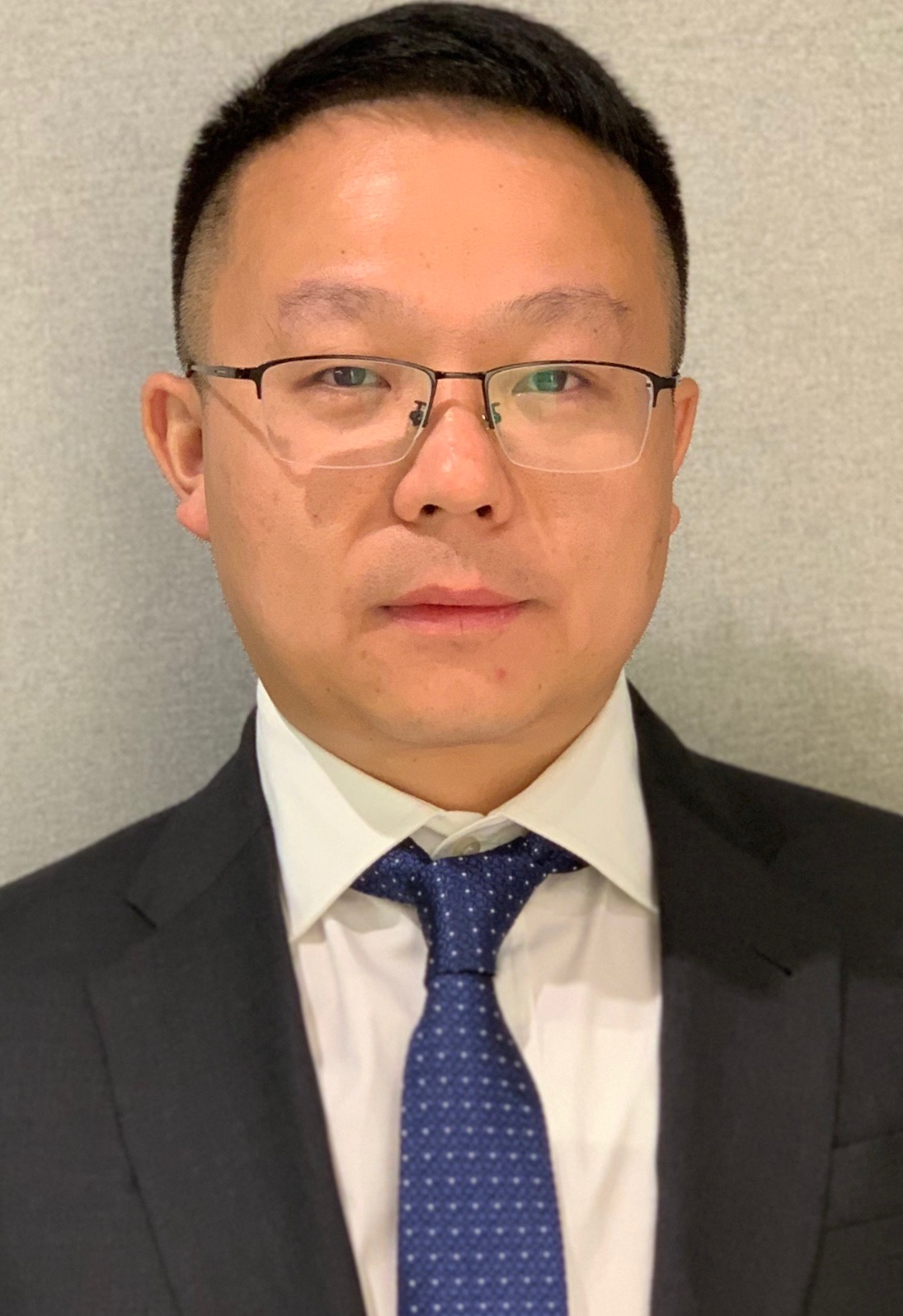}}]{Xiaochun Cao} is a Professor of School of Cyber Science and Technology, Shenzhen Campus of Sun Yat-sen University. He received the B.E. and M.E. degrees both in computer science from Beihang University (BUAA), China, and the Ph.D. degree in computer science from the University of Central Florida, USA, with his dissertation nominated for the university level Outstanding Dissertation Award. After graduation, he spent about three years at ObjectVideo Inc. as a Research Scientist. From 2008 to 2012, he was a professor at Tianjin University. Before joining SYSU, he was a professor at Institute of Information Engineering, Chinese Academy of Sciences. He has authored and coauthored over 200 journal and conference papers. In 2004 and 2010, he was the recipients of the Piero Zamperoni best student paper award at the International Conference on Pattern Recognition. He is on the editorial boards of IEEE Transactions on Image Processing and IEEE Transactions on Multimedia, and was on the editorial board of IEEE Transactions on Circuits and Systems for Video Technology.
\end{IEEEbiography}

\begin{IEEEbiography}
	[{\includegraphics[width=1in,height=1.25in,clip,keepaspectratio]{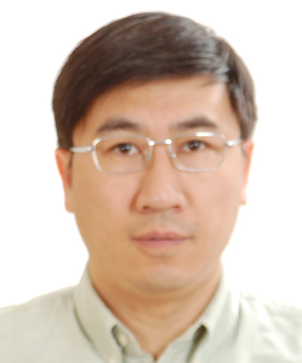}}]{Qingming Huang} is a chair professor in the University of Chinese Academy of Sciences and an adjunct research professor in the Institute of Computing Technology, Chinese Academy of Sciences. He graduated with a Bachelor degree in Computer Science in 1988 and Ph.D. degree in Computer Engineering in 1994, both from Harbin Institute of Technology, China. His research areas include multimedia computing, image processing, computer vision and pattern recognition. He has authored or coauthored more than 400 academic papers in prestigious international journals and top-level international conferences. He was the associate editor of IEEE Trans. on CSVT and Acta Automatica Sinica, and the reviewer of various international journals including IEEE Trans. on PAMI, IEEE Trans. on Image Processing, IEEE Trans. on Multimedia, etc. He is a Fellow of IEEE and has served as general chair, program chair, area chair and TPC member for various conferences, including ACM Multimedia, CVPR, ICCV, ICME, ICMR, PCM, BigMM, PSIVT, etc.
\end{IEEEbiography}


\clearpage
\onecolumn
\appendices
\section*{\textcolor{blue}{\Large{Contents}}}
\startcontents[sections]
\printcontents[sections]{l}{1}{\setcounter{tocdepth}{3}}
\newpage

\section{Proofs of Consistency Analysis}
\label{sec:proof_consistency}
\subsection{Proof of Prop.\ref{prop:local_linear_loss}}
\label{subsec:proof_local_linear_loss}
\locallinearloss*
\begin{proof}
    Let $\eta(\boldsymbol{x})_y$ denote the posterior probability $\mathbb{P}\left[ y \mid \boldsymbol{x} \right]$ and $f^*$ be the Bayes optimal score of the VS loss. According to Asm.\ref{asm:local_linear}, under constant re-weighting terms $\alpha_y = 1$, $f^*$ will satisfy 
    $$
        f^*(\boldsymbol{x})_y + \log \left( \delta_y / \kappa_y^+ \right) \propto  \log \left( \eta(\boldsymbol{x})_y / \kappa_y^+ \right) \Rightarrow f^*(\boldsymbol{x})_y \propto \log \left( \eta(\boldsymbol{x})_y / \delta_y \right).
    $$ 
    Then, the proof follows that in \cite{DBLP:conf/iclr/MenonJRJVK21}, which we attach right here for completeness. 
    Specifically, we next consider the re-weighting term $\alpha_y  = \delta_y / \pi_y$. Learning with this weighted loss is equivalent to learning with the unweighted one but with a modified class prior $\overline{\pi}_y = \pi_y \cdot \alpha_y$. Under this distribution, \citet{DBLP:conf/iclr/MenonJRJVK21} show that the posterior probability $\overline{\eta}(\boldsymbol{x})_y \propto \eta(\boldsymbol{x})_y \cdot \alpha_y$. Hence, the optimal prediction becomes
    $$
        f^*(\boldsymbol{x})_y \propto \log \left( \overline{\eta}(\boldsymbol{x})_y / \delta_y \right) = \log \left( \eta(\boldsymbol{x})_y / \pi_y \right) + A(\boldsymbol{x}),
    $$
    where $A(\boldsymbol{x})$ does not depend on $y$. Consequently, we have 
    $$
        \arg \max_{y} f^*(\boldsymbol{x})_y = \arg \max_{y} \log \left( \eta(\boldsymbol{x})_y / \pi_y \right) = \arg \max_{y} \eta(\boldsymbol{x})_y / \pi_y,
    $$
    which is exactly the Bayes optimal classifier for the balanced error.
\end{proof}

\subsection{Proof of Prop.\ref{prop:local_nonlinear_loss}}
\label{subsec:proof_local_nonlinear_loss}
\localnonlinearloss*
\begin{proof}
    Similarly, we first consider the case where $\alpha_y = 1$. Under Asm.\ref{asm:local_nonlinear}, the Bayes optimal score will satisfy
    $$
        f^*(\boldsymbol{x})_y \cdot \delta_y / \kappa_y^* = \eta(\boldsymbol{x})_y / \kappa_y^* \Rightarrow f^*(\boldsymbol{x})_y = \eta(\boldsymbol{x})_y / \delta_y.
    $$
    Next, we consider the weights. Similarly, the optimal prediction becomes
    $$
        f^*(\boldsymbol{x})_y \propto \frac{\overline{\eta}(\boldsymbol{x})_y}{\delta_y} \propto \frac{\eta(\boldsymbol{x})_y}{\pi_y}.
    $$
    Consequently, we have 
    $$
        \arg \max_{y} f^*(\boldsymbol{x})_y = \arg \max_{y} \frac{\eta(\boldsymbol{x})_y}{\pi_y}, 
    $$
    which is exactly the Bayes optimal classifier for the balanced error.
\end{proof}

\section{Proofs of Generalization Analysis}
\subsection{Proof of the Basic Lemma (Lem.\ref{lem:basic_lem})}
\label{app:basic_lem}
\basiclem*
\begin{proof}
  On one hand, 
    \begin{equation}
        \mathcal{R}^L (f) = \E{(\boldsymbol{x}, y) \sim \mathcal{D}}{L(f(\boldsymbol{x}), y)}  = \sum_{y=1}^{C} \pi_y \E{\boldsymbol{x} \sim \mathcal{D}_y}{L(f(\boldsymbol{x}), y)} = \sum_{j=1}^C \pi_y \mathcal{R}_y^L(f).
    \end{equation}
    On the other hand,
    \begin{equation}
        \mathcal{R}_\text{bal}^L (f) = \frac{1}{C} \sum_{j=1}^C \mathcal{R}_y^L(f) = \frac{1}{C} \sum_{j=1}^C \frac{1}{\pi_y} \cdot \pi_y \mathcal{R}_y^L(f) \le \frac{1}{C \pi_C} \sum_{j=1}^C \pi_y \mathcal{R}_y^L(f) = \frac{1}{C \pi_C} \mathcal{R}^L(f),
    \end{equation}
    where the inequality comes from the fact that $\forall \boldsymbol{a}, \boldsymbol{b} \in \mathbb{R}^C, \vert \left\langle \boldsymbol{a}, \boldsymbol{b} \right\rangle  \vert \le \Vert \boldsymbol{a} \Vert_\infty \Vert \boldsymbol{b} \Vert_1 $. Then, combining the traditional results in \cite{10.5555/2371238}, for any $\delta \in (0, 1)$, with probability at least $1 - \delta$ over the training set $\mathcal{S}$, the following generalization bound holds for all the $g \in \mathcal{G}$:
    \begin{equation}
        \mathcal{R}_\text{bal}^L (f) \precsim \frac{1}{C \pi_C} \left[ \widehat{\mathcal{R}}^L(f) + \hat{\mathfrak{C}}_{\mathcal{S}}(\mathcal{G}) + 3 M \sqrt{\frac{\log 2 / \delta}{2 N}} \right].
    \end{equation}
\end{proof}

\subsection{Proof of the Data-Dependent Contraction Lemma (Lem.\ref{lem:data_dependent_contraction})}
\label{app:data_dependent_contraction}
\datadeplem*
\begin{proof}
    According to the definition of complexity, we have
    \begin{equation}
        \begin{aligned}
            \hat{\mathfrak{C}}_{\mathcal{S}}(\mathcal{G}) & = \E{\boldsymbol{\xi}}{\sup_{g \in \mathcal{G}} \frac{1}{N}\sum_{n=1}^{N} \xi^{(n)} g(\boldsymbol{z}^{(n)}) } = \E{\boldsymbol{\xi}}{\sup_{g \in \mathcal{G}} \frac{1}{N} \sum_{y=1}^{C} \sum_{n=1}^{N_y} \xi_y^{(n)} g(\boldsymbol{z}_y^{(n)}) } \\
            & \le \sum_{y=1}^{C} \E{\boldsymbol{\xi}_y}{\frac{1}{N} \sup_{g \in \mathcal{G}} \sum_{n=1}^{N_y} \xi_y^{(n)} g(\boldsymbol{z}_y^{(n)}) } = \sum_{y=1}^{C} \frac{N_y}{N} \E{\boldsymbol{\xi}_y}{\frac{1}{N_y} \sup_{g \in \mathcal{G}} \sum_{n=1}^{N_y} \xi_y^{(n)} g(\boldsymbol{z}_y^{(n)}) } \\
            & = \sum_{y=1}^{C}\pi_y \hat{\mathfrak{C}}_{\mathcal{S}_y}(\mathcal{G}) \precsim \sum_{y=1}^{C} \sqrt{\pi_y} \mu_{y} \hat{\mathfrak{C}}_{\mathcal{S}}(\mathcal{F}), \\
        \end{aligned}
    \end{equation}
    where the last inequality comes from Asm.\ref{asm:complexity}.
\end{proof}

\subsection{Proof of the Local Lipschitz Property of the VS Loss (Prop.\ref{lem:Lip_of_vs})}
\label{app:Lip_of_vs}
\begin{lemma}
    \label{lem:squre_ineq}
    Given $\{a_i\}_{i=1}^C, \{b_i\}_{i=1}^C$, if $b_i \ge 0$, we have $\sum_{i=1}^C a_i^2 b_i^2 \le \left( \sum_{i=1}^C a_i^2 \right) \left( \sum_{i=1}^C b_i \right)^2$.
\end{lemma}
\begin{proof}
    According to the definition, we have
    \begin{equation}
        \left( \sum_{i=1}^C a_i^2 \right) \left( \sum_{i=1}^C b_i \right)^2 = \sum_{i=1}^{C} a_i^2 \left( \sum_{j=1}^C b_j \right)^2 = \sum_{i=1}^{C} a_i^2 \left( \sum_{j=1}^C b_j^2 + 2 \sum_{j \neq k} b_j b_k \right) \ge \sum_{i=1}^C a_i^2 b_i^2.
    \end{equation}
\end{proof}

\lipofmargin*
\begin{proof}
    According to the definition of the VS loss, we have
    \begin{equation}
        \begin{aligned}
            L_\text{VS}(f(\boldsymbol{x}), y) & = - \alpha_y \log \left( \frac{e^{\beta_y f(\boldsymbol{x})_y + \Delta_y}}{\sum_{y'} e^{\beta_{y'} f(\boldsymbol{x})_{y'} + \Delta_{y'}} } \right) \\
            & = \alpha_y \log [1 + \sum_{y' \neq y} e^{\beta_{y'} f(\boldsymbol{x})_{y'} - \beta_y f(\boldsymbol{x})_y + \Delta_{y'} - \Delta_y}], 
        \end{aligned}
    \end{equation}
    Let $\boldsymbol{s} := f(\boldsymbol{x})$, and define
    \begin{equation}
        \ell_y(\boldsymbol{s}) := \sum_{y' \neq y} e^{\beta_{y'} \boldsymbol{s}_{y'} + \Delta_{y'}}.
    \end{equation}
    In other words, $L_\text{VS}(f, y) = \alpha_y \log \left[ 1 + e^{- \left( \beta_{y} \boldsymbol{s}_{y} + \Delta_{y} \right)} \ell_y(\boldsymbol{s}) \right]$. Then, 
    \begin{equation}
        \begin{aligned}
            \frac{\partial L_\text{VS}(f, y)}{\partial \boldsymbol{s}_{y}} & = - \alpha_y \beta_y \frac{e^{- \left( \beta_{y} \boldsymbol{s}_{y} + \Delta_{y} \right)} \ell_y(\boldsymbol{s})}{1 + e^{- \left( \beta_{y} \boldsymbol{s}_{y} + \Delta_{y} \right)} \ell_y(\boldsymbol{s})}, \\
            \frac{\partial L_\text{VS}(f, y)}{\partial \boldsymbol{s}_{y'}} & = \alpha_y \beta_{y'} \frac{e^{- \left( \beta_{y} \boldsymbol{s}_{y} + \Delta_{y} \right)}}{1 + e^{- \left( \beta_{y} \boldsymbol{s}_{y} + \Delta_{y} \right)} \ell_y(\boldsymbol{s})} \cdot e^{\beta_{y'} \boldsymbol{s}_{y'} + \Delta_{y'}}, y' \neq y. \\
        \end{aligned}
    \end{equation}
    Hence,
    \begin{equation}
        \begin{aligned}
            & \Vert \nabla_{\boldsymbol{s}} L_\text{VS}(f, y) \Vert^2 = \left[ \beta_y^2 \ell_y(\boldsymbol{s})^2 + \sum_{y' \neq y} \left(\beta_{y'} e^{\beta_{y'} \boldsymbol{s}_{y'} + \Delta_{y'}} \right)^2 \right] \cdot \left[ \frac{\alpha_{y} e^{- \left( \beta_y \boldsymbol{s}_{y} + \Delta_{y} \right)}}{1 + e^{- \left( \beta_{y} \boldsymbol{s}_{y} + \Delta_{y} \right)} \ell_y(\boldsymbol{s})} \right]^2  \\
            & \le \left[ \beta_y^2 \ell_y(\boldsymbol{s})^2 + \left( \sum_{y' \neq y} \beta_{y'}^2 \right) \left( \sum_{y' \neq y} e^{\beta_{y'} \boldsymbol{s}_{y'} + \Delta_{y'}} \right)^2 \right] \cdot \left[ \frac{\alpha_{y} e^{- \left( \beta_y \boldsymbol{s}_{y} + \Delta_{y} \right)}}{1 + e^{- \left( \beta_{y} \boldsymbol{s}_{y} + \Delta_{y} \right)} \ell_y(\boldsymbol{s})} \right]^2 \\
            & = \left( \sum_{y'=1}^{C} \beta_{y'}^2 \right) \cdot \left[ \frac{\alpha_{y} e^{- \left( \beta_y \boldsymbol{s}_{y} + \Delta_{y} \right)} \ell_y(\boldsymbol{s}) }{1 + e^{- \left( \beta_{y} \boldsymbol{s}_{y} + \Delta_{y} \right)} \ell_y(\boldsymbol{s})} \right]^2,
        \end{aligned}
    \end{equation} 
    where the inequality comes from Lem.\ref{lem:squre_ineq}. Thus, let $\tilde{\beta} = \sqrt{\sum_{y'=1}^{C} \beta_{y'}^2}$, we have
    \begin{equation}
        \begin{aligned}
            \Vert \nabla_{\boldsymbol{s}} L_\text{VS}(f, y) \Vert & \le \alpha_{y} \tilde{\beta} \frac{e^{- \left( \beta_y \boldsymbol{s}_{y} + \Delta_{y} \right)} \ell_y(\boldsymbol{s}) }{1 + e^{- \left( \beta_{y} \boldsymbol{s}_{y} + \Delta_{y} \right)} \ell_y(\boldsymbol{s})} = \alpha_{y} \tilde{\beta} \frac{\ell_y(\boldsymbol{s})}{e^{  \beta_{y} \boldsymbol{s}_{y} + \Delta_{y} } +  \ell_y(\boldsymbol{s})} \\
            & = \alpha_{y} \tilde{\beta} \left[ 1 - \frac{e^{  \beta_{y} \boldsymbol{s}_{y} + \Delta_{y} }}{ \sum_{y'} e^{  \beta_{y'} \boldsymbol{s}_{y'} + \Delta_{y'} }} \right] = \alpha_{y} \tilde{\beta} \left[ 1 - \textit{softmax}\left( \beta_{y} \boldsymbol{s}_{y} + \Delta_{y} \right) \right] \\
        \end{aligned}
    \end{equation}
    Since the score function is bounded, for any $y \in \mathcal{Y}$, there exists a constant $B_y(f)$ such that $B_y(f) = \inf_{\boldsymbol{x} \in \mathcal{S}_y} \boldsymbol{s}_y$, which completes the proof.
\end{proof}

\begin{figure*}[h]
    \centering
    \includegraphics[width=0.9\textwidth]{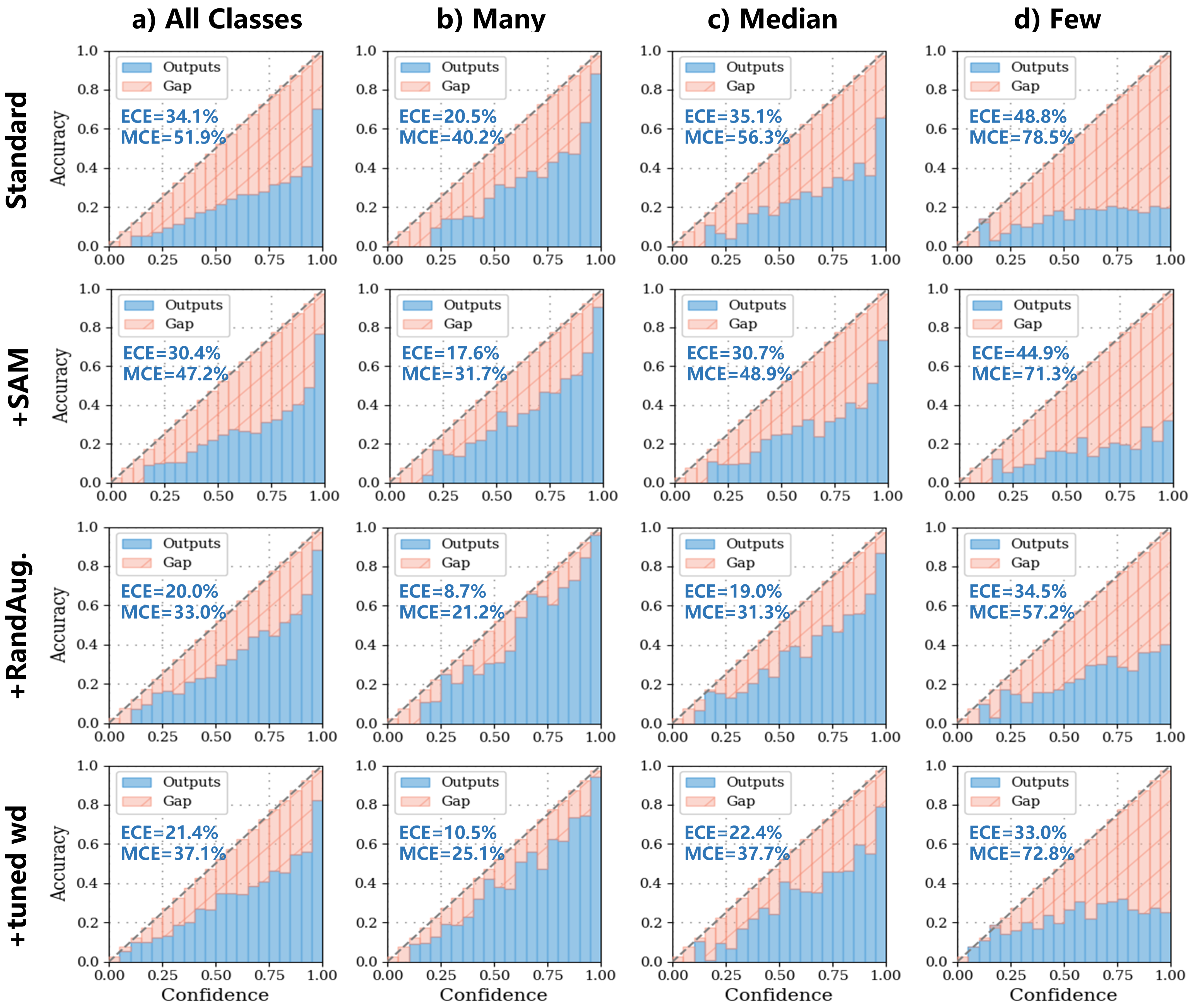}
    \caption{The calibration issue of minority classes: Although the selected techniques can improve the overall calibration (\textbf{a}), the model still exhibits poor calibration, especially in the minority classes (\textbf{d}).}
    \label{fig:local_calibration_exp}
\end{figure*}

\section{More Empirical Results}
\label{app:more_exp_results}
\subsection{Further Validation of the Calibration Issue}
Sec.\ref{subsec:calibration_issue} has shown that even with the mixup technique, the minority classes are not well calibrated. In Fig.\ref{fig:local_calibration_exp}, we additionally select three techniques commonly used by loss-oriented methods, including SAM \cite{DBLP:conf/iclr/ForetKMN21,rangwani2022escaping}, RandAugment \cite{DBLP:conf/nips/CubukZS020}, and tuned weight decay \cite{DBLP:conf/cvpr/AlshammariWRK22}. From the results, we can find that
\begin{itemize}
    \item All these techniques can improve the overall calibration to some extent. However, they are less effective than the mixup technique, no matter on the majority or minority classes. 
    \item Similar to the observations in Fig.\ref{fig:local_calibration}, the model exhibits poor calibration in the minority classes. This phenomenon again validates the necessity of our analysis in Sec.\ref{subsec:calibration_issue}.
\end{itemize}

\subsection{Validaion of Generalization Analysis}

Fig.\ref{fig:alpha_ratio10} reports more results of the baseline models on the CIFAR datasets with $\rho = 10$. Again, {\color[RGB]{230,180, 80} CE+ADRW} and {\color[RGB]{94, 227, 206} LDAM} perform better than {\color[RGB]{127, 231, 153} CE}, and {\color[RGB]{252, 136, 123} LDAM+ADRW} outperforms both {\color[RGB]{230,180, 80} CE+ADRW} and {\color[RGB]{94, 227, 206} LDAM}. All These results again validate the theoretical insights \textbf{(In1)} and \textbf{(In4-b)}.

Fig.\ref{fig:sensitivity_app} presents more sensitivity analysis of VS+ADRW on the CIFAR datasets. Similar to the results in Fig.\ref{fig:sensitivity_cifar10}, appropriately increasing $\nu$ and $\tau$ can improve the model performance, which again validates the theoretical insights \textbf{(In1)} and \textbf{(In4-b)}. Notably,  CIFAR-10 ($\rho = 100$), \textit{i.e.}, Fig.\ref{fig:sensitivity_cifar10}, shows a similar trend to CIFAR-100 ($\rho = 100$), \textit{i.e.}, (a) and (b), instead of those with $\rho = 10$, \textit{i.e.}, (c) and (d). This shows that the datasets with similar  imbalance ratios may have similar optimal hyperparameters, which is beneficial for hyperparameter searching.

Fig.\ref{fig:drw_app} provides a series of results on the CIFAR-10 dataset to validate the theoretical insight \textbf{(In2)}, similar to Fig.\ref{fig:drw_cifar100}. Once again, the imbalance of $B_y(f)$ is highly correlated with the model performance on the test set, not only for CE+DRW but also for CE+None. Besides, the optimal DRW epoch $T_0$ is around 60 when $\rho = 100$ but larger than 100 when $\rho = 10$. This again shows that the optimal hyperparameters are highly related to the imbalance ratio. For efficiency, we fix $T_0$ to 160 in Sec.\ref{subsec:performance_comparison}, which is a common choice in practice.

\begin{figure}[h]
    \centering
    \subfigure[CIFAR-10 LT]{
      \includegraphics[width=0.23\columnwidth]{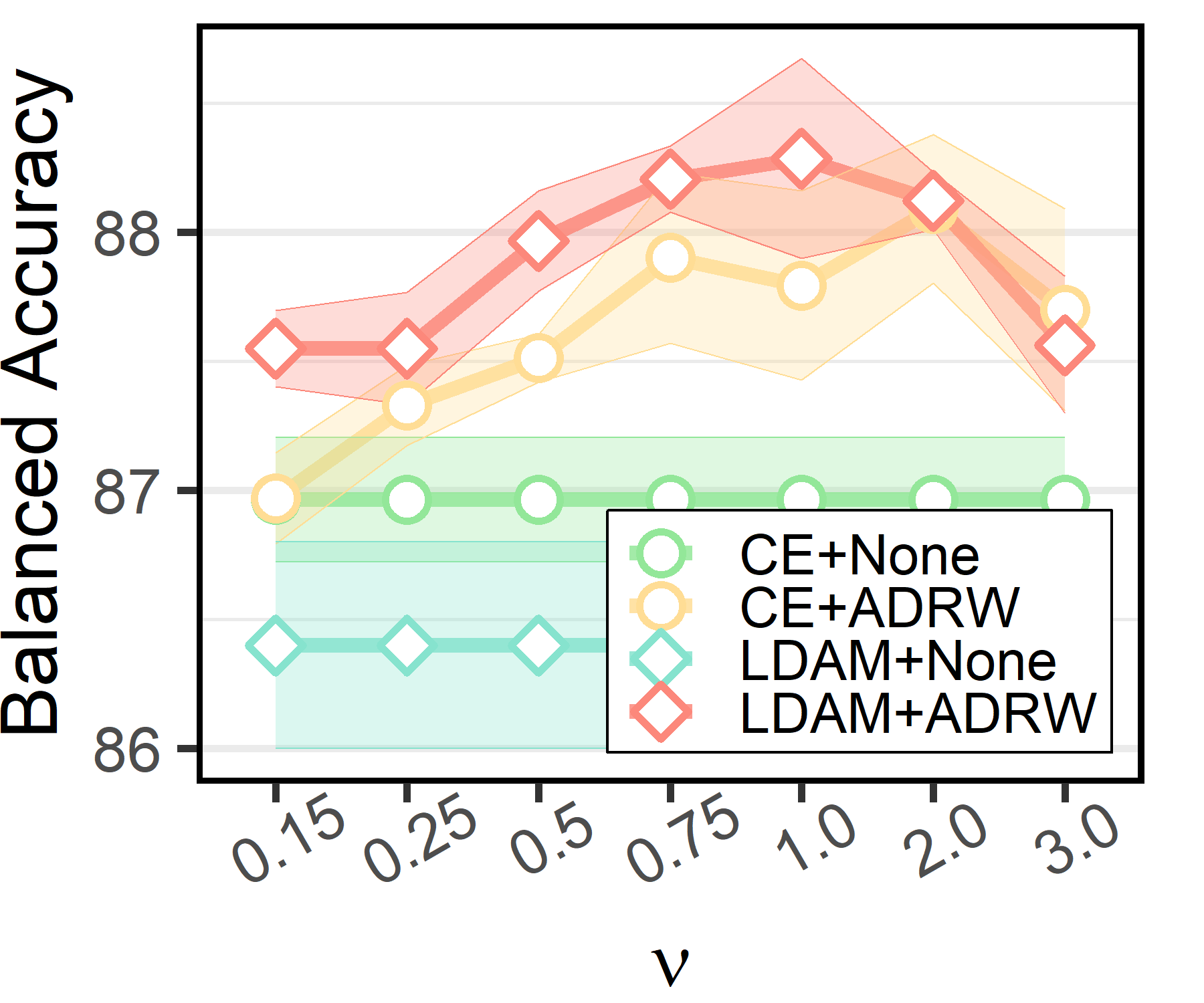}
     }
    \subfigure[CIFAR-10 Step]{
      \includegraphics[width=0.23\columnwidth]{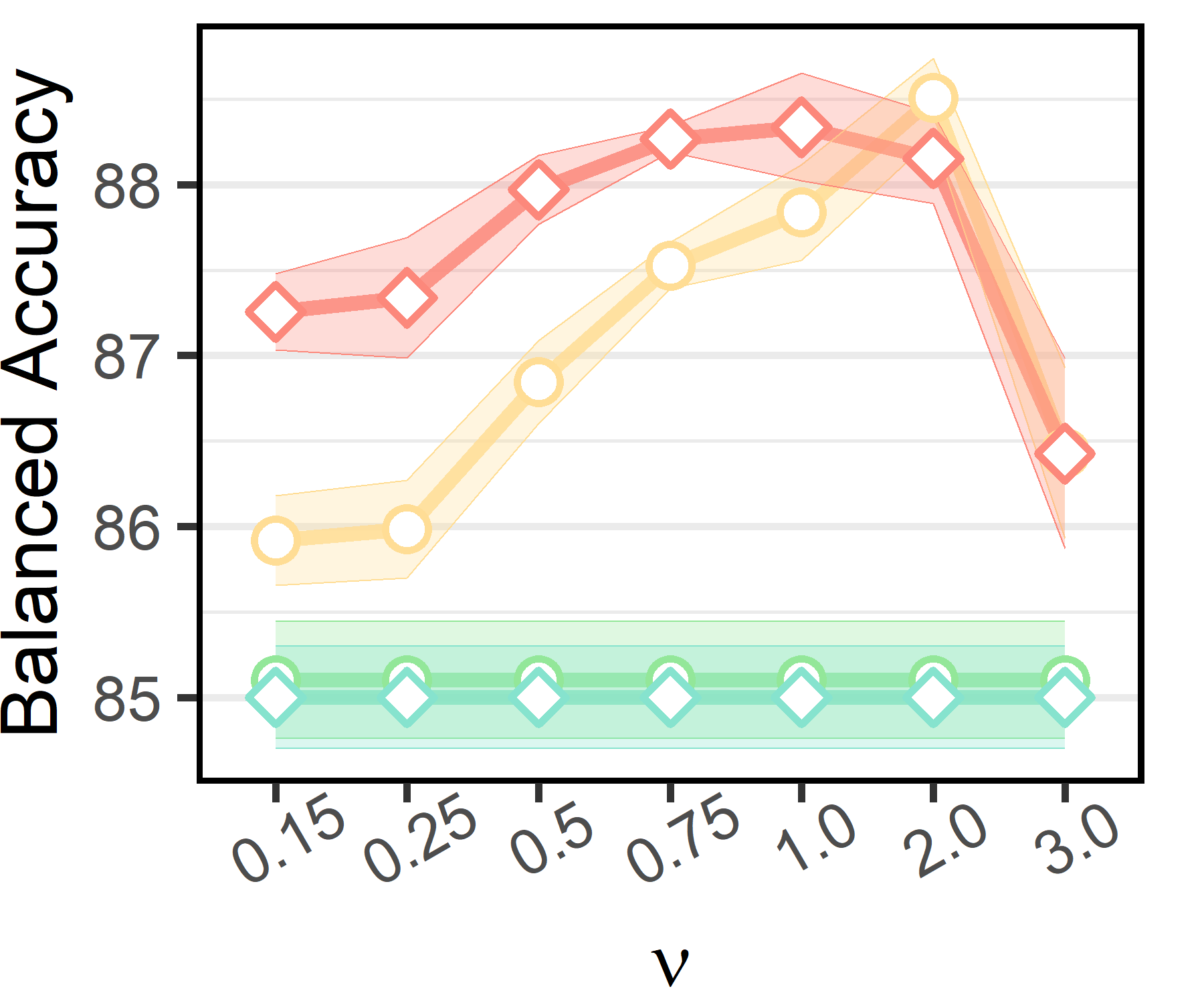}
     }
    \subfigure[CIFAR-100 LT]{
      \includegraphics[width=0.23\columnwidth]{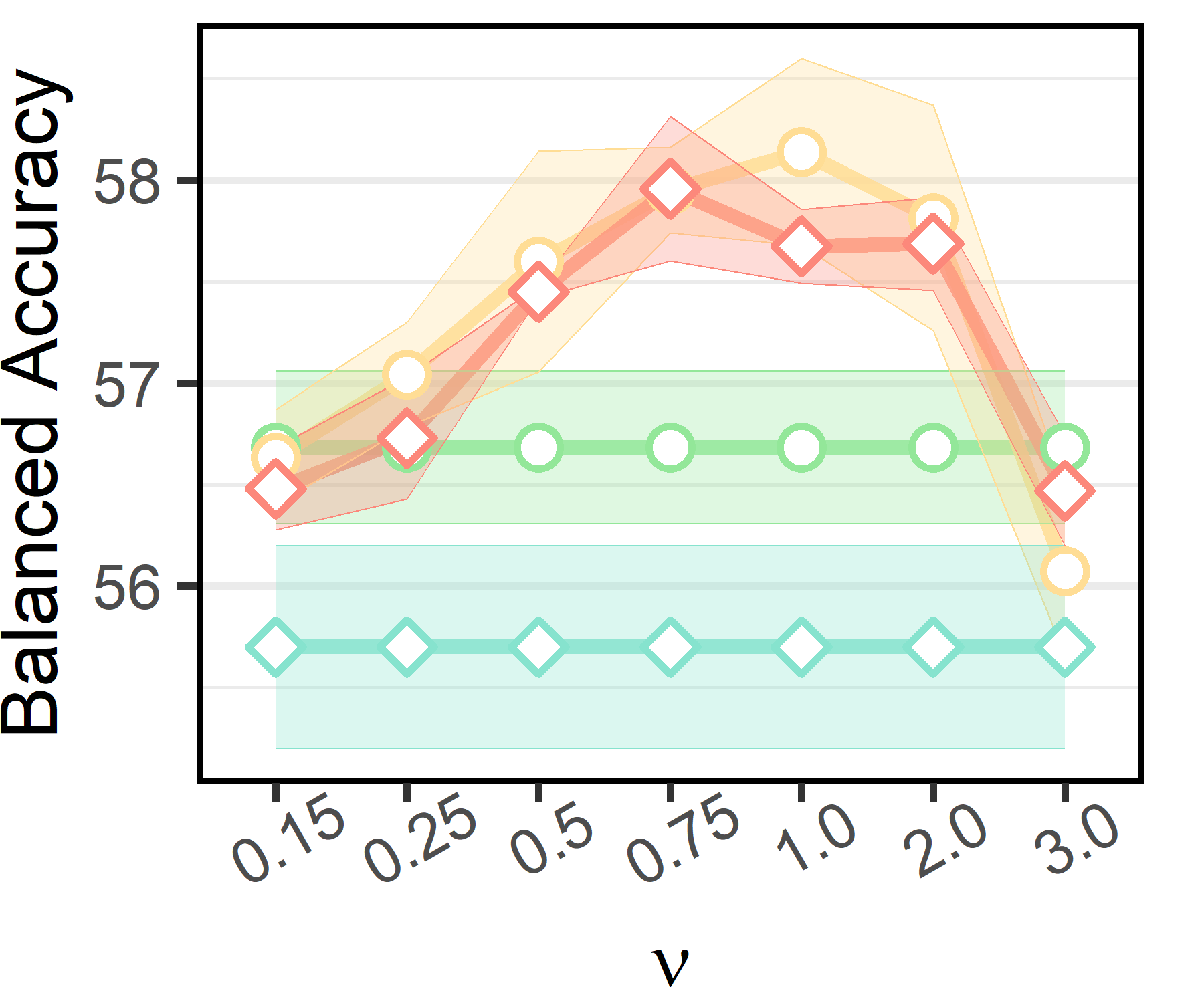}
     }
    \subfigure[CIFAR-100 Step]{
    \includegraphics[width=0.23\columnwidth]{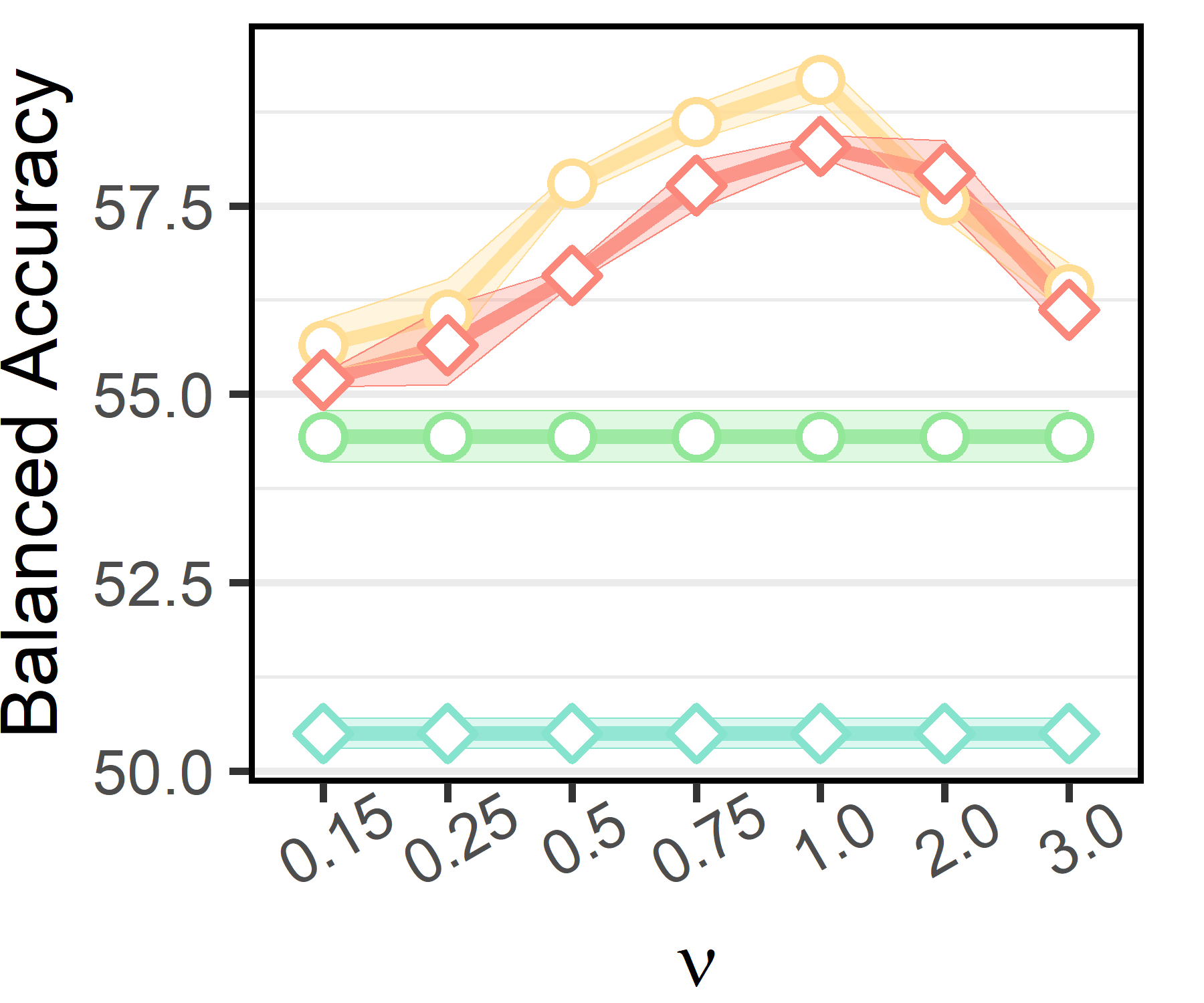}
    }
    \vspace{-2mm}
    \caption{The balanced accuracy of the CE loss and the LDAM loss \textit{w.r.t.} $\alpha_y \propto \pi_y^{-\nu}$ on the CIFAR datasets, where the imbalance ratio $\rho = 10$. Both re-weighting and logit-adjustment boost the model performance, which is consistent with the theoretical insight \textbf{(In1)} and \textbf{(In4-b)}.}
    \label{fig:alpha_ratio10}
\end{figure}

\begin{figure}[h]
    \centering
    \subfigure[CIFAR-100 LT ($\rho = 100$)]{
      \includegraphics[width=0.42\columnwidth]{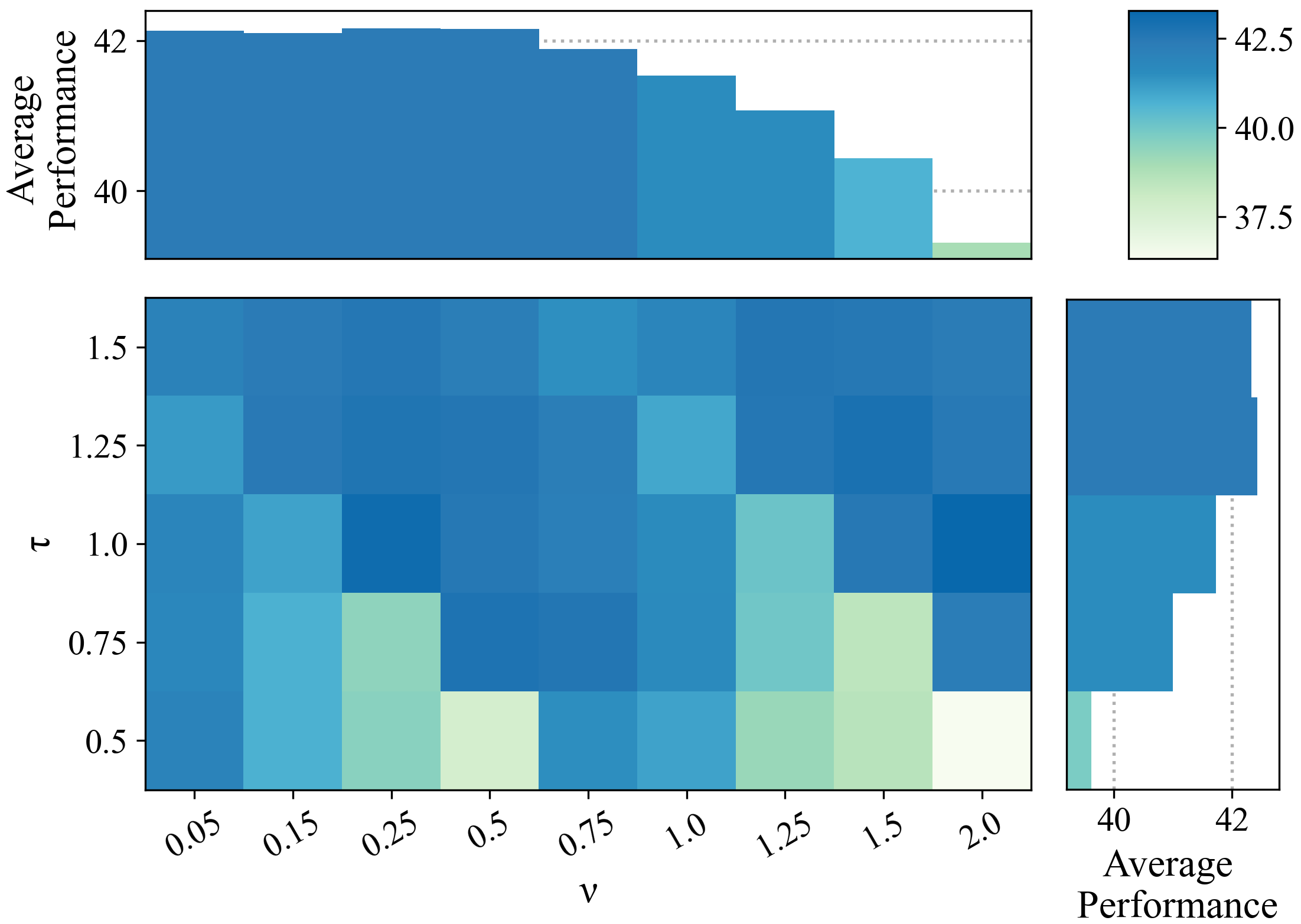}
     }
    \subfigure[CIFAR-100 Step ($\rho = 100$)]{
      \includegraphics[width=0.42\columnwidth]{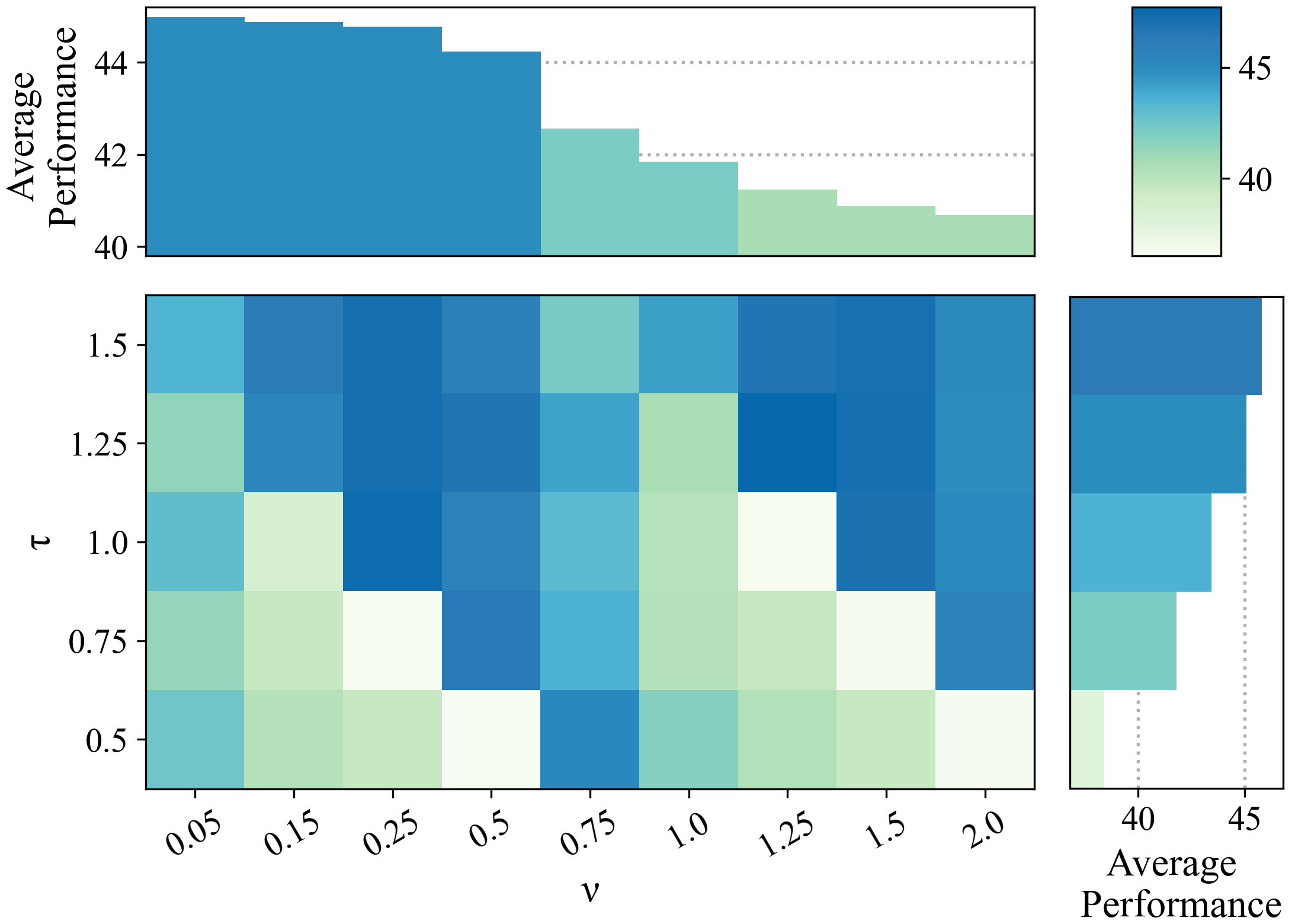}
     }

     \subfigure[CIFAR-100 LT ($\rho = 10$)]{
        \includegraphics[width=0.42\columnwidth]{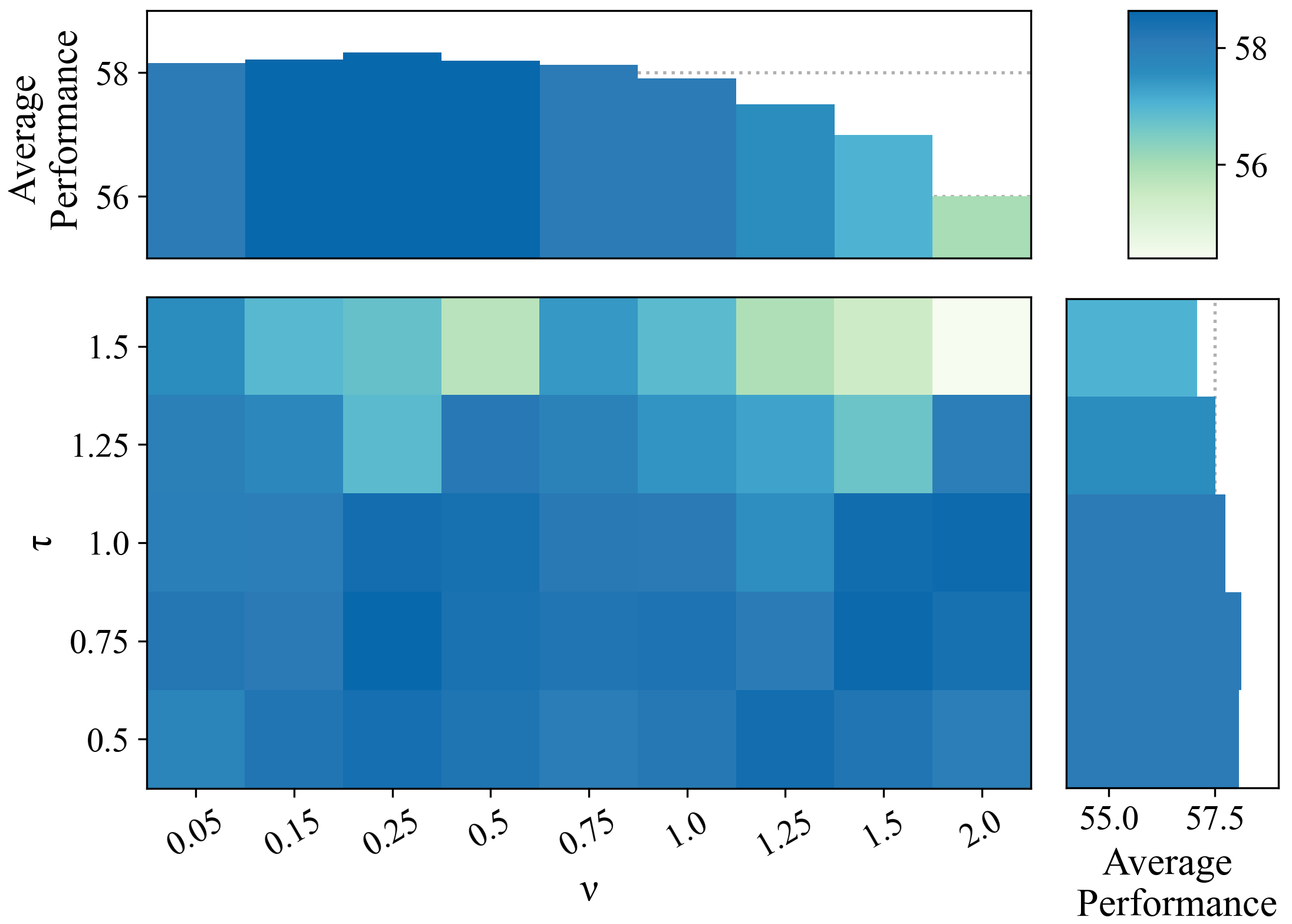}
       }
      \subfigure[CIFAR-10 LT ($\rho = 10$)]{
        \includegraphics[width=0.42\columnwidth]{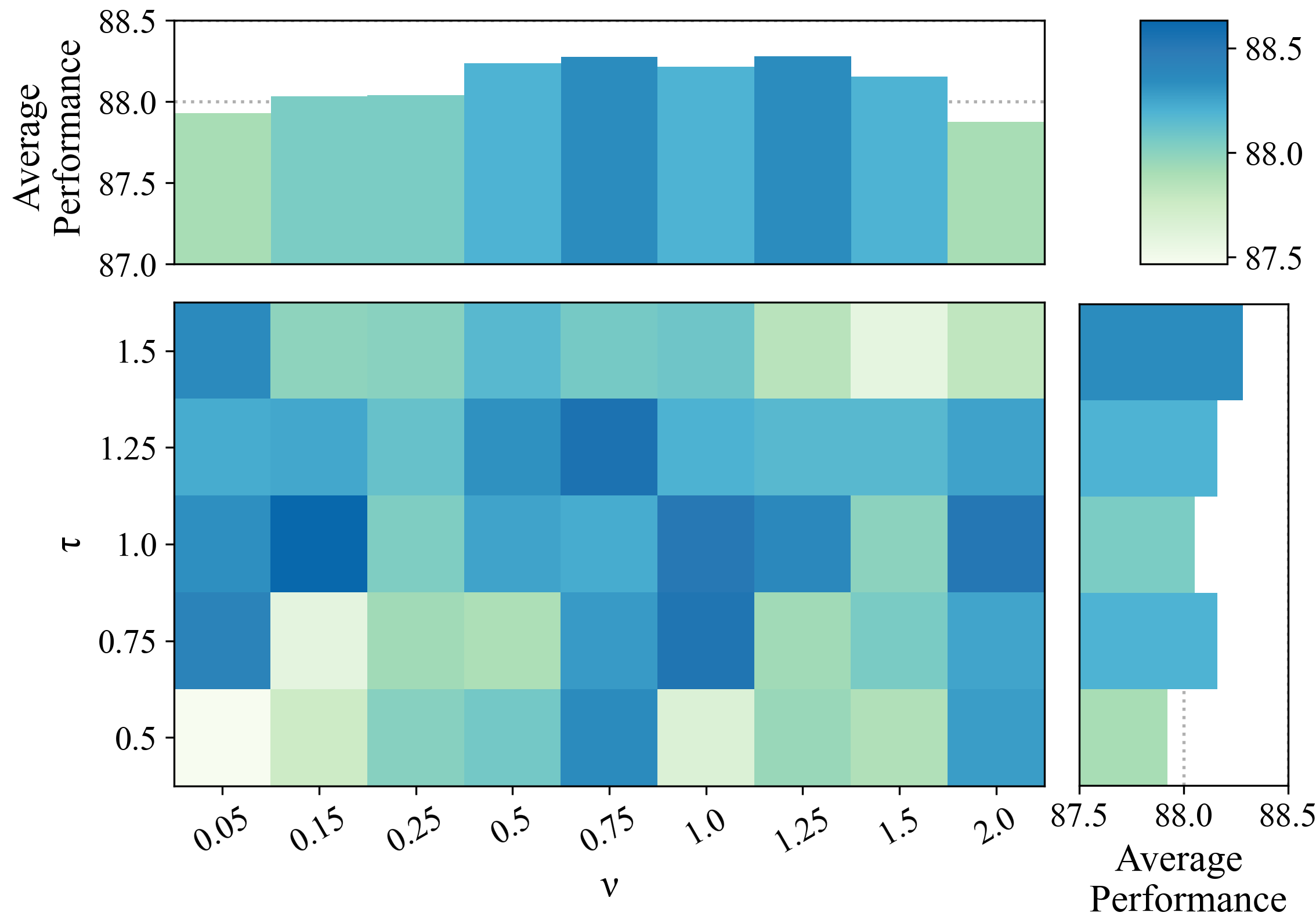}
       }
     \caption{Sensitivity analysis of VS+ADRW \textit{w.r.t.} $\alpha_y \propto \pi_y^{-\nu}$ and $\Delta_y = \tau \log \pi_y$ on the CIFAR datasets, where the imbalance ratio $\rho = 100$. Both re-weighting and logit-adjustment boost the model performance, which is consistent with the theoretical insights \textbf{(In1)} and \textbf{(In4-b)}.}
     \label{fig:sensitivity_app}
\end{figure}

\begin{figure}[!t]
    \centering
    \subfigure[CIFAR-10 LT ($\rho = 100$)]{
        \includegraphics[width=0.28\columnwidth]{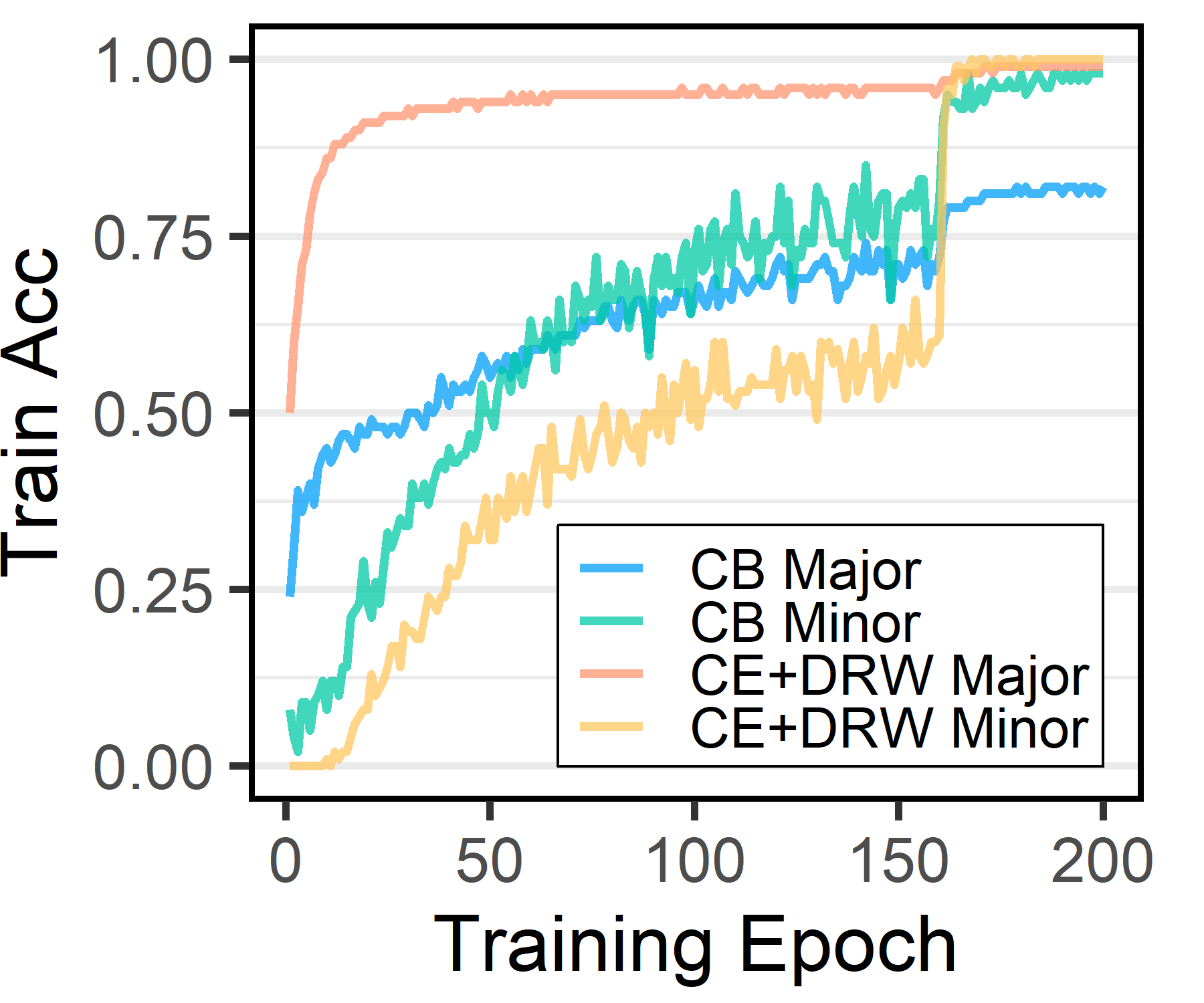}
    }
    \subfigure[CIFAR-10 LT ($\rho = 100$)]{
    \includegraphics[width=0.28\columnwidth]{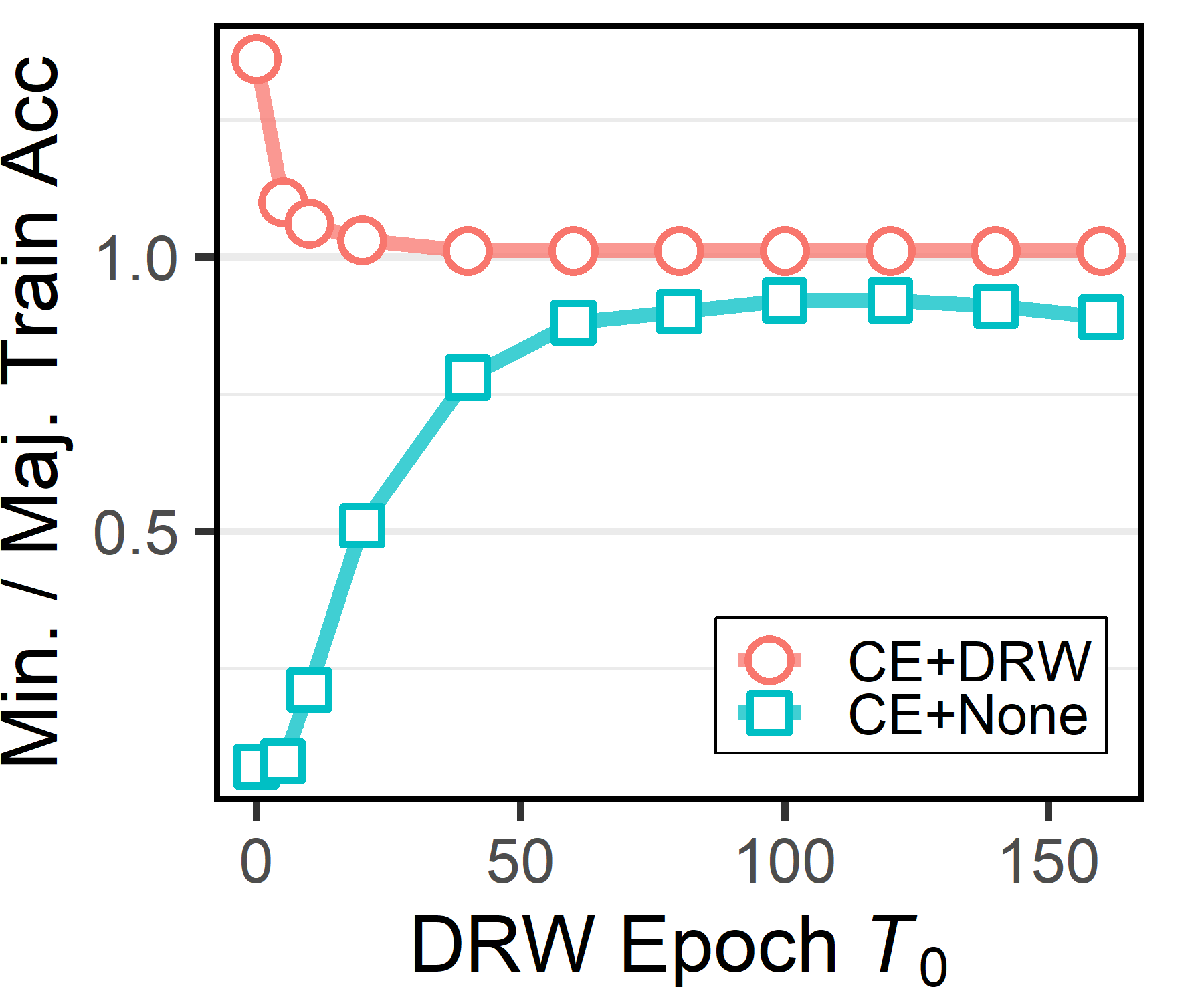}
    }
    \subfigure[CIFAR-10 LT ($\rho = 100$)]{
      \includegraphics[width=0.28\columnwidth]{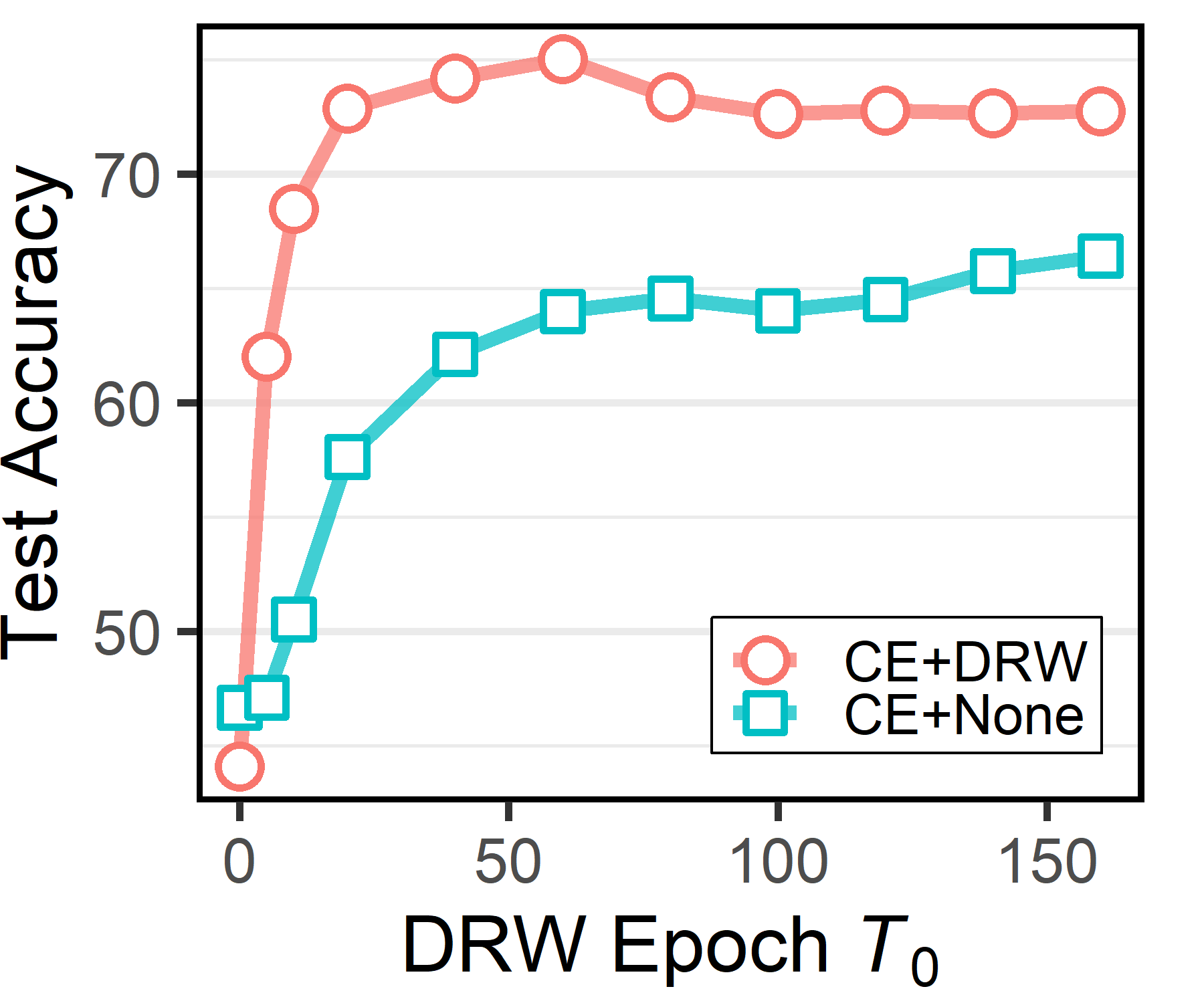}
    }

    \subfigure[CIFAR-100 LT ($\rho = 10$)]{
        \includegraphics[width=0.28\columnwidth]{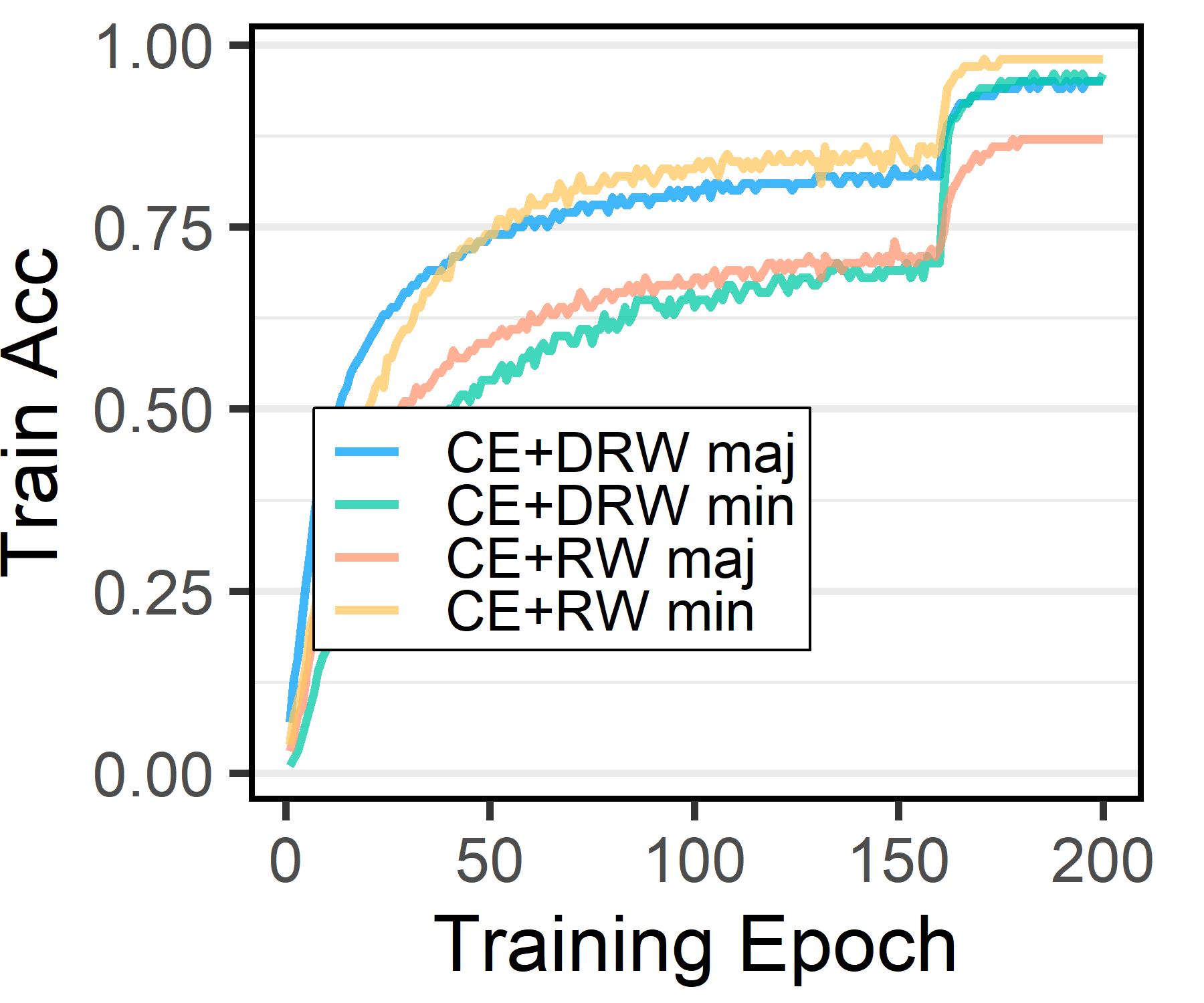}
    }
    \subfigure[CIFAR-100 LT ($\rho = 10$)]{
    \includegraphics[width=0.28\columnwidth]{drw-cifar100,exp,ratio.png}
    }
    \subfigure[CIFAR-100 LT ($\rho = 10$)]{
      \includegraphics[width=0.28\columnwidth]{drw-cifar100,exp,acc.png}
    }
    
    \subfigure[CIFAR-10 LT ($\rho = 10$)]{
        \includegraphics[width=0.28\columnwidth]{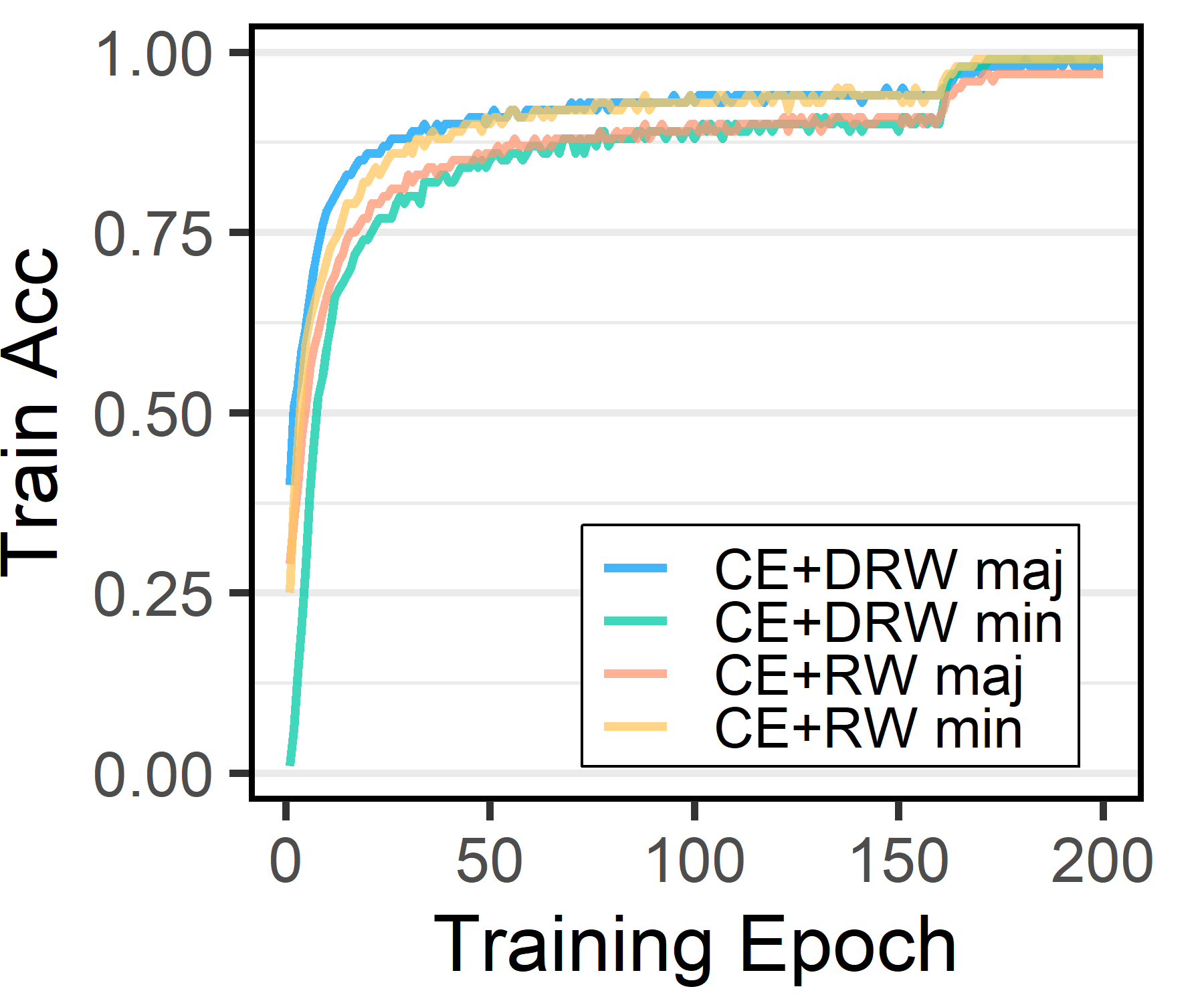}
    }
    \subfigure[CIFAR-10 LT ($\rho = 10$)]{
    \includegraphics[width=0.28\columnwidth]{drw-cifar10,exp,ratio.png}
    }
    \subfigure[CIFAR-10 LT ($\rho = 10$)]{
      \includegraphics[width=0.28\columnwidth]{drw-cifar10,exp,acc.png}
    }
    \caption{(a) Training accuracy of CE+DRW ($T_0 = 160$) and the CB loss ($\alpha_y = (1 - p) / (1 - p^{N_y})$). (b) The ratio of the training accuracy between the minority classes and the majority classes of the best model \textit{w.r.t.} the DRW epoch $T_0$. (c) The test accuracy of the best model \textit{w.r.t.} the DRW epoch $T_0$. We can find that the DRW scheme balances the training accuracy between the majority classes and the minority classes and thus improves the model performance on the test set, which is consistent with the theoretical insight \textbf{(In2)}.}
    \label{fig:drw_app}
\end{figure}

\begin{figure*}[h]
    \centering
    \subfigure[Re-weighting]{
        \includegraphics[width=0.29\linewidth]{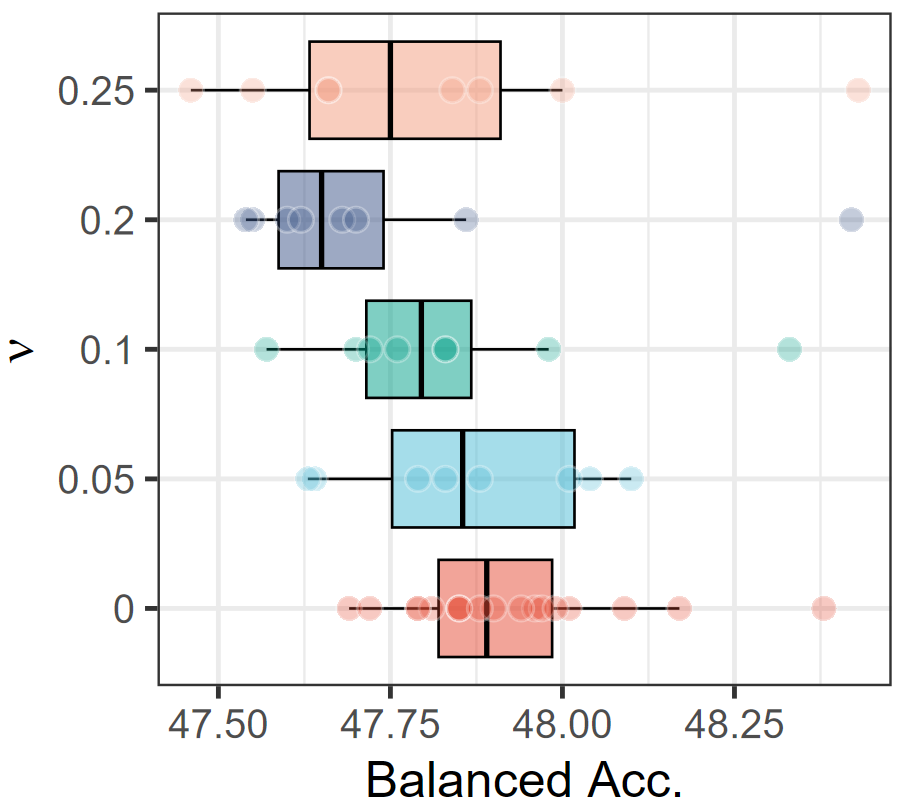}
       }
    \subfigure[Multiplicative adjustment]{
      \includegraphics[width=0.29\linewidth]{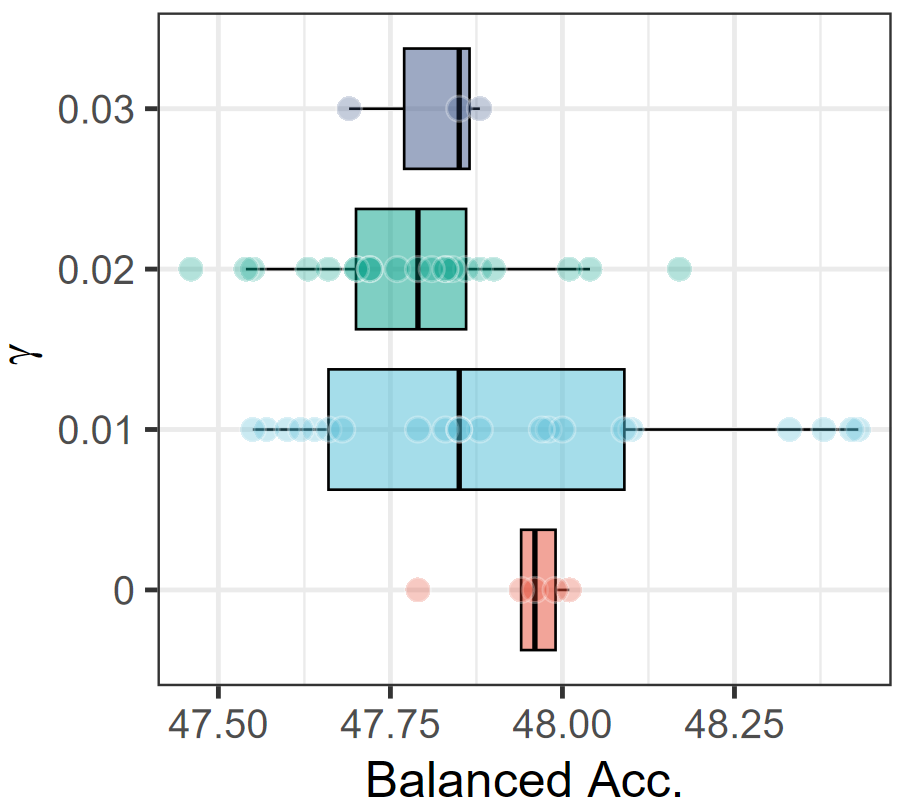}
     }
    \subfigure[Additive adjustment]{
      \includegraphics[width=0.29\linewidth]{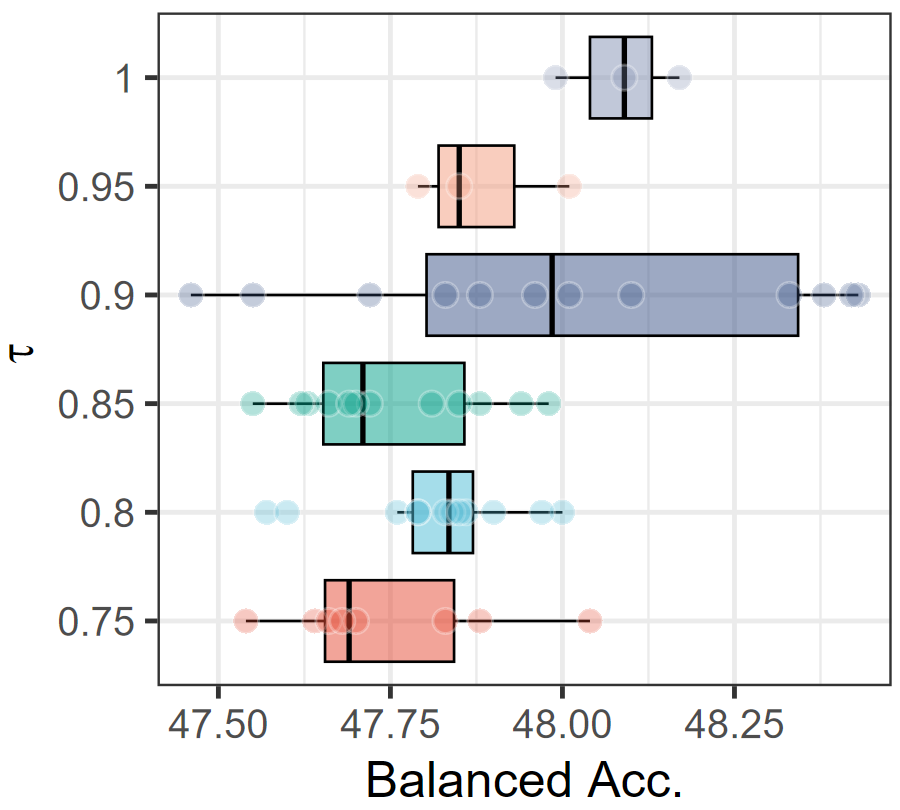}
     }
    \caption{Sensitivity analysis of the proposed method, conducted on CIFAR-100 LT ($\rho = 100$) under the protocol (b).}
    \label{fig:sensitivity_new_b}
\end{figure*}

\clearpage
\subsection{Sensitivity Analysis}
\label{app:sensitivity}

Fig.\ref{fig:sensitivity_new_b} presents the sensitivity analysis of the proposed method on CIFAR-100 LT ($\rho = 100$) under the protocol (b). On one hand, the optimal $\nu, \gamma, \tau$ are 0.25, 0.01, 0.9, respectively, which is similar to those under protocol (c), \textit{i.e.}, Fig.\ref{fig:sensitivity_new}. This again shows that the optimal hyperparameters are highly related to the imbalance ratio. On the other hand, $\gamma$ and $\tau$ are less sensitive than those under protocol (c). This difference is not surprising if we check the generalization bound in Prop.\ref{prop:vs_generalization}. Specifically, in protocol (b), the model has not been well trained due to fewer training epochs and the lack of RandAugment. Hence, the term $\Phi(L_\text{VS}, \delta)$ is larger than that in protocol (c), weakening the impact of the class-dependent terms. Hence, this result highlights the importance of training protocols in making the class-dependent terms effective.

\subsection{Computational and GPU memory usage}
We  evaluate the computational and GPU memory costs on the CIFAR-100 LT ($\rho = 100$) and iNaturalist datasets. Specifically, a ViT-B/16 model is trained via LIFT \cite{DBLP:conf/icml/Shi00SH024} on a single NVIDIA RTX 4090 GPU, equipped with common acceleration techniques such as mixed precision training. The slope estimation is accelerated by Pytorch method \texttt{torch.linalg.lstsq}. From the results in Tab.\ref{tab:computational_costs}, we can find that:
\begin{itemize}
    \item The computational cost of slope estimation is negligible on small-scale datasets (\textit{e.g.}, CIFAR-100 LT), where the slope estimation only takes 0.49 seconds per epoch. On large-scale datasets (\textit{e.g.}, iNaturalist), the slope estimation takes 40.10 seconds per epoch, which is also acceptable considering the training time (1359.67 seconds per epoch).
    \item The memory cost of slope estimation could be relatively high on large-scale datasets due to the increased number of classes and samples. For example, the memory cost of slope estimation on iNaturalist is 14286.4 MB, which accounts for 70.7\% of the training GPU memory. Even though, such GPU memory cost is still acceptable. On one hand, model training and slope estimation are performed sequentially. Hence, as long as the GPU memory of slope estimation is less than the training one, it will not cause out-of-memory issues. On the other hand, the estimation is parallelized across all the classes simultaneously. If a low GPU memory usage is essential, we reduce the number of parallelized classes to reduce the memory cost, slightly sacrificing the time efficiency.
\end{itemize}


\begin{table}[htbp]
    \centering
    \caption{Computational and memory costs of slope estimation on CIFAR-100 LT and iNaturalist datasets.}
    \begin{tabular}{ccccc}
    \toprule
    \multirow{2}{*}{Dataset} & \multicolumn{2}{c}{Training} & \multicolumn{2}{c}{Slope Estimation} \\
    \cmidrule(lr){2-3} \cmidrule(lr){4-5}
    & Time (s/epoch) & GPU Memory (MB) & Time (s/epoch) & GPU Memory (MB) \\
    \midrule
    CIFAR-100 LT & 26.39 & 12033.8 & 0.49 (1.9\%) & 648.5 (5.4\%) \\
    iNaturalist & 1359.67 & 20217.6 & 40.10 (2.9\%) & 14286.4 (70.7\%) \\
    \bottomrule
    \end{tabular}
    \label{tab:computational_costs}
\end{table}

    \subsection{Alternative Methods for Slope Estimation}
    In default, we estimate $\kappa_y^+$ and $\kappa_y^*$ via least square regression. For stability, we explore alternative methods for slope estimation. Specifically, the estimation of the majority classes is relatively stable due to the sufficient number of samples. The stability issue mainly arises from the minority classes. To address this issue, we can group the minority classes and estimate a shared slope for each group. Besides, we replace the least square regression with Huber regression to improve robustness against outliers. From the results in Tab.\ref{tab:cifar-100-lt-huber}, we can see that 1) When the imbalance ratio is high, the improved estimation can effectively improve the model performance on the minority classes. 2) When the imbalance ratio is low, the performance of the minority classes is already satisfactory, and the improvements become marginal.

\begin{table}[htbp]
    \centering
    \caption{The comparison of different slope estimation methods on the CIFAR-100 LT dataset.}
    \renewcommand{\arraystretch}{1.05}
    \begin{tabular}{lccccc}
        \toprule
        \multirow{2}{*}{Imbalance Ratio} & \multicolumn{4}{c}{100} & \multicolumn{1}{c}{10} \\
        \cmidrule(lr){2-6}    & Many  & Med. & Few & All  & All  \\
        \midrule
            \multicolumn{6}{c}{PEFT pre-trained ViTs via LIFT \cite{DBLP:conf/icml/Shi00SH024}} \\
        \midrule
        \textbf{CVS (LSTSQ)} & 82.1 & 81.5 & 75.4 & 79.9 & 83.1 \\
        \textbf{CVS (Huber)} & 82.0 & 81.6 & 77.2 & \underline{81.1} & \underline{83.2} \\
        \midrule
        \textbf{CVS+ADRW (LSTSQ)} & 79.7 & 80.5 & 80.7 & 80.3 & 83.5 \\
        \textbf{CVS+ADRW (Huber)} & 79.6 & 80.6 & 81.2 & \underline{80.4} & \underline{83.5} \\
        \bottomrule
    \end{tabular}%
    \label{tab:cifar-100-lt-huber}%
\end{table}

\end{document}